\documentclass[10pt, hidelinks]{article} 
\usepackage{fullpage}
\usepackage{microtype}
\usepackage[square,numbers]{natbib}
\usepackage[toc,page,header]{appendix}
\usepackage{minitoc}

\usepackage{minitoc}
\usepackage{minitoc}

\usepackage[utf8]{inputenc} 
\usepackage[T1]{fontenc}    
\usepackage{hyperref}       
\usepackage{url}            
\usepackage{booktabs}       
\usepackage{amsfonts}       
\usepackage{nicefrac}       
\usepackage{graphicx}

\usepackage{microtype}      
\usepackage{xcolor,colortbl}       
\usepackage{multicol}
 \usepackage{vwcol}   
\usepackage{wrapfig}
\usepackage{pifont} 

\usepackage{amsthm}
\usepackage{comment}
 \usepackage{subcaption} 
 \usepackage{arydshln}
\usepackage[ruled,algo2e]{algorithm2e}
\usepackage{amsmath,amsfonts,amssymb}
\usepackage{adjustbox}
\usepackage{paralist}
\usepackage{algorithm}
\usepackage{algorithmic}
\newcommand{\algemph}[3]{\algcolor{#1}{#2}{#3}}
\theoremstyle{plain}
\newtheorem{theorem}{Theorem}[section]
\newtheorem{proposition}[theorem]{Proposition}
\newtheorem{lemma}[theorem]{Lemma}
\newtheorem{corollary}[theorem]{Corollary}
\theoremstyle{definition}
\newtheorem{definition}[theorem]{Definition}
\newtheorem{assumption}[theorem]{Assumption}
\theoremstyle{remark}
\newtheorem{remark}[theorem]{Remark}

\def\ddefloop#1{\ifx\ddefloop#1\else\ddef{#1}\expandafter\ddefloop\fi}
\def\ddef#1{\expandafter\def\csname bb#1\endcsname{\ensuremath{\mathbb{#1}}}}
\ddefloop ABCDEFGHIJKLMNOPQRSTUVWXYZ\ddefloop
\def\ddef#1{\expandafter\def\csname c#1\endcsname{\ensuremath{\mathcal{#1}}}}
\ddefloop ABCDEFGHIJKLMNOPQRSTUVWXYZ\ddefloop
\def\ddef#1{\expandafter\def\csname v#1\endcsname{\ensuremath{\boldsymbol{#1}}}}
\ddefloop ABCDEFGHIJKLMNOPQRSTUVWXYZabcdefghijklmnopqrstuvwxyz\ddefloop
\def\ddef#1{\expandafter\def\csname v#1\endcsname{\ensuremath{\boldsymbol{\csname #1\endcsname}}}}
\ddefloop {alpha}{beta}{gamma}{delta}{epsilon}{varepsilon}{zeta}{eta}{theta}{vartheta}{iota}{kappa}{lambda}{mu}{nu}{xi}{pi}{varpi}{rho}{varrho}{sigma}{varsigma}{tau}{upsilon}{phi}{varphi}{chi}{psi}{omega}{Gamma}{Delta}{Theta}{Lambda}{Xi}{Pi}{Sigma}{varSigma}{Upsilon}{Phi}{Psi}{Omega}{ell}\ddefloop

\newcommand\ip[1]{\langle #1 \rangle} 
\newcommand{\E}{\ensuremath{\mathbb{E}}} 
\newcommand{\bm}[1]{\boldsymbol{#1}}

\newcommand{\norm}[1]{\left\|#1\right\|}
\newcommand{\abs}[1]{\left|#1\right|}
\hypersetup{colorlinks=true,linkcolor=blue,filecolor=magenta,urlcolor=cyan,}

\definecolor{antiquewhite}{rgb}{0.98, 0.92, 0.84} 
\definecolor{blizzardblue}{rgb}{0.67, 0.9, 0.93} 
\newcommand{\algcolor}[3]{\hspace*{-\fboxsep}\colorbox{#1}{\parbox{#2\linewidth}{#3}}}

\newcommand{\upstairs}[1]{\textsuperscript{#1}}
\newcommand{\affilone}{\dag}

\begin{document}

\doparttoc 
\faketableofcontents 

\part{} 

\begin{center}

  {\bf{\LARGE Tight Analysis of Extra-gradient and Optimistic Gradient Methods For Nonconvex Minimax Problems}} \\

  \vspace*{.15in}
  
  \begin{tabular}{cc}
    Pouria Mahdavinia\upstairs{\affilone}
    ~~Yuyang Deng\upstairs{\affilone}
    ~~Haochuan Li\upstairs{\S}
    ~~Mehrdad Mahdavi\upstairs{\affilone}
    \\[5pt]
    \upstairs{\affilone}Department of  Computer Science and Engineering \\
    The Pennsylvania State University\\
    \texttt{\{pxm5426, yzd82, mzm616\}@psu.edu}\\
    \\
    \upstairs{\S}Department of Electrical and Computer Engineering \\
    Massachusetts Institute of Technology\\
   \texttt{haochuan@mit.edu}
  \end{tabular}
  
 \vspace*{.0in}
\begin{abstract}
Despite the established convergence theory of Optimistic Gradient Descent Ascent (OGDA) and Extragradient (EG) methods for the convex-concave minimax problems, little is known about the theoretical guarantees of these methods in nonconvex settings. To bridge this gap, for the first time, this paper establishes the convergence of OGDA and EG methods under the nonconvex-strongly-concave (NC-SC) and nonconvex-concave (NC-C) settings by providing a unified analysis through the lens of single-call extra-gradient methods. We further establish lower bounds on the convergence of GDA/OGDA/EG, shedding light on the tightness of our analysis. We also conduct experiments supporting our theoretical results. We believe our results will advance the theoretical understanding of OGDA and EG methods for solving complicated nonconvex minimax real-world problems, e.g., Generative Adversarial Networks (GANs) or robust neural networks training. 
\end{abstract}

\end{center}

\section{Introduction}

\label{submission}

In this paper, we consider the following minimax problem:
\begin{equation}
\label{eqn:minimax}
 \underset{\vx \in \mathbb{R}^d }{\min}  \,  \underset{\vy \in \cY}{\max} \, \,  f(\vx , \vy) \end{equation}
where $\cY$ could be a bounded convex or unbounded set, and the function $f:\mathbb{R}^d \times \cY \to \mathbb{R}$ is smooth and strongly-concave/concave with respect to $\vy$, but possibly nonconvex in $\vx$. Minimax optimization (Problem \ref{eqn:minimax}) has been explored in a variety of fields, including classical game theory, online learning, and control theory~ \cite{basar1999dynamic,von2007theory, hast2013pid}. Minimax has emerged as a key optimization framework for machine learning applications such as generative adversarial networks (GANs)~\cite{goodfellow2014generative}, robust and adversarial machine learning~\cite{sinha2017certifying,madry2017towards,goodfellow2014explaining}, and reinforcement learning~\cite{zhang2021multi,qiu2020single}.

Gradient descent ascent (GDA) is a well-known algorithm for solving minimax problems, and it is widely used to optimize generative adversarial networks. GDA performs a gradient descent step on the primal variable $\vx$ and a gradient ascent step on the dual variable $\vy$ simultaneously in each iteration. GDA with equal step sizes for both variables converges linearly to Nash equilibrium under the  strongly-convex strongly-concave (SC-SC) assumption~\cite{liang2019interaction,fallah2020optimal}, but diverges even under the convex-concave (C-C) setting for functions such as bilinear~\cite{hommes2012multiple,mertikopoulos2018optimistic}. 

Given the high nonconvexity of practical applications such as GANs, exploring convergence guarantees of minimax optimization algorithms beyond the convex-concave (C-C) setting is one of the canonical research directions in minimax optimization. Several algorithms with convergence guarantees beyond the C-C domain have been explored in the literature. Alternating Gradient Descent Ascent (AGDA) is one of these methods demonstrated to have excellent convergence properties beyond the C-C setting~\cite{yang2020global,yang2021faster,chen2021accelerated}. Additionally, two alternative powerful algorithms are Extragradient (EG) and Optimistic GDA (OGDA), which have recently acquired prominence due to their superior empirical performance in optimizing GANs compared to other minimax optimization algorithms~\cite{liang2019interaction,daskalakis2017training,mertikopoulos2018optimistic}. 

\begin{table*}[t!]
\centering
\footnotesize
\resizebox{\columnwidth}{!}{%
\begin{tabular}{|c|c|c|c|c|}
\hline
 {Algorithm}& \multicolumn{2}{|c|}{NC-C} & \multicolumn{2}{|c|}{NC-SC}  \\
\cline{2-5}
         & Deterministic & Stochastic & Deterministic & Stochastic \\ \hline \hline 
PG-SVRG~\cite{rafique2018non}  &  -  & $\tilde{O}(\epsilon^{-6})$ & - & -\\    \hline 
  HiBSA~\cite{lu2020hybrid}  & $O(\epsilon^{-8})$   & - & - & -\\    \hline 
   Prox-DIAG~\cite{thekumparampil2019efficient}  & $\tilde{O}(\epsilon^{-3})$   & - & - & -\\    \hline 
    Minimax-PPA~\cite{lin2020near}  & $O(\epsilon^{-4})$   & -  & $O(\frac{\sqrt{\kappa}}{\epsilon^2})$ & - \\    \hline 
    
     ALSET~\cite{chen2021tighter} & - & - & $O(\frac{\kappa^3}{\epsilon^2})$ & $O(\frac{\kappa^3}{\epsilon^4})$
    \\ \hline
    Smoothed-AGDA~\cite{yang2021faster} & - & - & $O(\frac{\kappa}{\epsilon^2})$& $O(\frac{\kappa^2}{\epsilon^4})$
    \\ \hline
  GDA~\cite{lin2020gradient}  & $O(\epsilon^{-6})$   & $ O(\epsilon^{-8})$    & $O(\frac{\kappa^2}{\epsilon^{2}})$ & $O(\frac{\kappa^3}{\epsilon^4})$  \\    \hline 
\rowcolor{antiquewhite}
   
  OGDA/EG (Theorems~\ref{thm:ncsc_ogda},~\ref{thm:ncsc_sogda},~\ref{thm:ncc_ogda},~\ref{thm:ncc_sogda})  & $ O(\epsilon^{-6})$  &$  O(\epsilon^{-8})$  & $O(\frac{\kappa^2}{\epsilon^{2}})$ & $O(\frac{\kappa^3}{\epsilon^4})$  \\
 \hline
\end{tabular}}\label{tab:results} \vspace{-0.25cm}
\caption{A summary of prior  and our convergence rates in nonconvex-concave (NC-C) and nonconvex-strongly-concave (NC-SC) minimax optimization. For NC-C, we assume $f(\bm{x},\bm{y})$ is $\ell$-smooth,  $G$-Lipschitz in $\bm{x}$, and concave in $\bm{y}$, and for NC-SC we assume $\ell$-smoothness, and $\mu$-strong concavity in $\vy$, where $\kappa = \frac{\ell}{\mu}$.  }\label{tab:main} \vspace{-0.6cm}  
\end{table*}
Spurred by the empirical success of EG and OGDA methods, there has been a tremendous amount of work in theoretical understanding of their convergence rate under different sets of assumptions. Specifically, recently the convergence properties of EG and OGDA were investigated for SC-SC and C-C settings, where it has been shown that they tend to converge significantly faster than GDA in both deterministic and stochastic settings~\cite{mokhtari2020unified,fallah2020optimal,mokhtari2020convergence}. Despite these remarkable advances, there is a dearth of theoretical understanding  of the convergence of OGDA and EG methods in the nonconvex setting. This naturally motivates us to rigorously examine the convergence of these methods in nonconvex minimax optimization that we aim to investigate. Thus, we emphasize that our focus is on vanilla variants of OGDA/EG, and improved rates in NC-C and NC-SC problems have already been obtained with novel algorithms as mentioned in Section~\ref{sec:related}.

\noindent\textbf{Contributions.}~We propose a unified framework for analyzing and establishing the convergence of OGDA and EG methods for solving NC-SC and NC-C minimax problems. To the best of our knowledge, our analysis provides the first theoretical guarantees for such problems. Our contribution can be summarized as follows:
\begin{compactitem}
    \item For NC-SC objectives, we demonstrate that OGDA and EG iterates converge to the $\epsilon-$stationary point, with a gradient complexity of $O(\frac{\kappa^2}{\epsilon^2})$ for deterministic case, and $O(\frac{\kappa^3}{\epsilon^4})$ for the stochastic setting, matching the gradient complexity of GDA in~\cite{lin2020gradient}.  \vspace{2mm}
    \item For NC-C objectives, we establish the gradient complexity of $O(\epsilon^{-6})$ for the deterministic  and $O(\epsilon^{-8})$ for stochastic oracles, respectively. Compared to the most analogous work on GDA~\cite{lin2020gradient}, our rate matches the gradient complexity of GDA our results show that OGDA and EG have the advantage of shaving off a significant term related to primal
function gap ($\hat{\Delta}_0 = \Phi(\vx_0) - \min_{\vx} \Phi(\vx)$).\vspace{2mm}
\item We establish impossibility results on the achievable rates by providing an $\Omega(\frac{\kappa^2}{\epsilon^2})$, and $\Omega(\epsilon^{-6})$ lower bounds based on the common choice of parameters for both OGDA and EG in deterministic NC-SC and NC-C settings, respectively, thus demonstrating the tightness of our analysis of upper bounds.  \vspace{2mm}
\item By carefully designing hard instances, we establish a general lower bound  of $O(\frac{\kappa}{\epsilon^2})$, independent of the learning rate, for GDA/OGDA/EG methods in deterministic NC-SC setting--  demonstrating the optimality of  obtained upper bound  up to a factor of $\kappa$.
\end{compactitem}

\vspace{-2mm}

\section{Related Work} \label{sec:related}
\noindent\textbf{Extra-gradient (EG), and OGDA methods.} Under smooth SC-SC assumption, deterministic OGDA and EG have been shown to converge to an $O(\epsilon)$ neighborhood of the optimal solution with rate of $O(\kappa \log(\frac{1}{\epsilon}))$~\cite{mokhtari2020unified,tseng1995linear}.~\citet{fallah2020optimal} improved upon the previous rates by proposing multistage OGDA, which achieved the best-known rate of $O(\max(\kappa \log(\frac{1}{\epsilon}) , \frac{\sigma^2}{\mu^2 \epsilon^2}))$ for the stochastic OGDA in SC-SC setting. Under monotone and gradient Lipschitzness assumption (a slightly weaker notion of smooth convex-concave problems), ~\citet{cai2022tight} established the tight last iterate convergence of $O(\frac{1}{\sqrt{T}})$ for OGDA and EG, and similar results for EG has been achieved in~\cite{gorbunov2022extragradient,gorbunov2022stochastic}. Furthermore,  To the best of our knowledge, OGDA and EG methods have not been extensively explored in nonconvex-nonconcave settings except in a few recent works on structured nonconvex-nonconcave problems in which the analysis is done through the lens of a variational inequality. This line of work is discussed in the Nonconvex-nonconcave section. Moreover, recently,~\citet{guo2020fast} established the convergence rate of OGDA in NC-SC, however, they have $\mu$-PL assumption on $\Phi(\vx)$, which is a strong assumption and further allows them to show the convergence rate in terms of the objective gap. However, we did not make such an assumption on the primal function, and hence unlike~\cite{guo2020fast}, we measure the convergence by the gradient norm of the primal function. 

\noindent\textbf{Nonconvex-strongly-concave (NC-SC) problems.}~In deterministic setting,~\citet{lin2020gradient} demonstrated the first non-asymptotic convergence of GDA to $\epsilon$-stationary point of $\Phi(\vx)$, with the gradient complexity of $O(\frac{\kappa^2}{\epsilon^2})$.~\citet{lin2020near} and ~\citet{zhang2021complexity} proposed triple loop algorithms achieving gradient complexity of $O(\frac{\sqrt{\kappa}}{\epsilon^2})$ by leveraging ideas from catalyst methods (adding $\alpha \|\vx  - \vx_0 \|^2$ to the objective function), and inexact proximal point methods, which nearly match the existing lower bound~\cite{Li2021ComplexityLB,zhang2021complexity,Han2021LowerCB}. Approximating the inner loop optimization of catalyst idea by one step of GDA, Yang et al~\cite{yang2021faster} developed a single loop algorithm called smoothed AGDA, which provably converges to $\epsilon$-stationary point, with gradient complexity of $O(\frac{\kappa}{\epsilon^2})$. For stochastic setting, Lin et al~\cite{lin2020gradient} showed that Stochastic GDA, with choosing dual and primal learning rate ratio of $O(\frac{1}{\kappa^2})$, converges to $\epsilon$-stationary point with gradient complexity of $O(\frac{\kappa^3}{\epsilon^4})$.~\citet{chen2021tighter} proposed a double loop algorithm whose outer loop performs one step of gradient descent on the primal variable, and inner loop performs multiple steps of gradient ascent. Using this idea, they achieved gradient complexity of $O(\frac{\kappa^3}{\epsilon^4})$ with fixed batch size. However, their algorithm is double loop, and the iteration complexity of the inner loop is $O(\kappa)$. Yang et al~\cite{yang2021faster} also introduced the stochastic version of smoothed AGDA we mentioned earlier. They showed gradient complexity of $O(\frac{\kappa^2}{\epsilon^4})$, using fixed batch size. They achieved the best-known rate for NC-PL problems, which is an even weaker assumption than NC-SC. 

\noindent\textbf{Nonconvex-concave.}~Recently, due to the surge of GANs~\cite{goodfellow2014generative} and adversarially robust neural network training, a line of researches are focusing on nonconvex-concave or even nonconvex-nonconcave minimax optimization problems~\cite{lu2020hybrid,lin2019gradient,nouiehed2019solving,rafique2018non,thekumparampil2019efficient,gidel2018variational,liu2019towards,liu2019decentralized,jin2019local}. For nonconvex-concave setting, to our best knowledge, Rafique et al~\cite{rafique2018non} is the pioneer to propose provable nonconvex-concave minimax algorithm, where they proposed Proximally Guided Stochastic Mirror Descent Method, which achieves $O({\epsilon^{-6}})$ gradient complexity to find stationary point. Nouiehed et al~\cite{nouiehed2019solving} presented a double-loop algorithm to solve nonconvex-concave with constraint on both $\bm{x}$ and $\bm{y}$, and achieved $O(\epsilon^{-7})$ rate. Lin et al~\cite{lin2020gradient} provided the first analysis of the classic algorithm (S)GDA on nonconvex-strongly-concave and nonconvex-concave functions, and in nonconvex-concave setting they achieve $O({\epsilon^{-6}})$ for GDA and $O({\epsilon^{-8}})$ for SGDA. Zhang et al~\cite{zhang2020single} proposed smoothed-GDA and also achieve $O(\epsilon^{-8})$ rate.~\citet{thekumparampil2019efficient} proposed Proximal Dual Implicit Accelerated Gradient method and achieved the best known rate $O({\epsilon^{-3}})$ for nonconvex-concave problem.~\citet{kong2019accelerated} proposed an accelerated inexact proximal point method and also achieve $O({\epsilon^{-3}})$ rate. Lin et al~\cite{lin2020near} designed 
near-optimal algorithm using an acceleration  method with $O(\epsilon^{-3})$ rate.  However, their algorithms require  double or triple loops and are not as easy to implement as GDA, OGDA, or EG methods. 

\noindent\textbf{Nonconvex-nonconcave.}~Minimax optimization problems can be cast as one of the special cases of variational inequality problems (VIPs)~\cite{abernethy2021last,loizou2020stochastic}. Thus, one way of studying the convergence in Nonconvex-nonconcave problems is to leverage some variants of Variational Inequality properties such as Monotone variational inequality, Minty variational inequality (MVI), weak MVI, and negative comonotone, which are weaker assumptions compared to convex-concave problems. For instance, ~\citet{loizou2021stochastic} showed the linear convergence of SGDA under expected
co-coercivity, a condition that potentially holds for the non-monotone problem. Moreover, it has been shown that deterministic EG obtains gradient complexity of $O(\frac{1}{\epsilon^2})$ for the aforementioned settings~\cite{dang2015convergence,diakonikolas2021efficient,song2021optimistic,pethick2022escaping}.   Alternatively, another line  of works established the convergence under the weaker notions of strong convexity such as the Polyak-Łojasiewicz (PL) condition, or $\rho$-weakly convex. Yang et al~\cite{yang2020global} established the linear convergence of the AGDA algorithm assuming the two-sided PL condition. Hajizadeh et al~\cite{hajizadeh2022linear} achieved the same results for EG under the weakly-convex, weakly-concave assumption.

\section{Problem setup and
preliminaries}
We use lower-case boldface letters such as $\vx$ to denote vectors and let $\|\cdot\|$ denote the $\ell_2$-norm of vectors. In 
Problem~\ref{eqn:minimax}, we refer to $\vx$ as the primal variable and to $\vy$ as the dual variable. For a function $f:\mathbb{R}^m \times \mathbb{R}^n \to \mathbb{R}$, we use $\nabla_x f(\vx , \vy)$ to denote  the gradient of $f(\vx,\vy)$ with respect to primal variable $\vx$, and $\nabla_y f(\vx, \vy)$ to denote the gradient of $f(\vx,\vy)$ with respect to dual variable $\vy$. In stochastic setting, we let $\vg_{x,t}$ to be the unbiased estimator of $\nabla_x f(\vx_t,\vy_t)$, computed by a minibatch of size $M_x$ and $\vg_{y,t}$ to be the unbiased estimator of $\nabla_y f(\vx_t,\vy_t)$, computed by a minibatch of size $M_y$, where $\vx_t$ and $\vy_t$ are the $t$th iterates of the algorithms. Particularly, $ \vg_{x,t} = \frac{1}{M_x}  \sum _{i=1}^{M_x} \nabla_x f(\vx_t,\vy_t , \xi^x_{t,i}) $, and $\vg_{y,t} = \frac{1}{M_y} \sum_{i=1}^{M_y} \nabla_y f(\vx_t,\vy_t, \xi^y_{t,i})$, where $\{\xi^x_{t,i}\}_{i=1}^{M_x}$, and $\{\xi^y_{t,i}\}_{i=1}^{M_y}$ are i.i.d  minibatch samples  utilized to compute stochastic gradients at each iteration $t \in \{1, \dots ,T\}$. 

\begin{definition}[Primal Function]
We introduce $\Phi(\vx) = \max_{\vy} f(\vx,\vy)$ as the primal function, and define $\vy^*(\vx) = \arg \max_{\vy \in \mathcal{Y}} f(\vx,\vy)$ as the optimal dual variable at a point $\vx$. 
\end{definition}
\begin{definition}[Smoothness]
 A function $f(\vx,\vy)$ is $\ell$-smooth in both $\vx$, and $\vy$, if it is differentiable, and the following inequalities hold: $
         \|\nabla f(\vx_1, \vy_1) - \nabla f(\vx_2 , \vy_2 ) \|^2  \le \ell^2 \|\vx_1 - \vx_2 \|^2 + \ell^2 \|\vy_1 - \vy_2 \|^2.$
 \end{definition}
 \begin{definition}
 A function $g$ is $\mu$-strongly-convex, if for any $\vx_1,\vx_2 \in \mathbb{R}^d$ the following holds: $
g ( \vx_2 ) \ge g(\vx_1) + \ip{\nabla g(\vx_1) , \vx_2 - \vx_1} + \frac{\mu}{2} \|\vx_1 - \vx_2 \|^2$.
 \end{definition}
\begin{definition}
We say  $\vx$ is is an $\epsilon$-stationary point for a differentiable function $\Phi$ if $\|\nabla  \Phi(\vx)\| \le \epsilon$.
\end{definition}
We note that $\epsilon$-stationary point is a common optimality criterion used in the NC-SC setting. As pointed out in~\cite{lin2020gradient}, considering $\Phi(\bm{x})$ as convergence measure is natural since in many application scenarios, we mainly care about the value of the objective $f(\bm{x},\bm{y})$ under the maximized $\bm{y}$, e.g., adversarial training or distributionally robust learning.
 
When $f(\bm{x},\bm{y})$ is merely concave in $\bm{y}$, $\Phi(\bm{x})$ could be non-differentiable. Hence, following the routine of nonsmooth nonconvex minimization~\cite{davis2019stochastic}, we consider the following Moreau envelope function:
\begin{definition}[Moreau envelope] A function $\Phi_{p} (\bm{x})$ is the $p$-Moreau envelope of a function $\Phi$ if $    \Phi_{p} (\bm{x}) := \min_{\bm{x}'\in \mathbb{R}^{d}} \{ \Phi  (\bm{x}') + \frac{1}{2p}\|\bm{x}'-\bm{x}\|^2\}$. \end{definition}
We will utilize the  following property of the Moreau envelope of a nonsmooth function:
\begin{lemma}[\citet{davis2019stochastic}]\label{lemma: moreau}
Let $\hat{\bm{x}} = \arg\min_{\bm{x}'\in \mathbb{R}^d} \Phi  (\bm{x}') + \frac{1}{2p}\|\bm{x}'-\bm{x}\|^2$, then the following inequalities hold:
$\|\hat{\bm{x}} -  {\bm{x}}\| \leq p\|\nabla \Phi_p(\bm{x})\|$, $\min_{\bm{v}\in\partial\Phi(\hat{\bm{x}})} \|\bm{v}\| \leq \|\nabla \Phi_p(\bm{x})\|$.
\end{lemma}
Lemma~\ref{lemma: moreau} suggests that, if we can find a $\bm{x}$ with a small $\| \nabla \Phi_p(\bm{x})\|$, then ${\bm{x}}$ is near some point $\hat{\bm{x}}$ which is a near-stationary point of $\Phi$. We will use  $1/2\ell$-Moreau envelope of $\Phi$, following the setting in~\cite{lin2020gradient,rafique2021weakly}, and establish the convergence rates in terms of $\|\nabla \Phi_{1/2\ell} (\bm{x})\|$. We also define two quantities $\hat{\Delta}_\Phi = \Phi_{1/2\ell}(\bm{x}_0) - \min_{\bm{x}\in\mathbb{R}^d} \Phi_{1/2\ell}(\bm{x})$ and $\hat{\Delta}_{0} = \Phi (\bm{x}_{0}) - \min_{\bm{x}\in\mathbb{R}^d} \Phi (\bm{x})$ that appear in our convergence bounds. Before presenting our results on EG and OGDA, we briefly revisit the most related algorithm, Gradient Descent Ascent (GDA).

\newpage

\subsection{Gradient Descent Ascent (GDA) algorithm}

\begin{wrapfigure}{r}{0.47\textwidth}
	\begin{center}\small
\begin{minipage}{1\linewidth}
	\vspace{-1.2cm}
\begin{algorithm}[H]
\caption{GDA}
\label{alg:gda}
\begin{algorithmic}
\STATE\textbf{Input:} $(\vx_0,\vy_0)$, stepsizes $(\eta_x , \eta_y)$
\FOR{$t=1,2,\dots,T$}
\STATE $\vx_{t} \gets \vx_{t-1} - \eta_x \nabla_x f(\vx_{t-1},\vy_{t-1})$  ; 
\STATE $ \vy_{t} \gets \mathcal{P}_{\mathcal{Y}}( \vy_{t-1} + \eta_y \nabla_y f(\vx_{t-1} ,\vy_{t-1}))$ ;
\ENDFOR
\STATE Randomly choose $\Bar{\vx}$ from $\vx_1,\dots,\vx_T$ 
\STATE \textbf{Output:}$\Bar{\vx}$
\end{algorithmic}
\end{algorithm}
\end{minipage}
	\end{center}
	\vspace{-0.5cm}
\end{wrapfigure}

The GDA method, as detailed in Algorithm~\ref{alg:gda}, performs simultaneous gradient descent and ascent updates on primal and dual variables, respectively. This simple algorithm has been deployed extensively for minimax optimization applications such as Generative Adversarial Networks (GANs). Under Assumptions~\ref{asm:1}, and~\ref{asm:2},~\citet{lin2020gradient} established the convergence of GDA  by choosing $\eta_x = \Theta(\frac{1}{\kappa^2 \ell})$, and $\eta_y = \Theta(\frac{1}{\ell})$. In particular, they showed that deterministic GDA requires $O(\frac{\kappa^2}{\epsilon^2})$ calls to a gradient oracle, and stochastic GDA requires $O(\frac{\kappa^3}{\epsilon^4})$ calls using the minibatch size of $O(\frac{\kappa}{\epsilon^2})$ to find an $\epsilon$-stationary point of the primal function.

\subsection{Optimistic Gradient Descent Ascent (OGDA) and Extra-gradient (EG) Method}
We now turn to reviewing the algorithms we study in this paper: Optimistic GDA (OGDA) and Extra-gradient (EG) methods. To optimize Problem~(\ref{eqn:minimax}), at each iteration $t = 1, 2, \ldots, T$, OGDA performs the following updates on the primal and dual variables:
\begin{equation*}
\label{eqn:uogda}
\begin{split}
\vx_{t+1} &= \vx_t - \eta_x \nabla_x f(\vx_{t},\vy_{t}) - \eta_x ( \nabla_x f(\vx_{t},\vy_{t}) - \nabla_x f(\vx_{t-1},\vy_{t-1}))\\
\vy_{t+1} &= \cP_{\cY} \left(\vy_t + \eta_y \nabla_y f(\vx_{t},\vy_{t}) + \eta_y (\nabla_y f(\vx_{t},\vy_{t}) - \nabla_y f(\vx_{t-1},\vy_{t-1})) \right)
\end{split}
\tag{OGDA}
\end{equation*}
where correction terms (e.g. $\nabla_x f(\vx_{t},\vy_{t}) - \nabla_x f(\vx_{t-1},\vy_{t-1})$) are added to the updates of the GDA. 
EG method performs the following updates:
\begin{equation*}
\label{eqn:ueg}
\begin{aligned}[c]
\vx_{t+1/2} &= \vx_t - \eta_x \nabla_x f(\vx_{t},\vy_{t})\\
\vx_{t+1} &= \vx_t - \eta_x \nabla_x f(\vx_{t+1/2},\vy_{t+1/2})\\
\end{aligned}
\, \, ; \, \,
\begin{aligned}[c]
\vy_{t+1/2} &= \cP_{\cY}\left(\vy_t + \eta_y \nabla_y f(\vx_{t},\vy_{t}) \right)\\
\vy_{t+1} &= \cP_{\cY} \left(\vy_t + \eta_y \nabla_y f(\vx_{t+1/2},\vy_{t+1/2}) \right) \\
\end{aligned}
\tag{EG}
\end{equation*}
where the gradient at the current point is used to find a mid-point, and then  the gradient at the mid-point is used to find the next iterate.  We also consider \textit{stochastic} variants of  the two algorithms where we replace full gradients with unbiased stochastic estimations. The detailed versions of these algorithms are provided in  Algorithm~\ref{alg:ogda} , and Algorithm~\ref{alg:eg} in Appendix~\ref{app:NCSC_Upper}.

\section{Main Results}
We provide  upper bounds on the gradient complexity and iteration complexity of  OGDA and EG methods for  NC-C and NC-SC objectives in both deterministic and stochastic settings. We also show the tightness of obtained bounds for the choice of learning rates made. We will derive general stepsize-independent lower bounds in Section~\ref{sec:lower bound}.

\subsection{Nonconvex-strongly-concave minimax problems}\label{sec:ncsc}
We start by establishing the convergence of deterministic OGDA/EG in the NC-SC setting by making the following standard assumption on the loss function. 
\begin{assumption}
\label{asm:1}
We assume  $f:\mathbb{R}^m \times \mathbb{R}^n \to \mathbb{R}$ is $\ell$-smooth, and $f(\vx,.)$ is $\mu$-strongly-concave. 
\end{assumption}
Moreover, we assume the initial primal optimality gap is bounded. i.e., $\Delta_\Phi = \max ( \Phi(\vx_1) , \Phi(\vx_0)) - \min_x \Phi(\vx)$.
\begin{theorem}
\label{thm:ncsc_ogda}
Let $\Bar{\vx}$ be output of OGDA/EG algorithms and choose $\eta_x \le  \frac{c_1}{\kappa^2 \ell}$, $\eta_y = \frac{c_2}{\ell}$. For OGDA, let $c_1 = \frac{1}{50}, c_2=\frac{1}{6}$, and for EG, let $c_1 = \frac{1}{75}, c_2=\frac{1}{4}$. Then under Assumption~\ref{asm:1}, OGDA/EG converges to an $\epsilon$-stationary point, i.e., $\|\nabla \Phi(\bar{\bm{x}})\|^2 \leq \epsilon^2$, with iteration number $T$ bounded by: 
\begin{equation*}
     O \left(  \frac{ \kappa^2 \ell \Delta_\Phi +        \kappa \ell^2 D_0 }{\epsilon^2}  \right),
\end{equation*}
where $D_0 = \max \left ( \|\vx_1 - \vx_0\|^2 , \|\vy_1 - \vy_0\|^2 , \|\vy_1 - \vy^*_1\|^2 , \|\vy_0 - \vy^*_0\|^2\right )$.
\end{theorem}
To establish the convergence rate in stochastic setting, we will make the following
 assumption on the stochastic gradient oracle.
\begin{assumption}
\label{asm:2}
Let $\nabla_x f(\vx,\vy, \xi^x)$ and $\nabla_y f(\vx,\vy , \xi^y)$ to be the unbiased estimator of the $\nabla_x f(\vx,\vy)$ and $\nabla_y f(\vx,\vy)$, respectively. Then, the stochastic gradient oracle satisfies the following:
\begin{compactitem}
    \item Unbiasedness:    $\E_{\xi^x}\left[\nabla_x f(\vx,\vy, \xi^x)\right] = \nabla_x f(\vx,\vy)$ and 
    $\E_{\xi^y}\left[\nabla_y f(\vx,\vy , \xi^y) \right] = \nabla_y f(\vx,\vy)$.
\item Bounded variance: $\E_{\xi^x} \left [\|\nabla_x f(\vx,\vy, \xi^x) - \nabla_x f(\vx,\vy)\|^2 \right] \le  \sigma^2$ and $
    \E_{\xi^y}\left[\|\nabla_y f(\vx,\vy , \xi^y) - \nabla_y f(\vx,\vy)\|^2\right]\le \sigma^2$.
\end{compactitem}
\end{assumption}
We now turn to establishing the convergence rate in stochastic setting.
\begin{theorem}
\label{thm:ncsc_sogda}
Let $\Bar{\vx}$ be output of stochastic OGDA/EG algorithms and let $\eta_x$ and $\eta_y$ to be chosen as in Theorem~\ref{thm:ncsc_ogda}. For EG, choose minibatch size $M = \max\left\{1, \frac{\kappa \sigma^2}{\epsilon^2}\right\}$, and for OGDA choose primal minibatch size $M_x = \max\{1 , \frac{\sigma^2}{\epsilon^2}\}$, and dual minibatch size $M_y = \max\{1,\frac{\kappa \sigma^2}{\epsilon^2}\}$. Then under Assumptions~\ref{asm:1}, and~\ref{asm:2}, OGDA/EG converges to an $\epsilon$-stationary point, i.e., $ \mathbb{E}\|\nabla \Phi(\bar{\bm{x}})\|^2 \leq \epsilon^2$, with the iteration number $T$ bounded by: 
\begin{equation*}
     O \left(  \frac{ \kappa^2 \ell \Delta_\Phi +        \kappa \ell^2 D_0 }{\epsilon^2}  \right),
\end{equation*}
where $D_0 = \max \left ( \|\vx_1 - \vx_0\|^2 , \|\vy_1 - \vy_0\|^2 , \|\vy_1 - \vy^*_1\|^2 ,\|\vy_0 - \vy^*_0\|^2\right )$.
\end{theorem}

The proofs of Theorems~\ref{thm:ncsc_ogda} and~\ref{thm:ncsc_sogda} are deferred to Appendix~\ref{app:NCSC_Upper}. Our iteration complexity matches with the complexity of two-scale GDA obtained in~\cite{lin2020gradient}. However, we improve primal gradient oracle complexity for OGDA by a factor of $\kappa$ as our analysis works for smaller primal batch size $M_x$ compared to GDA~\cite{lin2020gradient}. This paper establishes primal gradient oracle complexity of $O(\frac{\kappa^2}{\epsilon^4})$, while the analysis for GDA in~\cite{lin2020gradient}, requires gradient oracle complexity of $O(\frac{\kappa^3}{\epsilon^4})$ for primal variable. 

In previous theorems, we established upper bounds on the convergence of OGDA and EG algorithms. In the following results, we turn to examining the tightness of obtained rates. To this end, we first consider a simple GDA algorithm and will extend the analysis to OGDA/EG. Note that in this section, we only consider the stepsize choice in our upper bound results.  

\begin{theorem}[Tightness of  GDA] \label{thm:NCSC_Tightness_GDA}

Consider GDA method (Algorithm~\ref{alg:gda}) with step sizes chosen as in Theorem $4.4$ in~\cite{lin2020gradient}, and let $\Bar{\vx}$ be the returned solution after $T$ iterations. Then, there exists a function $f(\cdot, \cdot)$ that is $\ell$-gradient Lipschitz and $\mu$-strongly concave in $\vy$, and an initialization $(\vx_0,\vy_0)$, such that Algorithm~\ref{alg:gda} requires at least $   T = \Omega\left(\frac{\kappa^2\Delta_\Phi}{\epsilon^2}\right)
 $ iterations to  guarantee $\|\nabla\Phi(\Bar{\vx})\| \leq \epsilon$.
 \label{thm:lower_bound_gda}
\end{theorem}
\begin{theorem}[Tightness of EG/OGDA] \label{thm:NCSC_Tightness_EG/OGDA}
 Consider deterministic EG and OGDA methods with step sizes chosen as in Theorem~\ref{thm:ncsc_ogda} and let $\Bar{\vx}$ be the returned solution after $T$ iterations. Then, there exists a function $f(\cdot, \cdot)$ that is $\ell$-gradient Lipschitz and $\mu$-strongly concave in $\vy$, and an initialization $(\vx_0,\vy_0)$, such that  both methods require  at least     $T = \Omega\left(\frac{\kappa^2\Delta_\Phi}{\epsilon^2}\right)
 $ iterations to guarantee $\norm{\nabla\Phi(\Bar{\vx})}\le\epsilon$. \label{thm:lower_bound_eg}
\end{theorem}
The proofs of Theorems~\ref{thm:NCSC_Tightness_GDA} and~\ref{thm:NCSC_Tightness_EG/OGDA} are deferred to Appendix~\ref{app:NCSC_Tightness_GDA} and~\ref{app:NCSC_Tightness_EG/OGDA},  respectively. Theorems~\ref{thm:NCSC_Tightness_EG/OGDA} show that to achieve $\epsilon$ stationary point of $\Phi$, EG and OGDA need at least $O(\frac{\kappa^2}{\epsilon^2})$ gradient evaluations, which match with our upper bound results (Theorems~\ref{thm:ncsc_ogda}). These impossibility results demonstrate the tightness of our analysis. It would also be interesting to see such analysis for stochastic setting, which we leave as a valuable future work.

\subsection{Nonconvex-concave minimax problems}
\label{sec:ncc}
We now turn to establishing the convergence rate of (stochastic) OGDA/EG in the NC-C setting. We make the following assumption throughout this subsection: 
\begin{assumption}
\label{asm:3}
We assume  $f:\mathbb{R}^m \times \cY \to \mathbb{R}$ is $\ell$-smooth in $\vx,\vy$,  $G$-Lipschitz in $\vx$ and $\cY$ is bounded convex set with diameter $D$, and also $f(\vx,.)$ is concave.
\end{assumption}
From the above assumption, we note when $f$ is merely concave in $\boldsymbol{y}$, we have to assume the dual variable domain is bounded since otherwise, the Moreau envelope function will not be well-defined (This is shown in Lemma $3.6$ in~\cite{lin2020gradient}). Therefore, the update rule for $\vy$ requires projection as follows:
\begin{equation*}
\label{eqn:uogda}
\begin{split}
\bm{y}_t =   \mathcal{P}_{\mathcal{Y}}\left(\bm{y}_{t-1} +  \eta_y \nabla_y f(\bm{x}_{t-1},\bm{y}_{t-1}) +\eta_y( \nabla_y f(\bm{x}_{t-1},\bm{y}_{t-1})- \nabla_y f(\bm{x}_{t-2},\bm{y}_{t-2})) \right) \quad \text{(OGDA)}
\end{split} 
\end{equation*}
\begin{equation*}
\label{eqn:uogda}
\begin{split}
\bm{y}_{t+1/2} = \mathcal{P}_{\mathcal{Y}}\left( \bm{y}_{t} + \eta_y \nabla_y f(\bm{x}_{t},\bm{y}_{t})\right), \quad 
\bm{y}_{t+1} = \mathcal{P}_{\mathcal{Y}}\left(\bm{y}_t + \eta_y \nabla_y f(\bm{x}_{t+1/2},\bm{y}_{t+1/2}) \right)\quad \text{(EG)}
\end{split} 
\end{equation*}

The following theorem establishes the convergence of OGDA/EG for NC-C objectives.

\begin{theorem}\label{thm:ncc_ogda}
 Let $\eta_x = O\left(\min\left\{\frac{\epsilon}{\ell G},\frac{\epsilon^2}{\ell G^2},\frac{\epsilon^4}{D^2 G^2 \ell^3}\right\}\right)$, and $\eta_y = \frac{1}{2\ell}$. By convention, we set $\bm{x}_{-1/2} = \bm{x}_0$, $\bm{y}_{-1/2} = \bm{y}_0$. Under Assumption~\ref{asm:3}, OGDA/EG converges to an $\epsilon$-stationary point, i.e., $\frac{1}{T+1} \sum_{t=0}^{T}  \|\nabla \Phi_{1/2\ell}(\bm{x}_{t})\|^2 \leq \epsilon^2$ for OGDA and $\frac{1}{T+1} \sum_{t=0}^{T}  \|\nabla \Phi_{1/2\ell}(\bm{x}_{t-1/2})\|^2 \leq \epsilon^2$ for EG, with the gradient complexity bounded by:
 \begin{align*}
     O\left( \frac{\ell G^2 \hat{\Delta}_{\Phi}}{\epsilon^4}\max \left\{ 1,\frac{D^2  \ell^2}{\epsilon^2}\right\}   \right).
 \end{align*}        
 \end{theorem}
\begin{theorem}\label{thm:ncc_sogda}
Let $\eta_x=O(\min\{\frac{\epsilon^2}{\ell (G^2+\sigma^2)},\frac{\epsilon^4}{D^2\ell^3  G\sqrt{G^2+\sigma^2} },\frac{\epsilon^6}{D^2\ell^3\sigma^2 G\sqrt{G^2+\sigma^2}}\})$, and $\eta_y=O(\min\{\frac{1}{4\ell},\frac{\epsilon^2}{\ell\sigma^2}\})$. By convention, we set $\bm{x}_{-1/2} = \bm{x}_0$, $\bm{y}_{-1/2} = \bm{y}_0$. Under Assumptions~\ref{asm:2} and~\ref{asm:3}, stochastic OGDA/EG algorithms converge to an $\epsilon$-stationary point, i.e., $\frac{1}{T+1} \sum_{t=0}^{T}  \mathbb{E}\|\nabla \Phi_{1/2\ell}(\bm{x}_{t})\|^2 \leq \epsilon^2$ for  OGDA and $\frac{1}{T+1} \sum_{t=0}^{T}  \mathbb{E}\|\nabla \Phi_{1/2\ell}(\bm{x}_{t-1/2})\|^2 \leq \epsilon^2$ for  EG, with the gradient complexity bounded by:
 \begin{align*}
     O\left(\frac{D^2\ell^3  G\sqrt{G^2+\sigma^2}\hat{\Delta}_{\Phi}}{\epsilon^6} \max\left\{ 1, \frac{  \sigma^2   }{\epsilon^2} \right\}\right).
 \end{align*} 
 \end{theorem}
The proofs of Theorems~\ref{thm:ncc_ogda} and~\ref{thm:ncc_sogda} are deferred to Appendix~\ref{app:NCC_Upper}. Here we show that OGDA/EG need at most $O\left( \frac{D^2 \ell^3 G^2 \hat{\Delta}_{\Phi}}{\epsilon^6} \right)$ gradient evaluations in deterministic setting and  $O\left(\frac{D^2\ell^3  \sigma^2 G\sqrt{G^2+\sigma^2} \hat{\Delta}_{\Phi}}{\epsilon^8} \right)$ gradient evaluations in stochastic setting to visit an $\epsilon$-stationary point. 
 
 Our stepsize choices for dual variable match the optimal analysis in convex-concave setting, $\Theta(\frac{1}{\ell})$ in deterministic setting~\cite{mokhtari2020convergence} and $\Theta(\frac{1}{\epsilon^2})$ in stochastic setting~\cite{hsieh2019convergence}, so we suppose our dual stepsize choice is optimal.  The stepsize ratio is $\frac{\eta_x}{\eta_y} = O(\epsilon^4)$ in both settings, same as~\citet{lin2020gradient}'s results on applying GDA to a nonconvex-concave objective, which reveals some connection and similarity between OGDA and GDA. However, compared to GDA~\cite{lin2020gradient}, where they get an 
$O\left( \frac{D^2  \ell^3  G^2 \hat{\Delta}_{\Phi}}{\epsilon^6} + \frac{\ell^3 D^2 \hat{\Delta}_0}{\epsilon^4}  \right)$ rate in deterministic setting, and $ O\left(\frac{D^2\ell^3 \sigma^2 G\sqrt{G^2+\sigma^2}\hat{\Delta}_{\Phi}}{\epsilon^8} +   \frac{\ell^3 D^2 \hat{\Delta}_0}{\epsilon^6} \right)$ in stochastic setting, we shave off the significant terms with dependency on $\hat{\Delta}_0$. As we will show in the proof, this acceleration is mainly due to the fact that OGDA/EG enjoys an inherent nice descent property on concave function, which is more elaborated in Section~\ref{sec:discussion}. In the stochastic setting, we observe similar superiority.

Now, we switch to examining the tightness of obtained rates. Similar to the NC-SC setting, we first consider a simple GDA algorithm and will extend the analysis to OGDA/EG.
 \begin{theorem}[Tightness of GDA ]\label{thm:NCC_Tightness_GDA}
 Consider GDA that runs $T$ iterations on solving (\ref{eqn:minimax}), and let $\vx_T$ be the returned solution. Then, there exists a function $f$ that is $G$-Lipschitz in $\vx$, $\ell$-gradient Lipschitz and concave in $\vy$, and an initialization point $(\vx_0, \vy_0)$ such that GDA requires at least $     T = \Omega\left(\frac{\ell^3 G^2 D^2\hat{\Delta}_\Phi }{\epsilon^6}\right)$ iterations to guarantee $\|\Phi_{1/2\ell}(\vx_T)\| \leq \epsilon$.
\end{theorem}
\begin{theorem}[Tightness of {OGDA/EG}]\label{thm:NCC_Tightness_EG/OGDA}
Consider {OGDA/EG} that runs $T$ iterations on solving (\ref{eqn:minimax}), and let $\vx_T$ be the returned solution. Then, there exists a function $f$ that is $G$-Lipschitz in $\vx$, $\ell$-gradient Lipschitz and concave in $\vy$, and an initialization point $(\vx_0, \vy_0)$ such that  to achieve $\|\Phi_{1/2\ell}(\vx_T)\| \leq \epsilon$, {OGDA/EG} requires at least  $T = \Omega\left(\frac{\ell^3 G^2 D^2\hat{\Delta}_\Phi }{\epsilon^6}\right)$.
\end{theorem}
The proof of Theorems~\ref{thm:NCC_Tightness_GDA} and~\ref{thm:NCC_Tightness_EG/OGDA} are deferred to Appendix~\ref{app:NCC_Tightness_GDA} and~\ref{app:NCC_Tightness_EG/OGDA}, respectively. Theorems~\ref{thm:NCC_Tightness_EG/OGDA} demonstrates that to find an $\epsilon$ stationary point of $\Phi_{1/2\ell}$, OGDA and EG with our stepsize choices need at least $O(\frac{1}{\epsilon^6})$ gradient evaluations, which verifies the tightness of  upper bound.
\subsection{Discussion}
\label{sec:discussion}
\noindent\textbf{Key technical challenges.}~Here, we present the key  technical challenges that arise in the nonconvex setting, which makes the analysis much more involved compared to the previous analysis of these algorithms in convex settings. Our proofs are mainly based on NC-C and NC-SC GDA analysis in~\cite{lin2020gradient}, and SC-SC OGDA/EG analysis in~\cite{mokhtari2020unified}. In the nonconvex-strongly-concave setting, finding an upper bound for $\sum_{i=1}^T \|\boldsymbol{y}_i - \boldsymbol{y}^*(\vx_i) \|^2$ is one of the key steps to establish the convergence rate, however bounding this term is much more complicated   for OGDA and EG than GDA due to difference in  updating rules. Note that in GDA analysis~\cite{lin2020gradient}, $\sum_{i=1}^T \|\boldsymbol{y}_i - \boldsymbol{y}^*(\vx_i) \|^2$ can be bounded by deriving simple recursive equation for $ \|\boldsymbol{y}_t - \boldsymbol{y}^*(\vx_t) \|^2 $, while extending it to OGDA is quite complicated. Hence, we propose to bound $r_t = \|\boldsymbol{z}_{t+1} - \boldsymbol{y}^*(\vx_t) \|^2 + \frac{1}{4} \|\boldsymbol{y}_t - \boldsymbol{y}_{t-1} \|^2$, and establish the upper bound on $ \sum_{i=1}^t \|\boldsymbol{y}_i  - \boldsymbol{y}^*(\vx_i)\|^2$ in terms of $\sum_{i=1}^t r_i$. In nonconvex-concave setting, we have to bound $\|\bm{y}_t - \bm{y}_{t-1}\|^2$, so we reduce it to the primal function gap: $\Phi(\bm{x}_t) - f(\bm{x}_t,\bm{y}_t)$. To bound this gap, we utilize the benign descent property of OGDA and EG on concave function and shave off a significant term $\hat{\Delta}_0$, which yields a better upper complexity bound than GDA.

\paragraph{On descent property of concave function for OGDA/EG}

Take OGDA, for example. The key step in NC-C analysis is to bound $\Phi( \vx_t) - f( \vx_t ,  \vy_t)$. In OGDA proof, we split this into  bounding the following:
\begin{equation}
\begin{split}
  \Phi( \vx_t) - f( \vx_t, \vy_t) &\leq f( \vx_t, {\vy}^*( \vx_t))-f( {\vx}_{s}, {\vy}^*( {\vx}_{t }))+f( {\vx}_{s}, {\vy}^*( {\vx}_{s})) \\
  &\qquad- f( {\vx}_{t }, {\vy}^*( {\vx}_{s}))  + f( {\vx}_{t }, {\vy}^*( {\vx}_{s})) - f( {\vx}_{t }, {\vy}_{t }). 
\end{split}
\end{equation}

For the last term $f( {\vx}_{t }, {\vy}^*( {\vx}_{s})) - f( {\vx}_{t }, {\vy}_{t })$, OGDA can guarantee its convergence without bounded gradient assumption on $\vy$. However, for GDA, it requires bounded gradient assumption on $\vy$ to show the convergence of this term, and without such assumption, we can only show the convergence of $f( {\vx}_{t }, {\vy}^*( {\vx}_{s})) - f( {\vx}_{t }, {\vy}_{t+1 })$, so~\citet{lin2020gradient} split the  $\Phi( {\vx}_{t }) - f( {\vx}_{t }, {\vy}_{t })$ as follow:
\begin{equation}
\begin{split}
     \Phi( {\vx}_{t }) - f( {\vx}_{t }, {\vy}_{t }) &\leq f( {\vx}_{t }, {\vy}^*( {\vx}_{t }))-f( {\vx}_{t}, {\vy}^*( {\vx}_{s }))+f( {\vx}_{t+1}, {\vy}_{t+1}) - f( {\vx}_{t }, {\vy}_t)  + f( {\vx}_{t }, {\vy}_{t+1}) \\
     &\qquad- f( {\vx}_{t+1 }, {\vy}_{t+1}) + f( {\vx}_t, {\vy}^*(\vx_s)) - f(\vx_t,\vy_{t+1})
\end{split}
\end{equation}
Hence they reduce the problem to bounding $f( {\vx}_{t }, {\vy}^*( {\vx}_{s})) - f( {\vx}_{t }, {\vy}_{t+1 })$. Therefore, they have to pay the price for the extra term $f( {\vx}_{t+1}, {\vy}_{t+1}) - f( {\vx}_{t }, {\vy}_t)$.

\noindent\textbf{Generalized OGDA.}~Generalized OGDA algorithm is a variant of OGDA in which different learning rates are used for current gradient $\nabla f(\vx_t,\vy_t)$, and the correction term $\nabla f(\vx_t,\vy_t) - \nabla f(\vx_{t-1} , \vy_{t-1})$. The update rule for this algorithm is as follows:
\begin{equation*}
\label{eqn:ugogda}
\begin{split}
\vx_{t+1} &= \vx_t - \eta_{x,1} \nabla_x f(\vx_{t},\vy_{t}) - \eta_{x,2} ( \nabla_x f(\vx_{t},\vy_{t}) - \nabla_x f(\vx_{t-1},\vy_{t-1}))\\
\vy_{t+1} &= \cP_{\cY} \left(\vy_t + \eta_{y,1} \nabla_y f(\vx_{t},\vy_{t}) + \eta_{y,2} (\nabla_y f(\vx_{t},\vy_{t}) - \nabla_y f(\vx_{t-1},\vy_{t-1})) \right)
\end{split}
\tag{OGDA+}
\end{equation*}
~\citet{mokhtari2020unified} introduced this algorithm and established the convergence bound for the bilinear setting while analysis beyond this setting remained as an open problem. In Appendix~\ref{app:gen_ogda}, we show that our analysis can be adapted to establish the convergence of the generalized OGDA algorithm. In Section~\ref{sec:exp}, the empirical advantage of generalized OGDA over the state of art optimization algorithms is shown, and it seems this algorithm is a better alternative to OGDA in practice. We also define the correction term ratios $\beta_1 = \frac{\eta_{x,2}}{\eta_{x,1}}$, $\beta_2 = \frac{\eta_{y,2}}{\eta_{y,1}}$, and empirically study the effect of these parameters on convergence. Note that if $\beta_1 = \beta_2 = 1 $, generalized OGDA would be same as OGDA. It would also be an interesting future direction to analyze this algorithm for C-C and SC-SC problems to understand its superior performance better.

\noindent\textbf{Projected OGDA/EG for NC-SC.}~Here, we highlight that while our analysis for NC-SC assumes that $\cY = \mathbb{R}^n$, it can be easily extended to a constrained setting, where the dual update is performed under projection onto a convex bounded set $\cY$.  In the following, we provide a proof sketch for extending our analysis of OGDA  to its projected variant, in which we do the same primal update as unconstrained OGDA and a projected (Optimistic gradient) OG update, as defined in \cite{hsieh2019convergence}, on the dual variable. The main idea behind our dual descent lemma, Lemma~\ref{lemma:2}, is interpreting OGDA as an extension of the PEG/OG method and then using Theorem 5 of~\cite{hsieh2019convergence} for PEG/OG analysis, which already considers the projected gradient updates. Thus, our Lemma~\ref{lemma:2} could be immediately adapted to the projected update. Lemma~\ref{lemma:4} can also be extended to projected setting by leveraging Lemma A.1 in~\cite{hsieh2019convergence}. Combining the projected variant of the mentioned lemmas, the convergence could be easily established for projected OGDA/EG.

\section{Stepsize-Independent Lower Bounds}\label{sec:lower bound}
So far, we have established upper bounds and tightness results given specific stepsize choices. In this section, we turn to establishing general  stepsize-independent lower bound results in the NC-SC setting. 
\begin{theorem}[Lower complexity bound for GDA] Consider deterministic GDA method (Algorithm~\ref{alg:gda}) with any arbitrary choice of learning rates, and let $\Bar{\vx}$ be the returned solution. Then, there exists a function $f$ satisfying Assumption~\ref{asm:1}, and an initialization $(\vx_0,\vy_0)$, such that Algorithm~\ref{alg:gda} requires at least $   T = \Omega\left(\frac{\kappa }{\epsilon^2}\right)
 $ iterations to  guarantee $\|\nabla\Phi(\Bar{\vx})\| \leq \epsilon$.
 \label{thm:lower_bound_gda}
\end{theorem}

Theorem~\ref{thm:lower_bound_gda} implies that GDA algorithm can not find $\epsilon$ stationary point of NC-SC problem with less than with $\Omega(\frac{\kappa}{\epsilon^2})$ many gradient evaluations. This result provides the first known lower bound for the GDA algorithm in NC-SC, showing that the rate obtained in~\cite{lin2020gradient} for the convergence of GDA is tight up to a factor of $\kappa$. The general proof idea is to consider the following quadratic NC-SC function $f:\mathbb{R} \times \mathbb{R} \mapsto \mathbb{R}$, which is strongly-concave in both $\vx$ and $\vy$:
\begin{align*}
	f(x,y):=-\tfrac{1}{2}\ell x^2+bxy-\tfrac{1}{2}\mu y^2.
\end{align*} 
By construction,  $f$ is nonconvex in $x$ (it is actually concave in $x$) and $\mu$-strongly-concave in $y$. Assume $\kappa:=\ell/\mu\ge 4$ and choose 	$b = \sqrt{\mu (\ell+\mu_x)}
$ for some $0<\mu_x\le \ell/2$ to be chosen later.
Then we know $b\le \ell/2$, and it is easy to verify that $f$ is $\ell$ smooth. Note that the primal function
\begin{align*}
	\Phi(x)=\max_y f(x,y)=\tfrac{1}{2}\mu_x x^2
\end{align*}
is actually strongly convex. This also justifies the symbol for $\mu_x$. We use GDA to find the solution for $\min_x\max_y f(x,y)$. Indeed for this problem, the optimal solution is achieved at the origin. The stepsizes ratio is chosen as
$	r = \frac{\eta_y}{\eta_x}$ and $ \eta_y = \frac{1}{\ell}$ for some numerical constants $c$. Then the GDA update rule can be written as
\begin{equation}
\begin{pmatrix}
x_{k+1}\\ y_{k+1} 
\end{pmatrix}= (\mathbf{I}+\eta_x \mathbf{M})\cdot \begin{pmatrix}
x_{k}\\ y_{k} 
\end{pmatrix}, \, \mathbf{M}:=\begin{pmatrix}
\ell & -b\\ r b & -\mu r 
\end{pmatrix}. 
\label{eq:gda_update-2}
\end{equation}
Note that~\eqref{eq:gda_update-2} is a linear time-invariant system, and due to the simplicity of quadratic form, we are able to track the dynamic of primal and dual variables. By iterating this linear system and analyzing the eigenvalues of the transition matrix, we are able to lower bound the gradient at final iterations.

 Now we turn to the extension of the lower bound analysis of GDA to OGDA/EG as stated below. 

\begin{theorem}[Lower complexity bound for OGDA/EG] 
 Consider the deterministic OGDA/EG method with any arbitrary choice of learning rates and let $\Bar{\vx}$ be the returned solution. Then, there exists a function $f$ satisfying Assumption~\ref{asm:1}, and an initialization $(\vx_0,\vy_0)$, such that OGDA/EG method requires  at least     $T = \Omega\left(\frac{\kappa \Delta_\Phi}{\epsilon^2}\right)
 $ iterations to guarantee $\norm{\nabla\Phi(\Bar{x})}\le\epsilon$. \label{thm:lower_bound_eg}
\end{theorem}
Theorem~\ref{thm:lower_bound_eg} shows that OGDA/EG methods can not find $\epsilon$-stationary point for any choice of learning rates with less than $\Omega(\frac{\kappa}{\epsilon^2})$ gradient evaluations. Given the upper bounds we derived for deterministic OGDA/EG in section~\ref{sec:ncsc}, our result indicates that our upper bounds is tight up to a factor of $\kappa$, however, we highlight that according to Theorem~\ref{thm:NCSC_Tightness_EG/OGDA}, given our choice of the learning rate, our upper bound is exactly tight. The complete proof of Theorems~\ref{thm:lower_bound_gda} and~\ref{thm:lower_bound_eg} are deferred to Appendix~\ref{app:lw_ncsc}.

\section{Experiments}
\label{sec:exp}
In this section, we empirically evaluate the performance of the OGDA algorithm. In particular, we follow~\cite{yang2021faster} and optimize 
Wasserstein GAN (WGAN) on a synthetic dataset generated from a Gaussian distribution. We mainly follow the setting of~\cite{yang2021faster,loizou2020stochastic} to conduct our experiment. We consider optimizing the following WGAN loss, where the generator approximates a one-dimensional Gaussian distribution:
\begin{equation}
\begin{split}
 \underset{w_G}{\min} \,  \underset{w_D}{\max} &\quad \E_{x \sim \cN (\mu,\sigma^2)}[D_{w_D} (x)] - \E_{z \sim \cN(0,1) }[D_{w_D} (G_{w_G}(z))] - \lambda \|w_D\|^2
 \end{split}
\end{equation}
Where $w_G$ and $w_D$ correspond to generator and discriminator parameters, respectively. We define discriminator to be $D(x) = \phi_1 x + \phi_2 x^2$, and generator to be a neural network with one hidden layer with 5 neurons with ReLU activation function, same as the setup considered in~\cite{yang2021faster}. We assume that real data comes from a Gaussian $\cN(\mu,\sigma^2)$ distribution, and the generator tries to approximate $\mu$ and $\sigma^2$ using a neural network. We set $\mu= 0$, and $\sigma = 0.1$. $\lambda$ is the regularization parameter which we set to $0.001$. Note that $\lambda$  makes the function strongly-concave/concave in terms of discriminator parameters, so the problem becomes NC-SC/NC-C. 

Performance of fine-tuned stochastic OGDA is depicted in Figure~\ref{fig:lw}, in comparison to ADAM~\cite{kingma2014adam}, RMSprop, SGDA~\cite{lin2020gradient}, SAGDA~\cite{yang2021faster}, and Smooth-SAGDA~\cite{yang2021faster}, which are well-known minimax optimization methods. Our evaluation shows that OGDA outperforms all of these methods, supporting the empirical advantage of OGDA as seen in relevant studies \cite{liang2019interaction,daskalakis2017training}. While our theoretical results show that OGDA/EG might not outperform GDA in terms of convergence rate, comparing the empirical result suggests that OGDA might converge faster. In Figure~\ref{fig:WD}, the evolution of the Wasserstein distance metric during the training has been shown. While GDA and OGDA are stabilized faster than other algorithms, it seems that they converge to a suboptimal solution, which incurs a higher Wasserstein distance. Thus, our study suggests that comparing different minimax algorithms only based on the convergence of gradient norm may not be that insightful in practice, as they might converge to a suboptimal equilibrium. This observation naturally leads to an interesting future direction to theoretically understand how different notions of equilibrium in first-order minimax optimization algorithms are related to the realistic  performance of practical methods such as GANs or WGANs.

The common version of OGDA, as depicted in Algorithm~\ref{alg:ogda} in Appendix~\ref{app:NCSC_Upper}, uses the same learning rate for the current gradient and correction term (difference between gradient). Empirically, we observed that using different learning rates for those terms (which we call generalized OGDA) makes the convergence faster and more stable. Hence in the following, we investigate the effect of using different correction term ratios in OGDA, which we refer them as $\beta_1$ and $\beta_2$ as defined in Subsection~\ref{sec:discussion}. The results in Figure~\ref{fig:beta}  demonstrate  that  small values of these parameters benefit the convergence rate, and larger values degrade the performance. We further observe that using correction term ratios larger than $0.5$ makes the algorithm diverge and become unstable. Hence, this corroborates the practical importance of the generalized OGDA algorithm compared to OGDA, as we are restricted to choosing the same learning rate  in OGDA (i.e.,  $\beta_1= \beta_2 = 1$).

\begin{figure}
\centering
\begin{subfigure}{.33\textwidth}
  \centering
  \includegraphics[width=.99\linewidth]{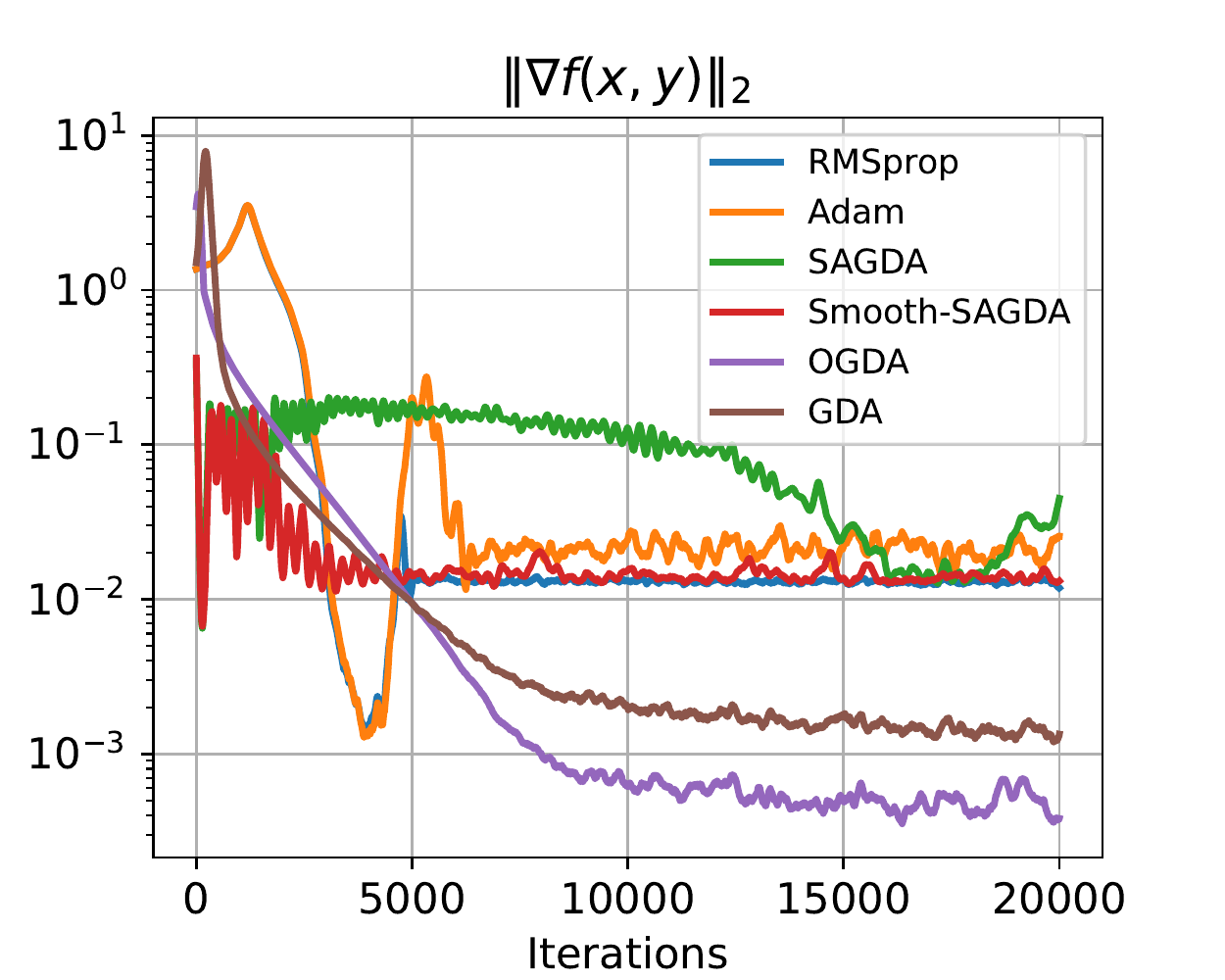}
  \caption{}
  \label{fig:lw}
\end{subfigure}%
\begin{subfigure}{.33\textwidth}
  \centering
  \includegraphics[width=.99\linewidth]{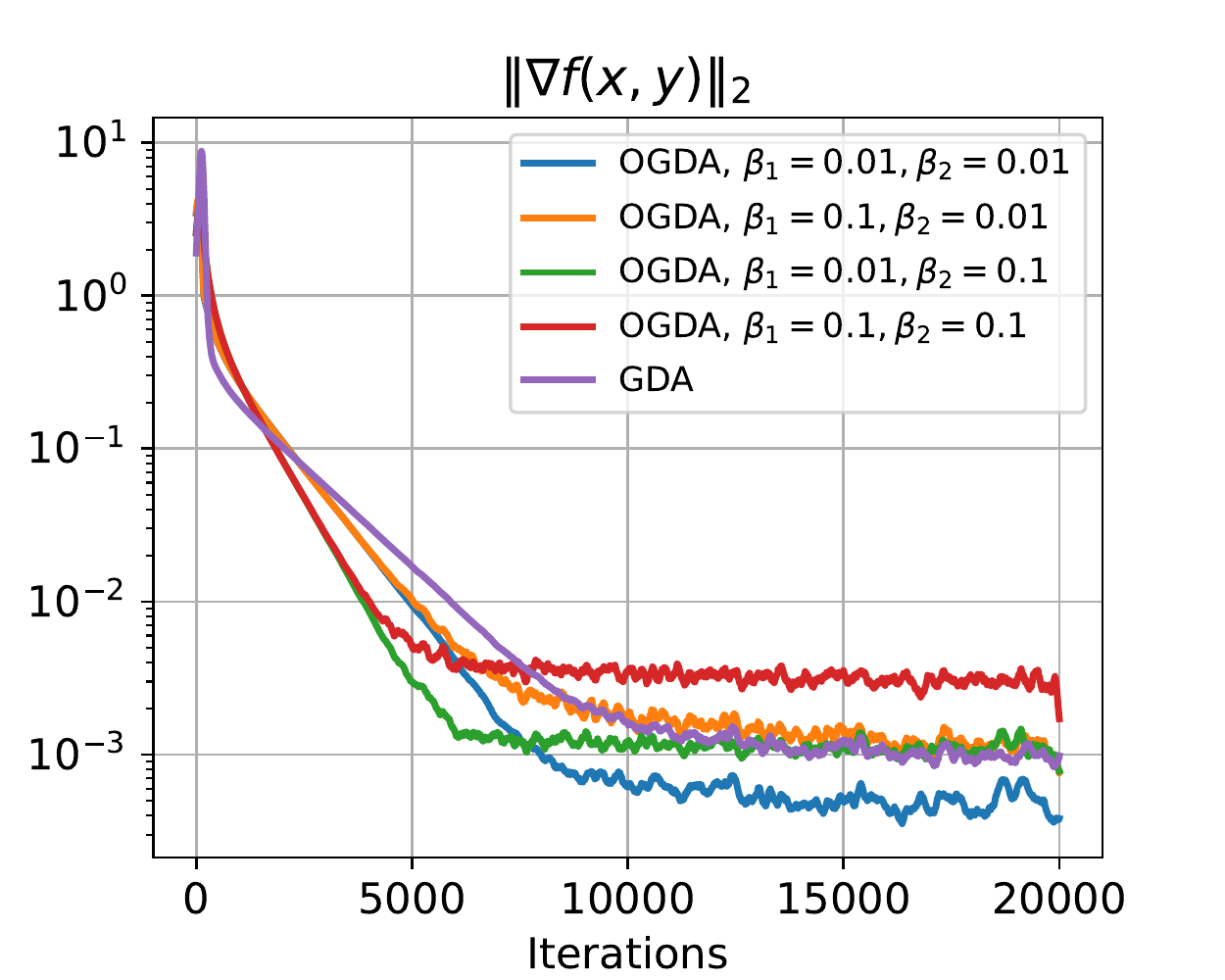}
  \caption{}
  \label{fig:beta} 
\end{subfigure}
\begin{subfigure}{.33\textwidth}
   \centering
  \includegraphics[width=.99\linewidth]{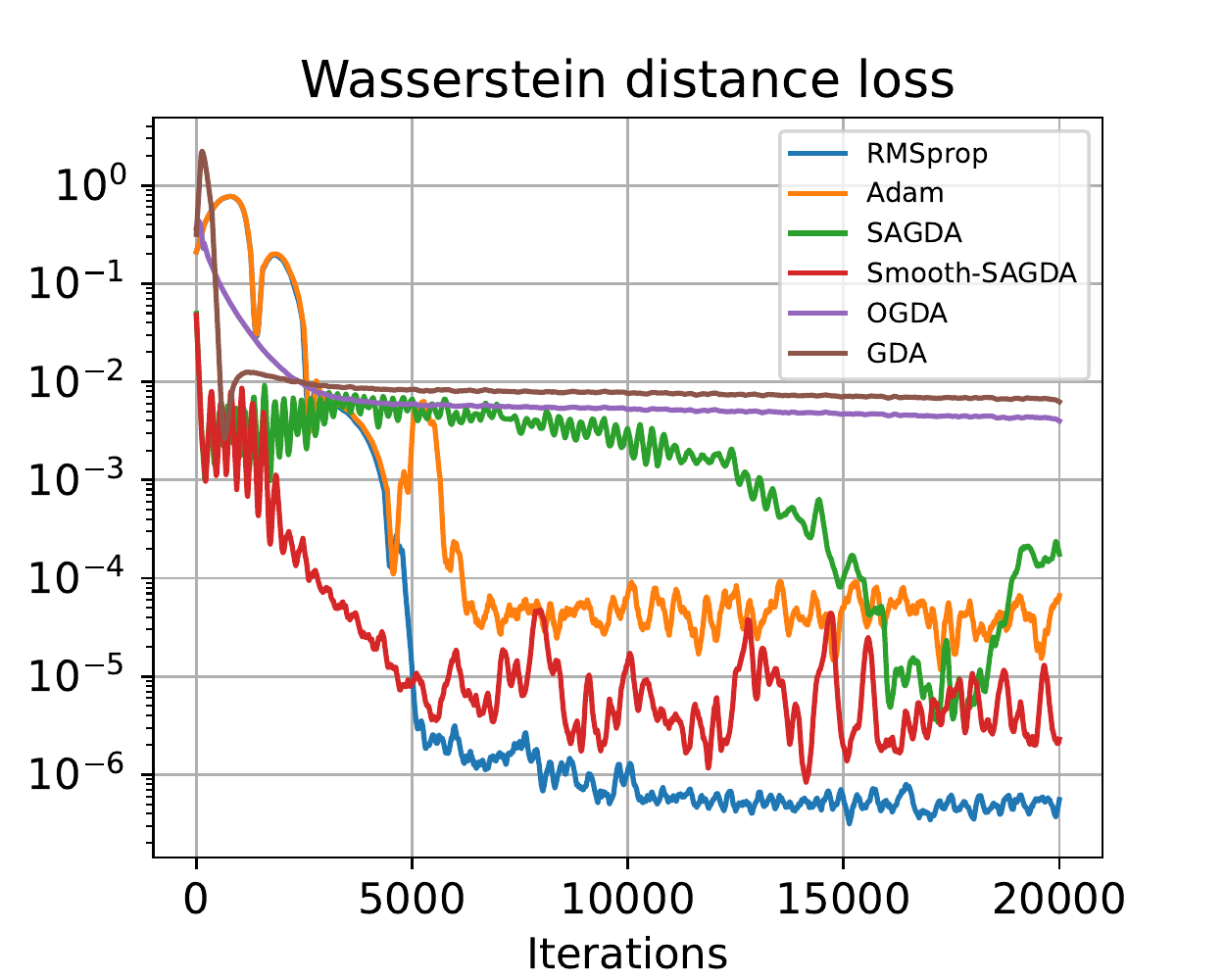}
  \caption{}
  \label{fig:WD}
\end{subfigure}
\caption{Figure~\ref{fig:lw} demonstrates  the best performance of different algorithms on optimizing  NC-SC objective in WGAN, where $\|\nabla f(\vx,\vy)\|^2 = \| \nabla_x f(\vx,\vy)\|^2 + \|\nabla_y f(\vx,\vy) \|^2$. For GDA, and OGDA, $\eta_x$, and $\eta_y$ chosen from the set $\{5e-5 , 1e-4 , 5e-4 , 1e-3 , 5e-3 , 1e-2 , 5e-2\}$ using grid search. For OGDA, we choose correction term ratios from the set $\{0 , 0.01 , 0.1 , 0.5, 1\}$. The optimal learning rates are as follows. For both OGDA, and GDA, we set  $\eta_x = \eta_y = 0.05$, and for OGDA $\beta_1 = \beta_2 =0.01$. For other algorithms, we used the same hyperparameters as reported in~\cite{yang2021faster}, using the same random seed. Figure~\ref{fig:beta} indicates effect of tuning correction term ratio $\beta$ on the performance of generalized OGDA algorithm. Figure~\ref{fig:WD} indicates the evaluation of the Wasserstein distance metric during the training for the best hyperparameter configuration.}\vspace{-0.5cm}
\label{fig:test}
\end{figure}

\section{Conclusion}
In this paper, we established the convergence of Optimistic Gradient Descent Ascent (OGDA) and  Extra-gradient (EG) methods in solving nonconvex minimax optimization problems. We demonstrated that  both methods exhibit the same convergence rate that is achievable by
GDA in both stochastic and deterministic settings. We also derived matching lower bounds for the choice of parameters that indicate the tightness of obtained rates. Further, we established general lower bounds (i.e, learning rate-independent) for GDA/EG/OGDA in the NC-SC setting, indicating the optimality of  obtained upper bounds up to the factor of $\kappa$. It remains an interesting future work to extend the lower bound results to the stochastic setting and also derive the general lower bound for GDA/EG/OGDA in the NC-C setting. Moreover, there is a gap by a factor of $\kappa$ between our lower and upper bounds for NC-SC problems, which would also be an interesting future work to close this gap.

\section*{Acknowledgements}

This work was supported in part by NSF grant CNS 1956276. We also would like to thank Mohammad Mahdi Kamani for
his help on conducting the experiments.

\clearpage
\newpage
\appendix

\addcontentsline{toc}{section}{Appendix} 

\part{Appendix} 
{
In the appendix, we provide the missing proofs and derivations from the main manuscript, as well as proposing a general variant of the OGDA algorithm where different learning rates can be employed in primal and dual updates. \\

\parttoc 
}
\newpage

\section{Proof of Convergence in Nonconvex-Strongly-Concave Setting}\label{app:NCSC_Upper}

\subsection{Proof of Convergence of  OGDA}\label{app:NCSC_Upper_OGDA}
Here we present the convergence proof for the OGDA algorithm in the NC-SC setting as detailed in Algorithm~\ref{alg:ogda}. Note that it is clear from context we abuse the notation and use $\bm{y}^*_t$ instead of $\bm{y}^*(\vx_t)$. In the following, we provide a proof sketch, making our analysis easier to follow.

Algorithm~\ref{alg:ogda} shows the deterministic and stochastic variants of the OGDA algorithm in detail.

\begin{algorithm}[H] 
\caption{(Stochastic) OGDA}
\label{alg:ogda}
\SetKwInOut{Input}{Input}
\SetKwInOut{Output}{Output}
\Input{Initialization $(\bm{x}_{-1}=\bm{x}_0,\bm{y}_{-1}=\bm{y}_0)$, learning rates $\eta_x, \eta_y$}
\For{$t=1,2,\dots,T$}{

 \algemph{antiquewhite}{0.95}{ 
 $\left. \bm{x}_t  = \bm{x}_{t-1} -  \eta_x  \nabla_x  f(\bm{x}_{t-1},\bm{y}_{t-1})  + \eta_x(\nabla_x f(\bm{x}_{t-1},\bm{y}_{t-1})-\nabla_x f(\bm{x}_{t-2},\bm{y}_{t-2}) )\right.$,  \\
 $\left.\bm{y}_t =   \bm{y}_{t-1} +  \eta_y( \nabla_y f(\bm{x}_{t-1},\bm{y}_{t-1}) -\eta_y( \nabla_y f(\bm{x}_{t-1},\bm{y}_{t-1})- \nabla_y f(\bm{x}_{t-2},\bm{y}_{t-2}))\right.$.  \hfill   \# OGDA  }
 
 \algemph{blizzardblue}{0.95}{ 
 $\bm{x}_t =   \bm{x}_{t-1} -  \eta_y \vg_{x,t-1}   +\eta_y(\vg_{x,t-1} - \vg_{x,t-2} )$,  \\
 $\bm{y}_t =   \bm{y}_{t-1} +  \eta_y \vg_{y,t-1}  -\eta_y(\vg_{y,t-1} - \vg_{y,t-2} )$.  \hfill   \# Stochastic OGDA}
} 
\end{algorithm}

\paragraph{Proof sketch.} We provide a sketch of key technical ideas. Specifically, we develop three key lemmas to prove the convergence. First lemma is primal descent, in which we use the $\kappa \ell$-smoothness property of $\Phi(\vx)$ at point $\vx_t$ and $\vx_{t-1}$ to find an upper bound for $\E[\Phi(\vx_t) - \Phi(\vx_{t-1})]$, and then by taking summation on this upper bound for all $t \in \{1, \dots , T\}$ we are able to show the following: 
\begin{equation}
\begin{split}
    \E[\Phi(\vx_T)] - \Phi(\vx_1) & \le - \frac{\eta_x}{2} \sum_{i=1}^{T-1} \E[\| \nabla \Phi(\vx_i) \|^2]  +  O(\eta_x \ell^2)  \\
    & \quad + O(\eta_x \ell^2)  \left (\sum_{i=1}^{T-1} \| \E[\vy_i - \vy^*_i \|^2] + \sum_{i=1}^{T-1} \E[\|\vy_i - \vy_{i-1} \|^2] \right ) \\
    &\quad - \frac{\eta_x}{2}( 1 - O(\eta_x)) \sum_{i=1}^{T-2} \E[\| \vg_i\|^2] + O\left(\eta_x \frac{T \sigma^2}{M_x}\right) 
\end{split}\label{eqn:primal}
\raisetag{2\normalbaselineskip}
\end{equation}
where $\vg_i = 2 \nabla_x f(\vx_i,\vy_i) - \nabla_x f(\vx_{i-1},\vy_{i-1})$.

The second key lemma is dual descent. To derive this lemma, first note that OGDA alternatively can be written in view of Past Extra-gradient algorithm (PEG) as defined in~\cite{hsieh2019convergence}: 
\begin{equation}
\begin{split}
    \vy_t = \vz_t   + \eta_y \vg_{y,t-1} \quad , \quad \vz_{t+1} = \vz_t + \eta_y \vg_{y,t}
\end{split}
\tag{Dual update}
\end{equation}
where $\vz_t = \vy_{t-1} + \eta_{y} ( \vg_{y,t-1} -  \vg_{y,t-2})$. Also, we have the following primal update: 
\begin{equation}
\begin{split}
    \vx_t = \vw_t  - \eta_x \vg_{x,t-1} \quad , \quad \vw_{t+1} = \vw_t - \eta_x \vg_{x,t}
\end{split}
\tag{Primal update}
\end{equation}
where $\vw_t = \vx_{t-1}- \eta_{x} (\vg_{x,t-1} - \vg_{x,t-2})$. This view of OGDA is presented in \cite{hsieh2019convergence,gidel2018variational,mokhtari2020unified}. Motivated by this interpretation of the OGDA algorithm, we define the following potential function to derive the dual descent. Let $ \vr_t = \| \vz_{t+1}- \vy^*_t \|^2 +  \frac{1}{4} \|\vy_t - \vy_{t-1}\|^2$, and $\eta_y = \frac{1}{6 \ell}$, then we show that:
\begin{equation*}
\begin{split}
    \E[\vr_t] &\le (1 - \frac{1}{12 \kappa}) \E[\vr_{t-1}] +  
    O(\eta_x^2) \E[\| \vg_{t-2 }\|^2] + O(\eta_x^2 \kappa^3) \E[\| \vg_{t-1} \|^2] + O\left(\frac{\sigma^2}{\ell^2 M_y}\right).
\end{split}
\end{equation*}
 We built on the top of OGDA analysis in~\cite{mokhtari2020unified,hsieh2019convergence} in strongly-concave-strongly-concave setting to prove the above lemma, which helps us directly find an upper bound for $\sum_{i=1}^{T-1} \|\vy_i - \vy_{i-1}\|^2$ in Equation~\ref{eqn:primal}.

Our third key lemma  aims to upper bound  $\sum_{i=1}^{T-1} \E[\|\vy_i - \vy^*_i\|^2]$ in terms of $\sum_{i=1}^{T-1} \E[\vr_i]$. Particularly we show that: 

\begin{equation*}
\begin{split}
&\sum_{i=1}^{T-1} \E[\|\vy_i - \vy^*_i \|^2]  \le \left(\|\vy_0 - \vy^*_0 \|^2 + \sum_{i=2}^{T-1} \E[\vr_i] + \eta_x^2 \kappa^2 \sum_{i=1}^{T-2} \E[\|\vg_i\|^2] + \frac{T \sigma^2}{\ell^2 M_y}\right).
\end{split}
\end{equation*}

Now note that using second, and third lemma both $\sum_{i=1}^{T-1} \E[\|\vy_i - \vy_{i-1} \|^2$], and $\sum_{i=1}^{T-1}   \E[\|\vy_i - \vy^*_i \|^2] $ terms can be upper bounded in terms of $\sum_{i=1}^{T-2} \E[\| \vg_i\|^2]$, and by properly choosing $\eta_x$ we show that $\sum_{i=1}^{T-2} \E[\|\vg_i\|^2]$ term can be ignored, which entails the  desired convergence rate.

\subsubsection{Useful lemmas}

\begin{lemma} [Lemma 4.3 in~\cite{lin2020gradient}] \label{app:lemma:smooth}
Let $\Phi(\vx) = \max_{\vy} f(\vx, \vy)$, and $\vy^*(\vx) =  \arg \max_{\vy} f(\vx,\vy)$. Then, under Assumption~\ref{asm:1}, $\Phi(\vx)$ is $\kappa \ell + \ell$-smooth, and $\vy^*(\vx)$ is $\kappa$ Lipschitz.
\end{lemma}

\begin{lemma}
\label{lemma:a2}
Let $\{a_t\}_{t=0}^{\infty}$, $\{b_t\}_{t=0}^{\infty}$ be the sequence of positive real valued number, and $\gamma \in (2,\infty)$ such that $\forall t\ge 1$: 
\begin{equation}
\label{eqn:a2}
a_t \le (1 - \frac{1}{\gamma}) a_{t-1} + b_t    
\end{equation}
then the following inequality holds for any $t_1 > t_2 \ge 0 $: 
\begin{equation}
\sum_{i = t_1}^{t_2} a_i \le \gamma a_{t_1} + \gamma \sum_{i=t_1 +1}^{t_2} b_i    
\end{equation}
\end{lemma}

\begin{proof}[Proof of Lemma~\ref{lemma:a2}]
Unfolding the recursion in Equation~\ref{eqn:a2} for $t-t_1$ steps we have:
\begin{equation}
a_{t} \le (1-\frac{1}{\gamma} )^{t - t_1} a_{t_1} + \sum_{i = t_1 +1}^{t} (1 - \frac{1}{\gamma})^{t - i} b_i
\end{equation}
Now taking summation of above equation we have:
\begin{equation}
\label{eqn:a2.1}
\sum_{t= t_1}^{t_2} a_t \le \left ( \sum_{t = t_1}^{t_2} (1 - \frac{1}{\gamma})^{t- t_1} \right)  a_{t_1}  + \sum_{t=t_1 +1}^{t_2} \, \sum_{i = t_1 +1 }^t ( 1 - \frac{1}{\gamma} )^{t-i} b_i
\end{equation}
However note that, we can write: 
\begin{equation}
\begin{split}
 \sum_{t=t_1 +1}^{t_2} \, \sum_{i = t_1 +1 }^t ( 1 - \frac{1}{\gamma} )^{t-i} b_i = \sum_{i=t_1+1}^{t_2} \left ( b_i \sum_{j=0}^{t_2 - i} (1 - \frac{1}{\gamma})^j \right ) &=  \sum_{i=t_1+1}^{t_2} b_i \frac{1 - (1 -\frac{1}{\gamma})^{t_2-i +1}}{1 - (1 - \frac{1}{\gamma})}  \\
 &\le \gamma \sum_{i=t_1+1}^{t_2} b_i
\end{split}
\end{equation}
\end{proof}
Plugging this back to Equation~\ref{eqn:a2.1}, and noting that $\sum_{t = t_1}^{t_2} (1 - \frac{1}{\gamma})^{t- t_1}  = \frac{1 - (1 -\frac{1}{\gamma})^{t_2-t_1 +1}}{1 - (1 - \frac{1}{\gamma})} \le \gamma$, we have:
\begin{equation}
\sum_{t=t_1}^{t_2} a_t \le \gamma a_{t_1} + \gamma \sum_{i = t_1 + 1}^{t_2} b_i    
\end{equation}

\begin{lemma}
\label{lemma:a3}
Let $\vy_{t+1} = \vy_t + \eta_y  \vg_{y,t}$, where $\vg_{y,t}$ is the unbiased estimator of $\, \nabla_y f(\vx_t , \vy_t)$. If $\eta_y \le \frac{1}{2 \ell}$, we have: 
\begin{equation}
\| \vy_{t+1} - \vy^*_t \|^2 \le (1 - \eta_y \mu) \|\vy_t - \vy^*_t \|^2 + 2 \eta_y^2 \| \vdelta^y_t \|^2 + 2 \eta_y \ip{\vdelta^y_t , \vy_t - \vy^*_t} 
\end{equation}
where $\vdelta^y_t = \vg_{y,t}- \nabla_y f(\vx_t,\vy_t)$.
\end{lemma}

\begin{proof}[Proof of Lemma~\ref{lemma:a3}]
Using the update rule for $\vy_{t+1}$, we can write:
\begin{equation}
    \|\vy_{t+1} - \vy^*_t\|^2 = \| \vy_t - \vy^*_t + \eta_y \vg_{y,t}\|^2 = \| \vy_t - \vy^*_t \|^2 + \eta_y^2 \| \vg_{y,t} \|^2 + 2\eta_y \ip{\vy_t - \vy^*_t , \vg_{y,t}} 
\end{equation}

Now replacing $\vg_{y,t} = \vdelta^y_t + \nabla_y f (\vx_t , \vy_t)$, and using Young's inequality we have: 
\begin{equation}
\label{eqn:a31}
\begin{split}
\|\vy_{t+1} - \vy^*_t \|^2 &\le \|\vy_t - \vy^*_t \|^2 + 2 \eta_y^2 \| \nabla_y f(\vx_t , \vy_t)\|^2 + 2 \eta_y \ip{\nabla_y f(\vx_t , \vy_t) ,\vy_t - \vy^*_t} \\
& \quad + 2 \eta_y^2 \| \vdelta^y_t \|^2 + 2 \eta_y \ip{ \vdelta^y_t , \vy_t - \vy^*_t}
\end{split}
\end{equation}
However, note that since $f(\vx,.)$ is $\mu$-strongly-concave, and $\ell$-smooth, we have:
\begin{equation}
\label{eqn:a32}
\begin{split}
    \ip{\nabla_y f(\vx_t , \vy_t) ,\vy_t - \vy^*_t} &\le - \frac{1}{\ell + \mu} \|\nabla_y f(\vx_t , \vy_t)\|^2 - \frac{\ell \mu}{\ell + \mu} \|\vy_t - \vy^*_t \|^2 \\
    &\le - \frac{1}{2 \ell} \|\nabla_y f(\vx_t , \vy_t)\|^2 - \frac{ \mu}{2} \|\vy_t - \vy^*_t \|^2,
\end{split}
\end{equation}
where in the last inequality, we used the fact that $\kappa \ge 1$, which means that $\ell \ge \mu$. Plugging Equation~\ref{eqn:a32} back to Equation~\ref{eqn:a31}, we have: 
\begin{equation}
 \begin{split}
\|\vy_{t+1} - \vy^*_t \|^2 &\le  (1 - \mu \eta_y) \|\vy_t - \vy^*_t \|^2 - \eta_y(\frac{1}{\ell}- 2 \eta_y) \| \nabla_y f(\vx_t,\vy_t)\|^2  \\
& \quad+ 2 \eta_y^2 \| \vdelta^y_t \|^2 + 2 \eta_y \ip{ \vdelta^y_t , \vy_t - \vy^*_t}
 \end{split}   
\end{equation}
Since $\eta_y \le \frac{1}{2 \ell}$, we have: 
\begin{equation}
 \begin{split}
\|\vy_{t+1} - \vy^*_t \|^2 &\le  (1 - \mu \eta_y) \|\vy_t - \vy^*_t \|^2 + 2 \eta_y^2 \| \vdelta^y_t \|^2 + 2 \eta_y \ip{ \vdelta^y_t , \vy_t - \vy^*_t}
 \end{split}   
\end{equation}

\end{proof}

\subsubsection{Key lemmas, and proof of Theorem~\ref{thm:ncsc_ogda}, and~\ref{thm:ncsc_sogda} for OGDA}

For the sake of brevity, we only present the convergence proof for the stochastic version of OGDA (Theorem~\ref{thm:ncsc_sogda}), since by letting $\sigma = 0$, we can recover the proof for the deterministic algorithm (Theorem~\ref{thm:ncsc_ogda}). Our proof is built on three key lemmas. First, we prove the following lemma, which we call primal descent: 
\begin{lemma}
\label{lemma:3}
Let $\Phi (\vx) = \max_{\vy} f(\vx,\vy)$, and $\vy^*(\vx)= \arg \max_{\vy} f(\vx,\vy)$. Also, let $\vg_{i}= 2 \vg_{x,i}- \vg_{x,i-1}$. Then for Algorithm~\ref{alg:ogda}, we have: 
\begin{equation}
\label{eqn:lm1}
\begin{split}
\E[\Phi(\vx_t)] &\le \E[\Phi(\vx_{t-1})] - \frac{\eta_x}{2} \E[\| \nabla \Phi(\vx_{t-1})\|^2] - \frac{\eta_x}{2}(1- 2 \kappa \ell \eta_x) \E[\|\vg_{t-1}\|^2] + \frac{3}{2} \eta_x^3 \ell^2 \E[\|\vg_{t-2}\|^2] \\
&\quad+ \frac{3}{2}\eta_x \ell^2  \E[\|\vy^*_{t-1} - \vy_{t-1} \|^2] + \frac{3}{2}\eta_x \ell^2 \E[\|\vy_{t-1} - \vy_{t-2}\|^2] + 15 \eta_x \frac{\sigma^2}{M_x}
\end{split}
\end{equation}
\end{lemma}

\begin{proof}[Proof of Lemma~\ref{lemma:3}]
First, let $\vdelta^x_i = \vg_{x,i} - \nabla_x f(\vx_i,\vy_i)$. By definition of $\vg_{x,i}$, we have $\E[\vdelta^x_i] = \mathbf{0}$, for all $i \in [T]$.

Using the fact that $ \Phi(\vx)$ is $ 2 \kappa \ell$ smooth, we have: 
\begin{equation}
\label{eqn:ogda1}
\begin{split}
\Phi(\vx_t) &\le \Phi(\vx_{t-1})  + \ip{\nabla \Phi(\vx_{t-1}) , \vx_t - \vx_{t-1} } + \kappa \ell \|\vx_t - \vx_{t-1}\|^2 \\
&= \Phi(\vx_{t-1})  -\eta_x \ip{\nabla \Phi(\vx_{t-1}) , \vg_{t-1} }  + \kappa \ell \eta_x^2 \| \vg_{t-1}\|^2 \\
&= \Phi(\vx_{t-1}) - \frac{\eta_x}{2} \| \nabla \Phi(\vx_{t-1})\|^2 - \frac{\eta_x}{2} \|\vg_{t-1}\|^2 + \frac{\eta_x}{2} \| \nabla \Phi (\vx_{t-1}) - \vg_{t-1} \|^2 +  \kappa \ell \eta_x^2 \|\vg_{t-1}\|^2 \\
&= \Phi(\vx_{t-1}) - \frac{\eta_x}{2} \| \nabla \Phi(\vx_{t-1})\|^2 - \frac{\eta_x}{2}(1- 2 \kappa \ell \eta_x) \|\vg_{t-1}\|^2 + \frac{\eta_x}{2} \| \nabla \Phi (\vx_{t-1}) - \vg_{t-1} \|^2 \\
\end{split}
\end{equation}
Now using $\ell$-smoothness of $f$, and $\kappa$-Lipschitzness of $\vy^*(\vx)$ (Lemma~\ref{app:lemma:smooth}) we have: 
\begin{equation}
\label{eqn:ogda2}
\begin{split}
 \| \nabla \Phi (\vx_{t-1}) - \vg_{t-1} \|^2  &= \| \nabla \Phi(\vx_{t-1}) - \nabla_x f(\vx_{t-1},\vy_{t-1}) \\
 &- \left( \nabla_x f(\vx_{t-1},\vy_{t-1}) -  \nabla_x f(\vx_{t-2},\vy_{t-2})\right)  - (2 \vdelta^x_{t-1} - \vdelta^x_{t-2}) \|^2 \\
 &\le 3  \| \nabla \Phi(\vx_{t-1}) - \nabla_x f(\vx_{t-1},\vy_{t-1})\|^2 + 3 \| \nabla_x f(\vx_{t-1},\vy_{t-1}) \\
 &-  \nabla_x f(\vx_{t-2},\vy_{t-2})\|^2 + 3 \|2 \vdelta^x_{t-1} - \vdelta^x_{t-2} \|^2 \\
 &\le 3 \ell^2 \| \vy^*(\vx_{t-1}) - \vy_{t-1} \|^2 +  3 \ell^2 \|\vx_{t-1} - \vx_{t-2}\|^2 + 3 \ell^2 \|\vy_{t-1} - \vy_{t-2}\|^2 \\
 & \quad + 24 \|\vdelta^x_{t-1}\|^2 + 6 \|\vdelta^x_{t-2}\|^2
\end{split}
\end{equation}
where in the first and second inequalities, we used Young's inequality.

By combining Equations~\ref{eqn:ogda1} and~\ref{eqn:ogda2} we have: 
\begin{equation}
\label{eqn:ogda3}
\begin{split}
\Phi(\vx_t) &\le \Phi(\vx_{t-1}) - \frac{\eta_x}{2} \| \nabla \Phi(\vx_{t-1})\|^2 - \frac{\eta_x}{2}(1- 2 \kappa \ell \eta_x) \|\vg_{t-1}\|^2 \\
&\quad+ \frac{3}{2}\eta_x \ell^2  \|\vy^*_{t-1} - \vy_{t-1} \|^2 + \frac{3}{2} \eta_x \ell^2 \|\vx_{t-1} - \vx_{t-2}\|^2 + \frac{3}{2}\eta_x \ell^2 \|\vy_{t-1} - \vy_{t-2}\|^2 \\ 
&\quad+ 12 \eta_x \|\vdelta_{t-1}^x\|^2 + 3 \eta_x \|\vdelta^x_{t-2}\|^2\\
&\le \Phi(\vx_{t-1}) - \frac{\eta_x}{2} \| \nabla \Phi(\vx_{t-1})\|^2 - \frac{\eta_x}{2}(1- 2 \kappa \ell \eta_x) \|\vg_{t-1}\|^2 + \frac{3}{2} \eta_x^3 \ell^2 \|\vg_{t-2}\|^2 \\
&\quad+ \frac{3}{2}\eta_x \ell^2  \|\vy^*_{t-1} - \vy_{t-1} \|^2 + \frac{3}{2}\eta_x \ell^2 \|\vy_{t-1} - \vy_{t-2}\|^2 +12 \eta_x \|\vdelta_{t-1}^x\|^2 + 3 \eta_x \|\vdelta^x_{t-2}\|^2
\end{split}
\end{equation}

We proceed by taking expectations on both sides of Equation~\ref{eqn:ogda3} to get: 
\begin{equation}
\label{eqn:ogda4}
\begin{split}
\E[\Phi(\vx_t)] &\le \E[\Phi(\vx_{t-1})] - \frac{\eta_x}{2} \E[\| \nabla \Phi(\vx_{t-1})\|^2] - \frac{\eta_x}{2}(1- 2 \kappa \ell \eta_x) \E[\|\vg_{t-1}\|^2] + \frac{3}{2} \eta_x^3 \ell^2 \E[\|\vg_{t-2}\|^2] \\
&\quad+ \frac{3}{2}\eta_x \ell^2  \E[\|\vy^*_{t-1} - \vy_{t-1} \|^2] + \frac{3}{2}\eta_x \ell^2 \E[\|\vy_{t-1} - \vy_{t-2}\|^2] + 15 \eta_x \frac{\sigma^2}{M_x}
\end{split}
\end{equation}
where we used the fact that $\E[\|\vdelta^x_i\|^2] \le \frac{\sigma^2}{M_x}$ for all $ i \in [T]$. 

\end{proof}

\begin{lemma}
\label{lemma:4}
Let $\eta_y = \frac{1}{6 \ell}$, then the following inequality holds true for OGDA iterates:
\begin{equation}
\label{eqn:lm2}
\begin{split}
\sum_{i=1}^{t+1} \E[\| \vy_i - \vy^*_i\|^2 ] \le \frac{9}{7} \E [\|\vy_1 - \vy^*_1\|^2] + \frac{36}{7} \sum_{i=2}^{t+1} \E[ \|\vz_{i} - \vy^*_i \|^2] + \frac{18}{7} \eta_x^2 \kappa^2 \sum_{i=1}^t \E[\| \vg_i\|^2] + \frac{2T \sigma^2}{ 7 \ell^2 M_y}
\end{split}    
\end{equation}

\end{lemma}

\begin{proof}[Proof of Lemma~\ref{lemma:4}]
Using Young's inequality and $\kappa$-Lipschitzness of $\vy^*(\vx)$,  we have:
\begin{equation}
\label{eqn:ogda5}
\begin{split}
\| \vy_{t+1} - \vy^*_{t+1} \|^2 &\le 2 \| \vy_{t+1} - \vy^*_t \|^2 + 2 \| \vy^*_{t+1} - \vy^*_t \|^2 \\
&\le  2 \| \vy_{t+1} - \vy^*_t \|^2 + 2 \kappa^2 \| \vx_{t+1} - \vx_t \|^2 
\end{split}
\end{equation}
Now, we try to find an upper bound for $\|\vy_{t+1} - \vy^*_t \|^2$. Let $\vz_{t+1} = \vy_t + \eta_y (\vg_{y,t} - \vg_{y,t-1})$, and $\vdelta^y_i = \vg_{y,i} - \nabla_y f(\vx_i,\vy_i) $. Then we have:
\begin{equation}
\label{eqn:ogda6}
\begin{split}
\|\vy_{t+1} - \vy^*_t \|^2 &= \| \vz_{t+1} - \vy^*_t + \eta_y \vg_{y,t} \|^2 \\
&\le 2 \| \vz_{t+1} - \vy^*_t \|^2  + 2 \eta_y^2 \| \vg_{y,t}\|^2 \\
&\le 2 \| \vz_{t+1} - \vy^*_t \|^2 + 4 \eta_y^2 \| \nabla_y f(\vx_t , \vy_t) \|^2+ 4\eta_y^2 \| \vdelta_t^y\|^2 \\
&\le 2 \| \vz_{t+1} - \vy^*_t \|^2 + 4 \eta_y^2 \ell^2 \| \vy_t - \vy^*_t \|^2+ 4\eta_y^2 \| \vdelta_t^y\|^2
\end{split}
\end{equation}
where in the first and second inequality, we used Young's inequality, and for the last inequality, we used smoothness of $f$. Now,   replacing replacing the choice $\eta_y = \frac{1}{6 \ell}$ in Equation~\ref{eqn:ogda6} yields: 
\begin{equation}
\label{eqn:ogda7}
\| \vy_{t+1} - \vy^*_t \|^2 \le \frac{1}{9} \| \vy_t - \vy^*_t \|^2 + 2 \|\vz_{t+1} - \vy^*_t \|^2 + \frac{1}{9 \ell^2} \|\vdelta^y_t\|^2
\end{equation}
Now plugging Equation~\ref{eqn:ogda7} in Equation~\ref{eqn:ogda5} we have:
\begin{equation}
\label{eqn:ogda8}
\| \vy_{t+1} - \vy^*_{t+1} \|^2 \le \frac{2}{9} \|\vy_t - \vy^*_t \|^2 + 4 \|\vz_{t+1}-\vy^*_t \|^2 + 2 \kappa^2 \|\vx_{t+1} -\vx_t \|^2 + \frac{2}{9 \ell^2} \| \vdelta^y_t \|^2     
\end{equation} 
Now taking expectations from both sides of Equation~\ref{eqn:ogda8}, we have:
\begin{equation}
\label{eqn:ogda9}
\E[\| \vy_{t+1} - \vy^*_{t+1} \|^2] \le \frac{2}{9} \E[\|\vy_t - \vy^*_t \|^2] + 4 \E[\|\vz_{t+1}-\vy^*_t \|^2] + 2 \kappa^2 \E[\|\vx_{t+1} -\vx_t \|^2] + \frac{2\sigma^2}{9 \ell^2 M_y}   
\end{equation}
Using Lemma~\ref{lemma:a2}, it can be easily shown that:
\begin{equation}
\label{eqn:ogda10}
\begin{split}
\sum_{i=1}^{t+1} \E[\| \vy_i - \vy^*_i\|^2 ] \le \frac{9}{7} \E [\|\vy_1 - \vy^*_1\|^2] + \frac{36}{7} \sum_{i=2}^{t+1} \E[ \|\vz_{i} - \vy^*_i \|^2] + \frac{18}{7} \eta_x^2 \kappa^2 \sum_{i=1}^t \E[\| \vg_i\|^2] + \frac{2T \sigma^2}{ 7 \ell^2 M_y}
\end{split}    
\end{equation}
\end{proof}

By extending the analysis in~\cite{mokhtari2020unified}  for OGDA from SC-SC to NC-SC, we derive the following lemma:
\begin{lemma}
\label{lemma:2}
Let $\vz_{t+1} = \vy_t + \eta_y (\vg_{y,t} - \vg_{y,t-1})$, $\vr_t = \|\vz_{t+1} - \vy^*_t \|^2 + \frac{1}{4} \|\vy_t - \vy_{t-1} \|^2 $ and $\eta_y = \frac{1}{6 \ell}$. Then OGDA iterates satisfy the following inequalities:
\begin{equation}
\label{eqn:lm3.1}
\begin{split}
 \E[\vr_t]   &\le  \left(1 - \frac{1}{12 \kappa}\right)  \E[\vr_{t-1}]  + 12 \eta_x^2 \kappa^3 \E[\|\vg_{t-1}\|^2]   + \frac{\eta_x^2}{18} \E[\|\vg_{t-2}\|^2]  + \frac{\sigma^2}{3 \ell^2 M_y}
\end{split}
\end{equation}
and
\begin{equation}
\label{eqn:lm3.2}
\begin{split}
\sum_{i=1}^{t} \E[\vr_i] \le 12 \kappa \E[\vr_1] + \frac{2}{3}  \kappa \E[\|\vx_1 - \vx_0 \|^2] +  145 \eta_x^2 \kappa^4  \sum_{i=1}^{t-1} \E[\|\vg_{i} \|^2]    + \frac{4 \kappa \sigma^2 (t-1)}{ \ell^2 M_y}.
\end{split}
\end{equation}
\end{lemma}

\begin{proof}[Proof of Lemma~\ref{lemma:2}]

Let $\vdelta^y_i = \vg_{y,i} - \nabla_y f(\vx_i , \vy_i)$, and note that we have $\vz_{t+1} - \vz_{t} = \eta_y \vg_{y,t}$. We have:
\begin{equation}
\label{eqn:ogda11}
\begin{split}
    \| \vz_{t+1} - \vy^*_t \|^2  &= \| \vz_{t} - \vy^*_t + \eta_y \vg_{y,t} \|^2   \\
    &= \| \vz_{t} - \vy^*_t \|^2 + 2 \eta_y \ip{ \vg_{y,t} , \vz_{t} - \vy^*_t } + \eta_y^2 \| \vg_{y,t}\|^2 \\
    &= \| \vz_t - \vy^*_t \|^2 - 2 \eta_y^2 \ip{ \vg_{y,t}, \vg_{y,t-1} } + 2 \eta_y \ip{ \vg_{y,t} , \vy_t - \vy^*_t } +  \eta_y^2 \| \vg_{y,t} \|^2 \\
    &= \| \vz_t - \vy^*_t \|^2  + \eta_y^2 \| \vg_{y,t} - \vg_{y,t-1} \|^2 + 2 \eta_y \ip{ \vg_{y,t} , \vy_t - \vy^*_t } - \eta_y^2 \| \vg_{y,t-1}\|^2\\
    &\le \| \vz_t - \vy^*_t \|^2 + 3 \eta_y^2 \|\nabla_y f(\vx_t,\vy_t) - \nabla_y f(\vx_{t-1},\vy_{t-1}) \|^2 \\& \quad+ 2 \eta_y \ip{\nabla_y f(\vx_t , \vy_t) , \vy_t - \vy^*_t}  - \eta_y^2 \| \vg_{y,t-1} \|^2 \\&\quad +  3 \eta_y^2 \| \vdelta^y_t\|^2 + 3 \eta_y^2 \| \vdelta^y_{t-1}\|^2  +  2 \eta_y \ip{\vdelta^y_t, \vy_t - \vy^*_t}
    \\
    &\le  \| \vz_t - \vy^*_t \|^2 + 3 \eta_y^2 \ell^2 \|\vx_t - \vx_{t-1}\|^2 + 3 \eta_y^2 \ell^2 \|\vy_t - \vy_{t-1}\|^2 - 2 \eta_y \mu  \|\vy_t - \vy^*_t\|^2 \\
    &\quad- \eta_y^2 \|\vg_{y,t-1} \|^2 + 3 \eta_y^2 \| \vdelta^y_t\|^2 + 3 \eta_y^2 \| \vdelta^y_{t-1}\|^2  +  2 \eta_y \ip{\vdelta^y_t, \vy_t - \vy^*_t}
\end{split}
\end{equation}
where the last inequality follows from the smoothness of $f$ and strong concavity of $f(\vx_t,.)$. Now note that using Young's inequality, we can write: 
\begin{equation}
\label{eqn:ogda12}
    \|\vy_t - \vy^*_t\|^2 \ge \frac{1}{2} \| \vz_t - \vy^*_t\|^2 - \eta_y^2 \| \vg_{y,t-1} \|^2 
\end{equation}
Now plugging Equation~\ref{eqn:ogda12} back to Equation~\ref{eqn:ogda11}, we have: 
\begin{equation}
\label{eqn:ogda13}
\begin{split}
     \| \vz_{t+1} - \vy^*_t \|^2   &\le  (1 - \eta_y \mu)\| \vz_t - \vy^*_t \|^2 + 3 \eta_y^2 \ell^2 \|\vx_t - \vx_{t-1}\|^2 + 3 \eta_y^2 \ell^2 \|\vy_t - \vy_{t-1}\|^2 \\ &\quad - \eta_y^2 (1 - 2 \eta_y \mu)  \|\vg_{y,t-1} \|^2  + 3 \eta_y^2 \| \vdelta^y_t\|^2 + 3 \eta_y^2 \| \vdelta^y_{t-1}\|^2  +  2 \eta_y \ip{\vdelta^y_t, \vy_t - \vy^*_t}
\end{split}
\end{equation}
Now note that we have the following: 
\begin{equation}
\label{eqn:ogda14}
\begin{split}
    \|\vy_{t} - \vy_{t-1} \|^2 &= \eta_y^2 \| \vg_{y,t-1} + \vg_{y,t-1} - \vg_{y,t-2} \|^2 \\
    &\le 2 \eta_y^2 \|\vg_{y,t-1}\|^2 + 2 \eta_y^2 \|\vg_{y,t-1} - \vg_{y,t-2}\|^2 \\
    &\le 2 \eta_y^2 \|\vg_{y,t-1}\|^2 + 6 \eta_y^2 \|\nabla_y f(\vx_{t-1},\vy_{t-1}) - \nabla_y f(\vx_{t-2},\vy_{t-2})\|^2 \\
    & \quad + 6 \eta_y^2 \| \vdelta^y_{t-1}\|^2 +  6 \eta_y^2 \| \vdelta^y_{t-2}\|^2 
    \\
    &\le 2 \eta_y^2 \|\vg_{y,t-1}\|^2 + 6  \eta_y^2 \ell^2 \|\vx_{t-1}-\vx_{t-2}\|^2 + 6  \eta_y^2 \ell^2 \|\vy_{t-1}-\vy_{t-2}\|^2 \\
    & \quad +  6 \eta_y^2 \| \vdelta^y_{t-1}\|^2 +  6 \eta_y^2 \| \vdelta^y_{t-2}\|^2
\end{split}
\end{equation}

Now adding $ 9  \eta_y^2 \ell^2 \|\vy_t -\vy_{t-1}\|^2$ to both side of Equation~\ref{eqn:ogda13}, and using Equation~\ref{eqn:ogda14} we have:
\begin{equation}   
\label{eqn:ogda15}
\begin{split}
 \| \vz_{t+1} - \vy^*_t \|^2 +  9 \eta_y^2 \ell^2 \|\vy_t -\vy_{t-1}\|^2   &\le  (1 - \eta_y \mu)\| \vz_t - \vy^*_t \|^2 + 3 \eta_y^2 \ell^2 \|\vx_t - \vx_{t-1}\|^2  \\
 &\quad- \eta_y^2 (1 - 2 \eta_y \mu -24 \eta_y^2 \ell^2)  \|\vg_{y,t-1} \|^2 \\
 &\quad+ 72 \eta_y^4 \ell^4 \|\vx_{t-1}-\vx_{t-2}\|^2   + 72 \eta_y^4 \ell^4 \|\vy_{t-1}-\vy_{t-2}\|^2 \\
 &\quad+ 3 \eta_y^2(1 + 24 \eta_y^2 \ell^2) \| \vdelta^y_t\|^2 + 3 \eta_y^2 (1 + 24 \eta_y^2 \ell^2) \| \vdelta^y_{t-1}\|^2  \\
 &\quad +  2 \eta_y \ip{\vdelta^y_t, \vy_t - \vy^*_t}
\end{split}
\end{equation}

We proceed by plugging $\eta_y = \frac{1}{6 \ell}$ into Equation~\ref{eqn:ogda15}: 
\begin{equation}
\label{eqn:ogda16}
\begin{split}
 \| \vz_{t+1} - \vy^*_t \|^2 +   \frac{1}{4} \|\vy_t -\vy_{t-1}\|^2   &\le  \left(1 - \frac{1}{6 \kappa}\right) \left (\| \vz_t - \vy^*_t \|^2 \right)+  \frac{1}{18} \|\vy_{t-1} -\vy_{t-2}\|^2  \\
 & \quad + \frac{1}{12} \|\vx_t - \vx_{t-1}\|^2   + \frac{1}{18} \|\vx_{t-1}-\vx_{t-2}\|^2   \\
 &\quad + \frac{1}{6 \ell^2} \| \vdelta^y_t\|^2 +  \frac{1}{6 \ell^2} \| \vdelta^y_{t-1}\|^2  +  \frac{2}{6 \ell} \ip{\vdelta^y_t, \vy_t - \vy^*_t}
\end{split}
\end{equation}
Taking expectations from both sides of Equation~\ref{eqn:ogda16}, we have:
\begin{equation}
\label{eqn:ogda17}
\begin{split}
 \E \left[\| \vz_{t+1} - \vy^*_t \|^2 +   \frac{1}{4} \|\vy_t -\vy_{t-1}\|^2 \right]   &\le  \left(1 - \frac{1}{6 \kappa}\right) \E\left [\| \vz_t - \vy^*_t \|^2\right ] +  \frac{1}{18} \E[\|\vy_{t-1} -\vy_{t-2}\|^2]   \\
 &\quad+ \frac{1}{12} \E[\|\vx_t - \vx_{t-1}\|^2]   + \frac{1}{18} \E[\|\vx_{t-1}-\vx_{t-2}\|^2]   \\
 &\quad+ \frac{\sigma^2}{3 \ell^2 M_y}
\end{split}
\end{equation}
Also, using Young's inequality, we have: 
\begin{equation}
\label{eqn:ogda18}
\| \vz_t - \vy^*_t \|^2 \le (1+\frac{1}{12 \kappa}) \| \vz_{t} - \vy^*_{t-1} \|^2 + (1 + 12 \kappa) \kappa^2 \|\vx_t-\vx_{t-1}\|^2,  
\end{equation}
where we used the fact that for any $\alpha > 0$, $\|\vx+\vy\|^2 \le (1 + \alpha) \|\vx\|^2 + (1 + \frac{1}{\alpha}) \|\vy\|^2$, and $\kappa$-lipschitzness of $\vy^*(\vx)$. Plugging Equation~\ref{eqn:ogda18} back to Equation~\ref{eqn:ogda17}, we have:
\begin{equation}
\label{eqn:ogda19}
\begin{split}
 \E \left[\| \vz_{t+1} - \vy^*_t \|^2 +   \frac{1}{4} \|\vy_t -\vy_{t-1}\|^2 \right]   &\le  \left(1 - \frac{1}{12 \kappa}\right) \E\left [\| \vz_t - \vy^*_{t-1} \|^2 +  \frac{1}{4} \E[\|\vy_{t-1} -\vy_{t-2}\|^2] \right ]   \\
 & \quad+ 12 \kappa^3 \E[\|\vx_t - \vx_{t-1}\|^2]   + \frac{1}{18} \E[\|\vx_{t-1}-\vx_{t-2}\|^2]   \\
 &\quad+ \frac{\sigma^2}{3 \ell^2 M_y}
\end{split}
\end{equation}
Therefore, if we let $\vr_t = \|\vz_{t+1} - \vy^*_t \|^2 +   \frac{1}{4} \|\vy_t -\vy_{t-1}\|^2 $, then we have: 
\begin{equation}
\label{eqn:ogda20}
\begin{split}
 \E[\vr_t]   &\le  \left(1 - \frac{1}{12 \kappa}\right)  \E[\vr_{t-1}]  + 12 \eta_x^2 \kappa^3 \E[\|\vg_{t-1}\|^2]   + \frac{\eta_x^2}{18} \E[\|\vg_{t-2}\|^2]  + \frac{\sigma^2}{3 \ell^2 M_y}
\end{split}
\end{equation}

We can derive the following equation by applying Lemma~\ref{lemma:a2}. 
\begin{equation}
\label{eqn:ogda21}
\begin{split}
\sum_{i=1}^{t} \E[\vr_i] &\le 12 \kappa \E[\vr_1] + 144 \eta_x^2 \kappa^4  \sum_{i=1}^{t-1} \E[\|\vg_{i} \|^2] + \frac{2}{3} \eta_x^2 \kappa  \sum_{i=1}^{t-2} \E[\|\vg_{i} \|^2] + \frac{2}{3} \kappa \E[\|\vx_1 - \vx_0] \|^2 \\
& \quad + \frac{4 \kappa \sigma^2 (t-1)}{ \ell^2 M_y}
\end{split}
\end{equation}
Or equivalently, we have: 
\begin{equation}
\label{eqn:ogda22}
\begin{split}
\sum_{i=1}^{t} \E[\vr_i] \le 12 \kappa \E[\vr_1] + \frac{2}{3}  \kappa \E[\|\vx_1 - \vx_0 \|^2] +  145 \eta_x^2 \kappa^4  \sum_{i=1}^{t-1} \E[\|\vg_{i} \|^2]    + \frac{4 \kappa \sigma^2 (t-1)}{ \ell^2 M_y}
\end{split}
\end{equation}

\end{proof}

\begin{proof}[Proof of Theorem~\ref{thm:ncsc_ogda}, and Theorem~\ref{thm:ncsc_sogda} for OGDA]

We begin by taking summation of Equation~\ref{eqn:lm1} (Lemma~\ref{lemma:3}) from $t=2$ to $t = T$ which yields: 
\begin{equation}
\begin{split}
    \frac{\eta_x}{2} \sum_{i=1}^{T-1} \E[\| \nabla \Phi (\vx_i) \|^2] &\le \Phi(\vx_1) -\E[ \Phi(\vx_T)] + \frac{3}{2} \eta_x \ell^2 \|\vx_1 - \vx_0 \|^2 \\
    &\quad- \frac{\eta_x}{2} ( 1 - 2 \kappa \ell \eta_x ) \sum_{i=1}^{T-1} \E[\|\vg_i\|^2] + \frac{3}{2}\eta_x^3 \ell^2 \sum_{i=1}^{T-2} \E[\|\vg_i \|^2] \\
    &\quad+ \frac{3}{2} \eta_x \ell^2 \sum_{i=1}^{T-1} \| \vy_i - \vy^*_i\|^2  + 
     \frac{3}{2} \eta_x \ell^2 \sum_{i=1}^{T-1} \E[\|\vy_i - \vy_{i-1}\|^2] \\
     &\quad+ 15 \eta_x \frac{(T-1)\sigma^2}{M_x}
\end{split}
\end{equation}
We proceed by  noting that  if $\eta_x \le \frac{1}{2 \kappa \ell}$, then we can drop $\|\vg_{T-1}\|^2$ term in above equation. By considering this, and multiplying both  sides by $\frac{2}{\eta_x}$ we get (also let $\Delta_{\Phi} = \max ( \Phi(\vx_0) ,\Phi(\vx_1)) - \min_{\vx} \Phi(\vx) $) :
\begin{equation}
\begin{split}
     \sum_{i=1}^{T-1} \E[\| \nabla \Phi (\vx_i) \|^2] &\le \frac{2 \Delta_{\Phi}}{\eta_x} + 3  \ell^2 \| \vx_1 - \vx_0\|^2  \\
     &\quad- ( 1 - 2 \kappa \ell \eta_x - 3 \eta_x^2 \ell^2) \sum_{i=1}^{T-2} \E[\|\vg_i\|^2]  \\
    &\quad+  3 \ell^2 \sum_{i=1}^{T-1} \E[\|\vy^*_i - \vy_i\|^2]  + 
     3  \ell^2 \sum_{i=1}^{T-1} \E[\|\vy_i - \vy_{i-1}\|^2] +  30 \frac{(T-1)\sigma^2}{M_x} \\
\end{split}
\end{equation}
We can replace $\sum_{i=1}^{T-1} \|\vy_i^* - \vy_i\|^2$ with its upper bound obtained  in Lemma~\ref{lemma:4} to get: 
\begin{equation}
\begin{split}
     \sum_{i=1}^{T-1} \| \nabla \Phi (\vx_i) \|^2 &\le \frac{2 \Delta_{\Phi}}{\eta_x} +  3 \ell^2 \| \vx_1 - \vx_0 \|^2  + \frac{27}{7}   \ell^2 \|\vy_1 - \vy^*_1\|^2 \\
     &\quad- ( 1 - 2 \kappa \ell \eta_x - 3 \eta_x^2 \ell^2 - \frac{54}{7}\eta_x^2 \kappa^2 \ell^2) \sum_{i=1}^{T-2} \E[\|\vg_i\|^2]  \\
    &\quad+  \frac{108}{7} \ell^2 \sum_{i=2}^{T-1} \E[ \| \vz_i - \vy^*_{i-1} \|^2] + 
     3  \ell^2 \sum_{i=1}^{T-1} \E[\|\vy_i - \vy_{i-1}\|^2] + 30 \frac{(T-1)\sigma^2}{M_x}\\
     &\quad+ \frac{6}{7}\frac{(T-2)\sigma^2}{M_y}
\end{split}
\end{equation}
Now note that $\frac{108}{7}  \E[ \| \vz_{i+1} - \vy^*_{i} \|^2] + 
     3   \sum_{i=2}^{T-1} \E[\|\vy_i - \vy_{i-1}\|^2] \le 15.5 \E[\vr_i] $. Therefore we have: 
\begin{equation}
\begin{split}
     \sum_{i=1}^{T-1} \| \nabla \Phi (\vx_i) \|^2 &\le \frac{2 \Delta_{\Phi}}{\eta_x} +  3 \ell^2 \| \vx_1 - \vx_0 \|^2  + \frac{27}{7}   \ell^2 \|\vy_1 - \vy^*_1\|^2 \\
     &\quad- ( 1 - 2 \kappa \ell \eta_x - 3 \eta_x^2 \ell^2 - \frac{54}{7}\eta_x^2 \kappa^2 \ell^2) \sum_{i=1}^{T-2} \E[\|\vg_i\|^2]  \\
    &\quad+ 15.5 \ell^2 \sum_{i=1}^{T-1} \E[\vr_i] + 30 \frac{(T-1)\sigma^2}{M_x} + \frac{6}{7}\frac{(T-2)\sigma^2}{M_y} \\
\end{split}
\end{equation}
Furthermore, using Lemma~\ref{lemma:2}, we can find an upper bound on $\sum_{i=1}^{T-1} \E[\vr_i] $, and replacing it in above equation yields:
\begin{equation}
\begin{split}
     \sum_{i=1}^{T-1} \| \nabla \Phi (\vx_i) \|^2 &\le \frac{2 \Delta_{\Phi}}{\eta_x} + 186 \kappa \ell^2 \E[\vr_1] + 11 \kappa \ell^2 \|\vx_1 - \vx_0\|^2+  3 \ell^2 \| \vx_1 - \vx_0 \|^2  + \frac{27}{7}   \ell^2 \|\vy_1 - \vy^*_1\|^2 \\
     &\quad- ( 1 - 2 \kappa \ell \eta_x - 3 \eta_x^2 \ell^2 - \frac{54}{7}\eta_x^2 \kappa^2 \ell^2 - 2248 \eta_x^2 \kappa^4 \ell^2) \sum_{i=1}^{T-2} \E[\|\vg_i\|^2]  \\
    &\quad+ \frac{62 \kappa \sigma^2 (T-2)}{M_y} + 30 \frac{(T-1)\sigma^2}{M_x} + \frac{6}{7}\frac{(T-2)\sigma^2}{M_y} \\
\end{split}
\end{equation}
By letting $\eta_x = \frac{1}{50 \kappa^2 \ell}$, it holds that $-( 1 - 2 \kappa \ell \eta_x - 3 \eta_x^2 \ell^2 - \frac{54}{7}\eta_x^2 \kappa^2 \ell^2- 2248 \eta_x^2 \kappa^4 \ell^2) \sum_{i=1}^{T-2} \E[\|\vg_i\|^2] \le 0$. Therefore, with the choice of  letting rate $\eta_x = \frac{1}{50 \kappa^2 \ell}$ and simplifying the terms, we have: 
\begin{equation}
\label{eqn:ogdaf}
\begin{split}
     \frac{1}{T-1} \sum_{i=1}^{T-1} \E[\| \nabla \Phi (\vx_i) \|^2] &\le 100 \frac{\kappa^2 \ell \Delta_{\Phi}}{T-1}  +  186 \frac{\kappa \ell^2}{T-1} \|\vy_1 - \vy^*_1+ \eta_y (\vg_{y,1} - \vg_{y,0}) \|^2 \\
     &\quad+ 47 \frac{\kappa \ell^2}{T-1} \|\vy_1 - \vy_0 \|^2 + 14 \frac{\kappa \ell^2}{T-1} \| \vx_1 - \vx_0 \|^2  \\
     &\quad + \frac{27}{7} \frac{\ell^2}{T-1} \| \vy_1 - \vy^*_1 \|^2 + \frac{63 \kappa \sigma^2 }{M_y} + 30 \frac{\sigma^2}{M_x}  
\end{split}
\end{equation}
Using Young's inequality and $\ell$-smoothness of $f$, we have: 
\begin{equation}
\begin{split}
\|\vy_1 - \vy^*_1+ \eta_y (\vg_{y,1} - \vg_{y,0}) \|^2 \le 2 \|\vy_1 - \vy^*_1\|^2 + \frac{1}{18} \| \vy_1 - \vy_0 \|^2 + \frac{1}{18} \| \vx_1 - \vx_0 \|^2
\end{split}
\end{equation}
Plugging this into Equation~\ref{eqn:ogdaf}, we have: 
\begin{equation}
\label{eqn:ogdaff}
\begin{split}
     \frac{1}{T-1} \sum_{i=1}^{T-1} \E[\| \nabla \Phi (\vx_i) \|^2] &\le 100 \frac{\kappa^2 \ell \Delta_{\Phi}}{T-1}  +  376 \frac{\kappa \ell^2}{T-1} \|\vy_1 - \vy^*_1 \|^2 \\
     &\quad+ 58 \frac{\kappa \ell^2}{T-1} \|\vy_1 - \vy_0 \|^2 + 25 \frac{\kappa \ell^2}{T-1} \| \vx_1 - \vx_0 \|^2 \\
     &\quad+ \frac{63 \kappa \sigma^2 }{M_y} + 30 \frac{\sigma^2}{M_x}  
\end{split}
\end{equation}
Now by letting $M_x = \frac{\sigma^2}{\epsilon^2}$, $M_y = \frac{ \kappa \sigma^2}{\epsilon^2}$ and $D_0 = \max ( \|\vy_1 - \vy^*_1\|^2, \|\vx_1 - \vx_0\|^2, \|\vy_1 - \vy_0\|^2 )$, we have:
\begin{equation}
     \frac{1}{T-1} \sum_{i=1}^{T-1} \E[\| \nabla \Phi (\vx_i) \|^2] \le O (\frac{\kappa^2 \ell \Delta_{\Phi} + \kappa \ell^2 D_0 }{T-1}) + O (\epsilon^2)
\end{equation}

which completes the proof as stated.
\end{proof}

\subsection{Proof of Convergence of  EG}\label{app:NCSC_Upper_EG}
In this section, we present the convergence proof of the EG algorithm as  detailed in Algorithm~\ref{alg:eg}. We start by providing the proof sketch.

\begin{algorithm}[H] 
\caption{(Stochastic) EG}
\label{alg:eg}
\SetKwInOut{Input}{Input}
\SetKwInOut{Output}{Output}
\Input{Initialization $(\bm{x}_{-1}=\bm{x}_0,\bm{y}_{-1}=\bm{y}_0)$, learning rates $\eta_x, \eta_y$}
\For{$t=1,2,\dots,T$}{ 
 \algemph{antiquewhite}{0.95} { 
$\bm{x}_{t+1/2} = \bm{x}_{t} -  \eta_x \nabla_x f(\bm{x}_{t},\bm{y}_{t})$ ; \qquad \ \ \ $\bm{y}_{t+1/2} =  \bm{y}_{t} + \eta_y \nabla_y f(\bm{x}_{t},\bm{y}_{t})$ \;

$\bm{x}_{t+1} = \bm{x}_t - \eta_x \nabla_x f(\bm{x}_{t+1/2},\bm{y}_{t+1/2})  $ ; $\bm{y}_{t+1} = \bm{y}_t + \eta_y \nabla_y f(\bm{x}_{t+1/2},\bm{y}_{t+1/2}) $ ; \hfill  \#EG}  
 
 \algemph{blizzardblue}{0.95}{ 
 $\bm{x}_{t+1/2} = \bm{x}_{t} -  \eta_x \vg_{x,t}$ ; \qquad \ \ \ $\bm{y}_{t+1/2} =  \bm{y}_{t} + \eta_y \vg_{y,t}$  ;

$\bm{x}_{t+1} = \bm{x}_t - \eta_x \vg_{x,t+1/2}  $ ;\qquad   $\bm{y}_{t+1} = \bm{y}_t + \eta_y \vg_{y,t+1/2} $ ;  \hfill     \# Stochastic EG  }
 
} 
\end{algorithm}

\paragraph{Proof sketch.} We highlight the key ideas here. The first step is to derive to find an upper bound on $\Phi(\vx_{t+1}) - \Phi(\vx_t)$. Using $\kappa \ell$-smoothness property of $\Phi(\vx)$ at point $\vx_{t+1}$, and $\vx_{t}$ we bound the $\Phi(\vx_{t+1}) - \Phi(\vx_t)$ term, and then taking summation over all iterates, we derive the following primal descent lemma: 
\begin{equation}
\label{eqn:primalseg}
\begin{split}
 \E[\Phi(\vx_T)] - \Phi(\vx_0)  &\le -\frac{\eta_x}{2} \sum_{t=0}^{T-1} \E[ \|\nabla \Phi(\vx_t) \|^2] - \frac{\eta_x}{4} (1 - O(\eta_x)) \sum_{t=0}^{T-1} \mathbb{E}[\| \vg_{x,t} \|^2] \\
&\quad+ O(\eta_x \ell^2) \sum_{t=0}^{T-1} \E [\| \vy_t - \vy^*_t \|^2] O(\eta_x) \frac{\sigma^2 T}{M}.
\end{split}
\end{equation}
We also show the following dual descent lemma to directly bound $\sum_{t=0}^{T-1} \|\vy_t - \vy^*_t\|^2$ term in above inequality:
\begin{equation*}
\begin{split}
 \E[\|\vy_{t+1} - \vy^*_{t+1} \|^2] &\le (1 - \frac{1}{12 \kappa}) \E[\|\vy_t - \vy^*_t\|^2] + O(\kappa^3 \eta_x^2) \E[\|\vg_{x,t}\|^2] + \frac{2 \sigma^2}{M \ell^2}
\end{split}    
\end{equation*}
where we assumed $\eta_y = \frac{1}{4 \ell}$. Combining the primal and dual descent lemmas  yields the desired result on the convergence of EG to an $\epsilon$-stationary point.

In what follows, we provide the formal key lemmas, and the complete proof of Theorem~\ref{thm:ncsc_ogda}, and Theorem~\ref{thm:ncsc_sogda} for EG algorithm. Similar to OGDA, for the sake of brevity, we only present the convergence proof for stochastic version of EG (Theorem~\ref{thm:ncsc_sogda}), since by letting $\sigma = 0 $, we can recover the proof for deterministic algorithm (Theorem~\ref{thm:ncsc_ogda}).

\begin{lemma}
\label{lemma:seg1}
Let $\eta_y = \frac{1}{4 \ell}$, and $M = \max(M_x, M_y)$. Also assume $\eta_x \le \frac{1}{64 \kappa^2 \ell}$, then the iterates of Algorithm~\ref{alg:eg} satisfy the following inequalities:
\begin{equation}
\label{eqn:lmseg1.1}
\begin{split}
\E [\|\vy_{t+1} - \vy^*_{t+1}\|^2]  &\le (1-\frac{1}{12 \kappa}) \E[\|\vy_t - \vy^*_t\|^2 ]  +  18  \eta_x^2 \kappa^3 \E[\| \vg_{x,t}\|^2]  + 2 \frac{\sigma^2}{M \ell^2}
\end{split}
\end{equation}

\begin{equation}
\label{eqn:lmseg1.2}
\sum_{i=0}^{T-1} \E [\| \vy_i - \vy^*_i \|^2 ] \le 12 \kappa \|\vy_0 - \vy^*_0 \|^2 +   216 \eta_x^2 \kappa^4 \sum_{i=0}^{T-2} \E[\| \vg_{x,i}\|^2] + \frac{24 \kappa \sigma^2 (T-1) }{M\ell^2}
\end{equation}

\end{lemma}

\begin{proof}[Proof of Lemma \ref{lemma:seg1}]
Now we turn to convergence analysis for EG. The deterministic and stochastic variants of the EG algorithm are detailed  in Algorithm~\ref{alg:eg}.

To prove this lemma, we built on top of analysis in~\cite{mokhtari2020unified}.  We start by noting that:
\begin{equation}
\label{eqn:seg1}
\begin{split}
\|\vy_{t+1} - \vy^*_{t+\frac{1}{2}}\|^2 &= \|\vy_t - \vy^*_{t+\frac{1}{2}}\|^2 \\
&\quad - \|\vy_{t+1} - \vy_{t+\frac{1}{2}}\|^2 \\
&\quad - \|\vy_{t+\frac{1}{2}}-\vy_t\|^2 \\
& \quad + 2 \eta_y \ip{ \vg_{y,t} , \vy_{t+ \frac{1}{2}} - \vy_{t+1} }\\
& \quad + 2 \eta_y \ip{ \vg_{y,t+\frac{1}{2}},\vy_{t+1}-\vy^*_{t+\frac{1}{2}} }\\
\end{split}
\end{equation}
Let $\vdelta^y_i = \vg_{y,i} - \nabla_y f(\vx_i , \vy_i)$. We have: 
\begin{equation}
\label{eqn:seg2}
\begin{split}
 &2 \eta_y \langle \vg_{y,t} , \vy_{t+1/2}-\vy_{t+1} \rangle + 2 \eta_y \langle \vg_{y, t+ \frac{1}{2}},\vy_{t+1}-\vy^*_{t+\frac{1}{2}} \rangle\\
 &= 2\eta_y \langle \vg_{y,t} -\vg_{y,t+\frac{1}{2}}, \vy_{t+1/2}-\vy_{t+1}  \rangle + 2 \eta_y \langle \nabla_y f(\vx_{t+1/2},\vy_{t+1/2}),\vy_{t+1/2}-\vy^*_{t+1/2}\rangle \\
 &\quad+ \ip{\vdelta^y_{t+\frac{1}{2}} ,\vy_{t+1/2}-\vy^*_{t+1/2}}\\
 &\le  \|\vy_{t+1/2} - \vy_{t+1}\|^2 + \eta_y^2  \| \vg_{y,t} -\vg_{y,t+\frac{1}{2}} \|^2 - 2 \eta_y  \mu \|\vy_{t+1/2}-\vy^*_{t+1/2} \|^2 \\
 &\quad+ \ip{\vdelta^y_{t+\frac{1}{2}} ,\vy_{t+1/2}-\vy^*_{t+1/2}}\\
 &\le  \|\vy_{t+1/2} - \vy_{t+1}\|^2 + 2 \eta_y^2  \| \nabla_y f(\vx_t , \vy_t) - \nabla_y f(\vx_{t+\frac{1}{2}} , \vy_{t+\frac{1}{2}}) \|^2 - 2 \eta_y  \mu \|\vy_{t+1/2}-\vy^*_{t+1/2} \|^2  \\
 &\quad +\ip{\vdelta^y_{t+\frac{1}{2}} ,\vy_{t+1/2}-\vy^*_{t+1/2}} + 4 \eta_y^2   \| \vdelta^y_t \|^2 + 4 \eta_y^2  \| \vdelta^y_{t+\frac{1}{2}} \|^2 \\
 &\le  \|\vy_{t+1/2} - \vy_{t+1}\|^2 + 2 \eta_y^2 \ell^2 \|\vx_{t+\frac{1}{2}} - \vx_t\|^2 + 2\eta_y^2 \ell^2 \| \vy_{t+\frac{1}{2}} - \vy_t\|^2- 2 \eta_y  \mu \|\vy_{t+1/2}-\vy^*_{t+1/2} \|^2  \\
 &\quad +\ip{\vdelta^y_{t+\frac{1}{2}} ,\vy_{t+1/2}-\vy^*_{t+1/2}} + 4 \eta_y^2  \| \vdelta^y_t \|^2 + 4 \eta_y^2  \| \vdelta^y_{t+\frac{1}{2}} \|^2
\end{split}
\end{equation}
where in the first inequality, we used $\mu$-strong-concavity of $f(\vx,.)$, and in the second inequality, we used Young's inequality, and in the last one, we used the smoothness property. Now plugging Equation~\ref{eqn:seg2} back to Equation~\ref{eqn:seg1}, we have: 
\begin{equation}
\label{eqn:seg3}
\begin{split}
    \|\vy_{t+1} - \vy^*_{t+1/2}\|^2 &\le  \|\vy_t - \vy^*_{t+1/2}\|^2  - (1 -  2 \eta_y^2 \ell^2) \|\vy_{t+1/2} - \vy_t\|^2\\
    & \quad +2 \eta_y^2 \ell^2 \|\vx_{t+1/2} - \vx_t\|^2 - 2 \eta_y  \mu \|\vy_{t+1/2}-\vy^*_{t+1/2} \|^2 \\
    & \quad + \ip{\vdelta^y_{t+\frac{1}{2}} ,\vy_{t+1/2}-\vy^*_{t+1/2}} + 4 \eta_y^2   \| \vdelta^y_t \|^2 + 4 \eta_y^2 \| \vdelta^y_{t+\frac{1}{2}} \|^2
\end{split}
\end{equation}
Using Young's inequality, we can rewrite Equation~\ref{eqn:seg3} as follows:
\begin{equation}
\label{eqn:seg4}
\begin{split}
    \|\vy_{t+1} - \vy^*_{t+1/2}\|^2 &\le  (1-\eta_y \mu)\|\vy_t - \vy^*_{t+1/2}\|^2  - (1 - 2 \eta_y \mu - 2 \eta_y^2 \ell^2) \|\vy_{t+1/2} - \vy_t\|^2 \\
    &\quad +2 \eta_y^2 \ell^2 \|\vx_{t+1/2} - \vx_t\|^2 + \ip{\vdelta^y_{t+\frac{1}{2}} ,\vy_{t+1/2}-\vy^*_{t+1/2}} + 4 \eta_y^2  \| \vdelta^y_t \|^2 \\
    &\quad+ 4 \eta_y^2  \| \vdelta^y_{t+\frac{1}{2}} \|^2
\end{split}
\end{equation}
Assuming $\eta_y = \frac{1}{4 \ell}$, using Young's inequality, we have the following equation:
\begin{equation}
\label{eqn:seg5}
    \|\vy_t - \vy^*_{t+1/2}\|^2 \le (1+\frac{1}{16 \kappa}) \|\vy_t - \vy^*_t\|^2 + ( 1  + 16 \kappa) \|\vy^*_{t+1/2}-\vy^*_t\|^2
\end{equation}
\begin{equation}
\label{eqn:seg6}
    \|\vy_{t+1} - \vy^*_{t+1}\|^2 \le (1+\frac{1}{16 \kappa}) \|\vy_{t+1} - \vy^*_{t+1/2}\|^2 + (1+16 \kappa) \|\vy^*_{t+1}- \vy^*_{t+1/2}\|^2  
\end{equation}
Combining Equations~\ref{eqn:seg4},~\ref{eqn:seg5},~\ref{eqn:seg6} and using the $\kappa$ Lipschitzness of $\bm{y}^*(.)$, and noting that $1 - 2 \eta_y \mu -  2 \eta_y^2 \ell^2 >  0$, we get: 
\begin{equation}
\label{eqn:seg7}
\begin{split}
\|\vy_{t+1} - \vy^*_{t+1}\|^2 &\le (1-\frac{1}{8 \kappa}) \|\vy_t - \vy^*_t\|^2 + 17 \kappa^3 \|\vx_{t+1/2}-\vx_t\|^2 + 17 \kappa^3 \|\vx_{t+1}-\vx_{t+1/2}\|^2 \\
& \quad +   2 \ip{\vdelta^y_{t+\frac{1}{2}} ,\vy_{t+1/2}-\vy^*_{t+1/2}} + \frac{1}{2\ell^2}  \| \vdelta^y_t \|^2 + \frac{1}{2 \ell^2} \| \vdelta^y_{t+\frac{1}{2}} \|^2
\end{split}
\end{equation}
Using Young's inequality, we have: 
\begin{equation}
\label{eqn:seg8}
\begin{split}
\| \vx_{t+1} - \vx_{t+ \frac{1}{2}} \|^2 &= \eta_x^2 \| \vg_{x,t+\frac{1}{2}} - \vg_{x,t}\|^2 \\
&\le  2 \eta_x^2 \| \nabla_x f(\vx_{t+ \frac{1}{2}} , \vy_{t+ \frac{1}{2}})  - \nabla_x f(\vx_t , \vy_t) \|^2 + 4 \eta_x^2 \|\vdelta^x_{t+\frac{1}{2}} \|^2 +  4 \eta_x^2 \|\vdelta^x_{t} \|^2 \\
&\le 2 \eta_x^2 \ell^2 \| \vx_{t+ \frac{1}{2}} - \vx_t \|^2 + 2 \eta_x^2 \ell^2 \|\vy_{t+\frac{1}{2}} - \vy_t \|^2 +  4 \eta_x^2 \|\vdelta^x_{t+\frac{1}{2}} \|^2 +  4 \eta_x^2 \|\vdelta^x_{t} \|^2 \\
&\le 2 \eta_x^2 \ell^2 \| \vx_{t+ \frac{1}{2}} - \vx_t \|^2 + 4 \eta_x^2 \ell^2 \|\vy_{t+\frac{1}{2}} - \vy^*_t \|^2 +  4 \eta_x^2 \ell^2 \| \vy_t - \vy^*_t\|^2 \\&\quad  + 4 \eta_x^2 \|\vdelta^x_{t+\frac{1}{2}} \|^2 
+  4 \eta_x^2 \|\vdelta^x_{t} \|^2 \\
&\le 2 \eta_x^2 \ell^2 \| \vx_{t+ \frac{1}{2}} - \vx_t \|^2 +  8 \eta_x^2 \ell^2 \| \vy_t - \vy^*_t\|^2  + \frac{\eta_x^2}{2} \|\vdelta^y_t\|^2 +   4 \eta_x^2 \|\vdelta^x_{t+\frac{1}{2}} \|^2 \\&\quad +  4 \eta_x^2 \|\vdelta^x_{t} \|^2  + 2 \eta_x^2 \ell \ip{\vdelta^y_t , \vy_t - \vy^*_t} 
\end{split}
\end{equation}
where in the last inequality, we used Lemma~\ref{lemma:a3}. Plugging Equation~\ref{eqn:seg8}, in Equation~\ref{eqn:seg7}, and assuming $\eta_x \le \frac{1}{64 \kappa^2 \ell}$ gives: 
\begin{equation}
\label{eqn:seg9}
\begin{split}
\|\vy_{t+1} - \vy^*_{t+1}\|^2 &\le (1-\frac{1}{12 \kappa}) \|\vy_t - \vy^*_t\|^2 +  18 \kappa^3 \|\vx_{t+1/2}-\vx_t\|^2  \\
&\quad+   2 \ip{\vdelta^y_{t+\frac{1}{2}} ,\vy_{t+1/2}-\vy^*_{t+\frac{1}{2}}} + \frac{1}{64 \kappa \ell} \ip{\vdelta^y_t , \vy_t - \vy^*_t}  \\
&\quad+ \frac{1}{\ell^2} \| \vdelta^y_{t} \|^2 + \frac{1}{2 \ell^2}\| \vdelta^y_{t+\frac{1}{2}} \|^2 + \frac{1}{4 \ell^2} \|\vdelta^x_{t+\frac{1}{2}} \|^2 +  \frac{1}{4 \ell^2} \|\vdelta^x_{t} \|^2
\end{split}
\end{equation}
or equivalently: 
\begin{equation}
\label{eqn:seg10}
\begin{split}
\|\vy_{t+1} - \vy^*_{t+1}\|^2 &\le (1-\frac{1}{12 \kappa}) \|\vy_t - \vy^*_t\|^2 +  18 \eta_x^2 \kappa^3 \|\vg_{x,t}\|^2  \\
&\quad +   2 \ip{\vdelta^y_{t+\frac{1}{2}} ,\vy_{t+\frac{1}{2}}-\vy^*_{t+\frac{1}{2}}} + \frac{1}{64 \kappa \ell} \ip{\vdelta^y_t , \vy_t - \vy^*_t}  \\
&\quad+ \frac{1}{\ell^2} \| \vdelta^y_{t} \|^2 + \frac{1}{2 \ell^2}\| \vdelta^y_{t+\frac{1}{2}} \|^2 + \frac{1}{4 \ell^2} \|\vdelta^x_{t+\frac{1}{2}} \|^2 +  \frac{1}{4 \ell^2} \|\vdelta^x_{t} \|^2
\end{split}
\end{equation}
Taking expectation from both sides of Equation~\ref{eqn:seg10} yields: 
\begin{equation}
\label{eqn:seg11}
\begin{split}
\E [\|\vy_{t+1} - \vy^*_{t+1}\|^2]  &\le \left(1-\frac{1}{12 \kappa}\right) \E[\|\vy_t - \vy^*_t\|^2 ]  +  18  \eta_x^2 \kappa^3 \E[\| \vg_{x,t}\|^2]  + 2 \frac{\sigma^2}{M \ell^2}
\end{split}
\end{equation}
Now using Lemma~\ref{lemma:a2} we get \begin{equation}
\label{eqn:seg12}
\sum_{i=0}^{T-1} \E [\| \vy_i - \vy^*_i \|^2 ] \le 12 \kappa \|\vy_0 - \vy^*_0 \|^2 +   216 \eta_x^2 \kappa^4 \sum_{i=0}^{T-2} \E[\| \vg_{x,i}\|^2] + \frac{24 \kappa \sigma^2 (T-1) }{M \ell^2}
\end{equation}
as stated in the lemma.
\end{proof}

\begin{lemma}
\label{lemma:seg2}
Let $\Phi (\vx) = \max_y f(\vx,\vy)$, and $\eta_y = \frac{1}{4 \ell}$. Then the iterates of Algorithm~\ref{alg:eg} satisfy the following inequality:
\begin{equation}
\label{eqn:lmseg2}
\begin{split}
\E[\Phi (\vx_{t+1})] &\le \E[\Phi(\vx_t)] - \frac{\eta_x}{2}\E[\|\nabla \Phi(\vx_t)\|^2] - \frac{\eta_x}{4}(1- 2 \eta_x \kappa \ell - 8 \eta_x^2 \ell^2 ) \E[\| \vg_{x,t}\|^2] \\
&\quad+ 5 \eta_x \ell^2 \E[\|\vy_t - \vy^*_t\|^2] + 7 \eta_x  \frac{\sigma^2}{M}
\end{split}
\end{equation}
\end{lemma}

\begin{proof}[Proof of Lemma~\ref{lemma:seg2}]
Let $\vdelta^x_i = \vg_{x,i}  - \nabla_x f(\vx_i , \vy_i)$.Using smoothness property at $\vx_{t+1}$ and $\vx_t$, we have: 
\begin{equation}
\label{eqn:seg13}
\begin{split}
\Phi (\vx_{t+1}) &\le \Phi(\vx_t) - \eta_x \langle \nabla \Phi (\vx_t) , \vg_{x,t+\frac{1}{2}} \rangle  + \eta_x^2 \kappa \ell  \| \vg_{x,t+\frac{1}{2}}\|^2 \\ 
&= \Phi(\vx_t) - \frac{\eta_x}{2}\|\nabla \Phi(\vx_t)\|^2 - \frac{\eta_x}{2}(1- 2 \eta_x \kappa \ell ) \| \vg_{x,t+\frac{1}{2}}\|^2 + \frac{ \eta_x}{2} \|\nabla \Phi (\vx_t) - \vg_{x,t+\frac{1}{2}}\|^2\\
&\le \Phi(\vx_t) - \frac{\eta_x}{2}\|\nabla \Phi(\vx_t)\|^2 - \frac{\eta_x}{2}(1- 2 \eta_x \kappa \ell ) \| \vg_{x,t+\frac{1}{2}}\|^2 \\&\quad + \eta_x \|\nabla \Phi (\vx_t) - \nabla_x f(\vx_{t +\frac{1}{2}} , \vy_{t+ \frac{1}{2}})\|^2 + \eta_x \|\vdelta^x_{t+\frac{1}{2}}\|^2 \\
&\le \Phi(\vx_t) - \frac{\eta_x}{2}\|\nabla \Phi(\vx_t)\|^2 - \frac{\eta_x}{2}(1- 2 \eta_x \kappa \ell ) \| \vg_{x,t+\frac{1}{2}}\|^2 + \eta_x \ell^2 \|\vx_{t + \frac{1}{2}} - \vx_t \|^2 \\
&\quad + \eta_x \ell^2 \|\vy_{t+\frac{1}{2}} - \vy^*_{t} \|^2 + \eta_x \|\vdelta^x_{t+\frac{1}{2}}\|^2 \\
\end{split}
\end{equation}
Using Young's inequality, we have: 
\begin{equation}
\label{eqn:seg14}
\begin{split}
\| \vg_{x,t + \frac{1}{2}} \|^2 \ge \frac{1}{2} \|\vg_{x,t} \|^2 - \|\vg_{x,t+\frac{1}{2}} - \vg_{x,t} \|^2    
\end{split}
\end{equation}
Plugging Equation~\ref{eqn:seg14} back to Equation~\ref{eqn:seg13}, and assuming $\eta_x \le \frac{1}{2 \kappa \ell}$ results in:
\begin{equation}
\label{eqn:seg15}
\begin{split}
\Phi (\vx_{t+1}) &\le \Phi(\vx_t) - \frac{\eta_x}{2}\|\nabla \Phi(\vx_t)\|^2 - \frac{\eta_x}{4}(1- 2 \eta_x \kappa \ell ) \| \vg_{x,t}\|^2 + \eta_x \ell^2 \|\vx_{t + \frac{1}{2}} - \vx_t \|^2 \\ &\quad + \eta_x \ell^2 \|\vy_{t+\frac{1}{2}} - \vy^*_{t} \|^2 +\frac{\eta_x}{2} \|\vg_{x,t+\frac{1}{2}} - \vg_{x,t} \|^2  +\eta_x \|\vdelta^x_{t+\frac{1}{2}}\|^2 \\ 
&\le \Phi(\vx_t) - \frac{\eta_x}{2}\|\nabla \Phi(\vx_t)\|^2 - \frac{\eta_x}{4}(1- 2 \eta_x \kappa \ell ) \| \vg_{x,t}\|^2 + \eta_x \ell^2 \|\vx_{t + \frac{1}{2}} - \vx_t \|^2 \\ &\quad+ \eta_x \ell^2 \|\vy_{t+\frac{1}{2}} - \vy^*_{t} \|^2 + \eta_x \|\nabla_x f(\vx_{t+\frac{1}{2}} , \vy_{t+\frac{1}{2}}) - \nabla_x f(\vx_t , \vy_t)\|^2  \\
&\quad + 2 \eta_x \| \vdelta^x_{t+\frac{1}{2}}\|^2 + 2 \eta_x \| \vdelta^x_{t}\|^2+\eta_x \|\vdelta^x_{t+\frac{1}{2}}\|^2 \\
&\le \Phi(\vx_t) - \frac{\eta_x}{2}\|\nabla \Phi(\vx_t)\|^2 - \frac{\eta_x}{4}(1- 2 \eta_x \kappa \ell ) \| \vg_{x,t}\|^2 + \eta_x \ell^2 \|\vx_{t + \frac{1}{2}} - \vx_t \|^2  \\ &\quad  + \eta_x \ell^2 \|\vy_{t+\frac{1}{2}} - \vy^*_{t} \|^2 + \eta_x \ell^2 \| \vx_{t+\frac{1}{2}} - \vx_t \|^2 + \eta_x \ell^2 \|\vy_{t+\frac{1}{2}}- \vy_t \|^2 \\
&\quad + 2 \eta_x \| \vdelta^x_{t+\frac{1}{2}}\|^2 + 2 \eta_x \| \vdelta^x_{t}\|^2+\eta_x \|\vdelta^x_{t+\frac{1}{2}}\|^2 \\
\end{split}
\end{equation}
Using Lemma~\ref{lemma:a3} and Young's inequality, we have:
\begin{equation}
\label{eqn:seg16}
\begin{split}
\| \vy_{t+\frac{1}{2}} - \vy^*_t \|^2  + \|\vy_{t+\frac{1}{2}} - \vy_t \|^2 &\le  3\| \vy_{t+\frac{1}{2}} - \vy^*_t \|^2 + 2 \|\vy_t - \vy^*_t\|^2 \\
&\le 5 \| \vy_t - \vy^*_t\|^2 + \frac{3}{8 \ell^2} \| \vdelta^y_t \|^2 + \frac{3}{2\ell} \ip{\vdelta^y_t , \vy_t - \vy^*_t}
\end{split}
\end{equation}
Plugging Equation~\ref{eqn:seg16} in Equation~\ref{eqn:seg15}, we get: 
\begin{equation}
\label{eqn:seg17}
\begin{split}
\Phi (\vx_{t+1}) &\le \Phi(\vx_t) - \frac{\eta_x}{2}\|\nabla \Phi(\vx_t)\|^2 - \frac{\eta_x}{4}(1- 2 \eta_x \kappa \ell - 8 \eta_x^2 \ell^2 ) \| \vg_{x,t}\|^2  + 5 \eta_x \ell^2 \|\vy_t - \vy^*_t\|^2\\
&\quad + \frac{3}{8}\eta_x  \|\vdelta^y_t\|^2+ \frac{3}{2} \eta_x \ell \ip{\vdelta^y_t , \vy_t - \vy^*_t} + 2 \eta_x \| \vdelta^x_{t+\frac{1}{2}}\|^2 + 2 \eta_x \| \vdelta^x_{t}\|^2+\eta_x \|\vdelta^x_{t+\frac{1}{2}}\|^2 \\
\end{split}
\end{equation}
Taking expectations from both sides of Equation~\ref{eqn:seg17}, we have: 
\begin{equation}
\label{eqn:seg18}
\begin{split}
\E[\Phi (\vx_{t+1})] &\le \E[\Phi(\vx_t)] - \frac{\eta_x}{2}\E[\|\nabla \Phi(\vx_t)\|^2] - \frac{\eta_x}{4}(1- 2 \eta_x \kappa \ell - 8 \eta_x^2 \ell^2 ) \E[\| \vg_{x,t}\|^2] \\
&\quad + 5 \eta_x \ell^2 \E[\|\vy_t - \vy^*_t\|^2]  + 7 \eta_x \frac{\sigma^2}{M}
\end{split}
\end{equation}

\end{proof}

\begin{proof}[Proof of Theorem~\ref{thm:ncsc_ogda}, and Theorem~\ref{thm:ncsc_sogda} for EG]
Equipped with the above lemmas, we can prove the theorem as follows. We start by taking summation from $t=0$ to $t = T-1$ of Equation~\ref{eqn:lmseg2} in  Lemma~\ref{lemma:seg2}, to get:
\begin{equation}
\label{eqn:seg19}
\begin{split}
    \E[\Phi(\vx_T)] &\le \Phi(\vx_0) - \frac{\eta_x}{2} \sum_{t=0}^{T-1} \E[\|\nabla \Phi (\vx_t) \|^2] - \frac{\eta_x}{4} (1  - 2 \eta_x \kappa \ell - 8 \eta_x^2 \ell^2) \sum_{t=0}^{T-1} \E[\|\vg_{x,t}\|^2] \\
    &\quad + 5\eta_x \ell^2 \sum_{t=0}^{T-1} \E[\|\vy_t - \vy^*_t\|^2] + 7 \eta_x  \frac{\sigma^2 T}{M}
\end{split}
\end{equation}
Replacing $\sum_{t=0}^{T-1} \E[\|\vy_t - \vy^*_t\|^2]$ with the upper bound in Lemma~\ref{lemma:seg1}, we have: 
\begin{equation}
\label{eqn:seg20}
\begin{split}
    \E[\Phi(\vx_T)] &\le 60 \eta_x \kappa \ell^2 \|\vy_0 - \vy^*_0\|^2 + \Phi(\vx_0) - \frac{\eta_x}{2} \sum_{t=0}^{T-1} \E[\|\nabla \Phi (\vx_t) \|^2]\\
    &- \frac{\eta_x}{4} (1  - 2 \eta_x \kappa \ell - 8 \eta_x^2 \ell^2- 4320\eta_x^2 \kappa^4 \ell^2) \sum_{t=0}^{T-1} \E[\|\vg_{x,t}\|^2] \\
    &+  \frac{120 \eta_x \kappa  \sigma^2 (T-1)}{M}+ 7 \eta_x \frac{\sigma^2 T}{M} 
\end{split}
\end{equation}
Let $\eta_x = \frac{1}{75 \kappa^2 \ell}$. Then $1  - 2 \eta_x \kappa \ell - 8 \eta_x^2 \ell^2- 4320\eta_x^2 \kappa^4 \ell^2 > 0 $. After rearranging and simplifying the terms of Equation~\ref{eqn:seg20}, we have: 
\begin{equation}
\label{eqn:seg21}
\begin{split}
   \sum_{t=0}^{T-1} \E[\|\nabla \Phi (\vx_t)\|^2]  &\le \frac{2 \Delta_{\Phi}}{\eta_x}+ 120 \kappa \ell^2 \|\vy_0 - \vy_0^*\|^2 + \frac{240 \kappa \sigma^2 T}{M} + \frac{14 \sigma^2 T}{M}
\end{split}
\end{equation}
Replacing $\eta_x = \frac{1}{75 \kappa^2 \ell}$ in Equation~\ref{eqn:seg21}, we have:
\begin{equation}
\label{eqn:seg22}
\begin{split}
   \frac{1}{T}\sum_{t=0}^{T-1} \E[\|\nabla \Phi (\vx_t)\|^2]  &\le \frac{150 \kappa^2 \ell \Delta_{\Phi}+ 120 \kappa \ell^2 \|\vy_0 - \vy_0^*\|^2}{T} + \frac{240 \kappa \sigma^2 }{M} + \frac{14 \sigma^2 }{M}.
\end{split}
\end{equation}
Now by letting, $M = \frac{\kappa \sigma^2}{\epsilon^2}$, and $D_0 = \|\vy_0 - \vy^*_0\|^2$, we have: 
\begin{equation}
   \frac{1}{T}\sum_{t=0}^{T-1} \E[\|\nabla \Phi (\vx_t)\|^2]  \le O(\frac{ \kappa^2 \ell \Delta_{\Phi}+ \kappa \ell^2 D_0}{T}) + O(\epsilon^2)
\end{equation}

\end{proof}

\subsection{Tightness Analysis} \label{app:NCSC_Tightness}

In this section we provide the complete proofs for Theorem~\ref{thm:NCSC_Tightness_GDA} (Subsection~\ref{app:NCSC_Tightness_GDA}), and Theorem~\ref{thm:NCSC_Tightness_EG/OGDA} (Subsection~\ref{app:NCSC_Tightness_EG/OGDA}), showing the tightness of the obtained upper bounds given our choice of learning rates.

\subsubsection{GDA}\label{app:NCSC_Tightness_GDA}
\begin{proof}[Proof of Theorem~\ref{thm:NCSC_Tightness_GDA}]

Recall that we consider the following quadratic NC-SC function $f:\mathbb{R}\times\mathbb{R}\to \mathbb{R}$
\begin{align*}
	f(x,y):=-\tfrac{1}{4}\ell x^2+bxy-\tfrac{1}{2}\mu y^2.
	\end{align*}
We know $f$ is nonconvex in $x$ (it is actually concave in $x$) and $\mu$ strongly concave in $y$. Assume $\kappa:=\ell/\mu\ge 4$ and choose 	$b = \sqrt{\mu (\ell+2\mu_x)/2}
$ for some $0<\mu_x\le \ell/2$ to be chosen later.
Then we know $b\le \ell/2$ and it is easy to verify $f$ is $\ell$ smooth. Note that the primal function
\begin{align*}
	\Phi(x)=\max_y f(x,y)=\tfrac{1}{2}\mu_x x^2
\end{align*}
is actually strongly convex. This also justifies the symbol for $\mu_x$. We use GDA to find the solution for $\min_x\max_y f(x,y)$. Actually for this problem the optimal solution is achieved at the origin. The stepsizes are chosen as
$	\eta_x = \frac{c_1}{\kappa^2\ell}$ and $ \eta_y = \frac{c_2}{\ell}$ for some small enough numerical constants $c_1$ and $c_2$ such that $c=c_2/c_1\ge 1$. Also denote $r=\eta_y/\eta_x=c\kappa^2$ as the stepsize ratio. Then the GDA update rule can be written as
\begin{align}
\begin{pmatrix}
x_{k+1}\\ y_{k+1} 
\end{pmatrix}= (I+\eta_x \mathbf{M})\cdot \begin{pmatrix}
x_{k}\\ y_{k} 
\end{pmatrix},
\label{eq:gda_update}
\end{align}
where 
\begin{align*}
\mathbf{M}:=\begin{pmatrix}
\ell/2 & -b\\ r b & -\mu r 
\end{pmatrix}.
\end{align*}
We note that the above update is a linear time invariant system. We need to analyze its eigenvalues. Let $\lambda_1$ and $\lambda_2$ be the two eigenvalues of $\mathbf{M}$, we have
\begin{align*}
	\lambda_{1,2} = -\frac{1}{2}\left(\mu r-\frac{1}{2}\ell\right) \pm \frac{1}{2}\sqrt{\left(\mu r-\frac{1}{2}\ell\right)^2-4r\mu\mu_x}.
\end{align*}
Note that if we choose $\mu_x< \ell/8$, plugging into $r=c\kappa^2$, we can bound
\begin{align*}
0\ge	\lambda_1 &= -\frac{(2c\kappa-1)\ell}{4}\left(1-\sqrt{1-\frac{4c\kappa\mu_x}{(c\kappa-1/2)^2\ell}}\right)\\
&\ge-\frac{2c\kappa \mu_x}{c\kappa-1/2} \ge -4\mu_x.
\end{align*}

Let $s_1$ be the corresponding eigenvalue of $I+\eta_x M$, for small enough $c_1\le 1$, it satisfies
\begin{align*}
	 0\le 1-\frac{4c_1\mu_x}{\kappa^2}\le s_1=1+\eta_x \lambda_1\le 1.
\end{align*}
We adversarially choose the initial point $(x_0,y_0)$ such that it is parallel to the eigenvector of $I+\eta_x M$ corresponding to $s_1$. We can always choose $x_0\ge 0$ for simplicity. Then we have
\begin{align*}
 \begin{pmatrix}
    x_{k+1}\\ y_{k+1} 
\end{pmatrix}&= (I+\eta_x \mathbf{M})^T \begin{pmatrix}
x_{0}\\ y_{0} 
\end{pmatrix} = s_1^T \begin{pmatrix}
x_{0}\\ y_{0} 
\end{pmatrix},
\end{align*}
so we can compute the magnitude of $x_T$ as	$x_T = s_1^T x_0$.  Also note that $\Delta_\Phi = \Phi(x_0)=\frac{1}{2}\mu_x x_0^2$. Note that if $\Delta_{\Phi}=0$, this lemma is trivially true. Therefore we can assume $\Delta_{\Phi}>0$. Choosing $\mu_x=\epsilon^2/\Delta_\Phi$, we have
\begin{align*}
	\abs{\nabla \Phi(\Bar{x})}=\mu_x \Bar{x}\ge& \mu_x x_T\ge \mu_x x_0\left(1-\frac{4c_1\mu_x}{\kappa^2}\right)^{T}\\
	=&\sqrt{2}\epsilon \left(1-\frac{4c_1 \epsilon^2}{\kappa^2\Delta_\Phi}\right)^{T},
\end{align*}
where $\Bar{x}\ge x_T$ because $x_0\ge x_1\ge\cdots\ge x_T$ and $\Bar{x}$ is sampled from this sequence.
Then we know that to achieve $\abs{\nabla \Phi(\Bar{x})}\le \epsilon$, we must have
$T= \Omega\left(\frac{\kappa^2\Delta_\Phi}{\epsilon^2}\right)$ as stated. 

\end{proof}

\subsubsection{EG/OGDA}\label{app:NCSC_Tightness_EG/OGDA}

\begin{proof}[Proof of Theorem~\ref{thm:NCSC_Tightness_EG/OGDA} for EG]
We consider the same quadratic hard example $f$ and notation used in the proof of Theorem~\ref{thm:NCSC_Tightness_GDA}.
For simplicity, denote $\vw=(x,y)$. Then EG satisfies
\begin{align*}
	\vw_{k+1/2}=&(I+\eta_x \mathbf{M})\vw_k,\\
	\vw_{k+1}=& \vw_k+\eta_x \mathbf{M} \vw_{k+1/2}\\
	=& (I+\eta_x \mathbf{M}+\eta_x^2\mathbf{M}^2) \vw_k.
\end{align*}
Therefore, similar to GDA, EG is also a linear time invariant system. The transition matrix for EG is $(I+\eta_x \mathbf{M}+\eta_x^2\mathbf{M}^2)$.
Its eigenvalues are
\begin{align*}
	s_i = 1+\eta_x\lambda_i+\eta_x^2\lambda_i^2\ge 1+\eta_x\lambda_i,\quad i=1,2.
\end{align*}
The rest of analysis is the same as that of GDA.

\end{proof}

\begin{proof}[Proof of Theorem~\ref{thm:NCSC_Tightness_EG/OGDA} for OGDA]

We consider the same quadratic hard example $f$ and the notation used in the proofs of Theorems~\ref{thm:lower_bound_gda}~and~\ref{thm:lower_bound_eg}.
The dynamics of OGDA is
\begin{align*}
	\vw_{k+1} = \vw_k +2\eta_x\mathbf{M} \vw_k -\eta_x \mathbf{M} \vw_{k-1}.
\end{align*}
If we initialize $\vw_0$ parallel to the eigenvector of $\mathbf{M}$ corresponding to $\lambda_1$ and let $\vw_1=\vw_0$, we know every $\vw_k$ is parallel to it, i.e., $\vw_k = z_k \vw_0$ for some scalar $z_k$ which satisfies
\begin{align*}
		z_{k+1} = z_k +2\eta_x \lambda_1 z_k -\eta_x \lambda_1 z_{k-1}.
\end{align*}
The general solution of the above recurrence relation is
\begin{align*}
	z_k = a \alpha^k + b\beta^k
\end{align*}
for some constant $a,b$ and
\begin{align*}
	\alpha =& \frac{1}{2}\left(1+2\eta_x\lambda_1+\sqrt{1+4\eta_x^2\lambda_1^2}\right),\\
	\beta =& \frac{1}{2}\left(1+2\eta_x\lambda_1-\sqrt{1+4\eta_x^2\lambda_1^2}\right).
\end{align*}
We have
\begin{align*}
	1+\eta_x\lambda_1 \le \alpha\le 1, \quad \eta_x\lambda_1\le  \beta\le 0.
\end{align*}
Using the initial condition $z_{-1}=z_0=1$, we can get the constants
\begin{align*}
	a=&\frac{\alpha(1-\beta)}{\alpha-\beta}=\frac{1}{2}+\frac{1}{2\sqrt{1+4\eta_x^2\lambda_1^2}}\ge 1/2,\\
	b=&-\frac{\beta(1-\alpha)}{\alpha-\beta}=\frac{\sqrt{1+4\eta_x^2\lambda_1^2}-1}{2\sqrt{1+4\eta_x^2\lambda_1^2}}\le \eta_x^2\lambda_1^2.
\end{align*}
We can bound
\begin{align*}
	|z_T|&\ge \frac{1}{2}\left(1+\eta_x\lambda_2 \right)^T - |\eta_x\lambda_1|^{k+2}\\
	&\ge \frac{1}{2}\left(1-\frac{4c_1\mu_x}{\kappa^2}\right)^{T}-\frac{1}{4},
\end{align*}
where we use the fact $|\eta_x\lambda_1|\le 1/2$.
Similar to the analysis for GDA, choosing $\mu_x=50\epsilon^2/\Delta_{\Phi}$, we have
\begin{align*}
	\abs{\nabla \Phi(\Bar{x})}=\mu_x \Bar{x}\ge& \mu_x x_T\ge \mu_x x_0\left[\frac{1}{2}\left(1-\frac{4c_1\mu_x}{\kappa^2}\right)^{T}-\frac{1}{4}\right]\\
	=&10\epsilon \left[\frac{1}{2}\left(1-\frac{4c_1\mu_x}{\kappa^2}\right)^{T}-\frac{1}{4}\right].
\end{align*}
Therefore, if $\abs{\nabla \Phi(\Bar{x})}\le\epsilon$, we must have
\begin{align*}
	T= \Omega\left(\frac{\kappa^2}{\mu_x}\right)=\Omega\left(\frac{\kappa^2\Delta_\Phi}{\epsilon^2}\right).
\end{align*}

\end{proof}


\section{Proof of Convergence in Nonconvex-Concave Setting}\label{app:NCC_Upper}

\subsection{Proof of convergence of OGDA}\label{app:NCC_Upper_OGDA}
In this section, the convergence of OGDA in NC-C setting has been established. Before presenting the complete proofs, here we briefly discuss the proof sketch. 

\paragraph{Proof sketch} We start from the standard descent analysis on Moreau envelope function~\cite{davis2019stochastic}. Let $\delta_t = \Phi(\bm{x}_{t}) - f(\bm{x}_{t},\bm{y}_{t})$, then we can show:
\begin{align*} 
  \frac{1}{T+1}\sum_{t=0}^T \|\nabla \Phi_{1/2\ell}(\bm{x}_{t}) \|^2  &\leq \frac{\Phi_{1/2\ell}({\bm{x}}_{0}) - \Phi_{1/2\ell}(\bm{x}_{T+1})}{T+1}  + O\left(\frac{1}{T+1}\sum_{t=0}^T\ell\delta_t\right)  +O(\ell \eta_x^2 G^2) \\
    &\quad +\frac{1}{T+1}\sum_{t=0}^T O( \| \nabla_x f(\bm{x}_{t},\bm{y}_{t}) - \nabla_x f(\bm{x}_{t-1},\bm{y}_{t-1})\|^2 ) .
\end{align*}
It turns out that the gradient norm depends on two terms,  difference between gradient at time $t$ and $t-1$ and $\delta_t$: primal function gap at iteration $t$. To bound the first term, we can utilize smoothness of $\nabla f$ and reduce the problem to bounding $\|\bm{y}_t - \bm{y}_{t-1}\|^2$:
\begin{align*}
      \sum_{t=0}^{T}\|\bm{y}_t - \bm{y}_{t-1}\|^2   
      & \leq \sum_{t=0}^{T} O \left(\eta_y^2 \ell \sum_{j=0}^{T}(2\eta_y^2 \ell^2)^{j}\right)  \delta_t +\sum_{t=0}^{T}O\left(\eta_x^2 \eta_y^2 \ell^2 G^2\sum_{j=0}^{T}(2\eta_y^2 \ell^2)^{j} \right) .
 \end{align*}
Here we reduce difference between dual iterates to primal function gap $\delta_t$. Now, it remains to bound $\delta_t$. We have the following recursion relation holding for any $t$ and any $s \leq t$:
     \begin{align}
     &\Phi(\bm{x}_t) - f(\bm{x}_t,\bm{y}_t) \leq O(\eta_x (t-s)G^2) +  \frac{1}{2\eta_y} ( \|\bm{y}_{t-1}-y^*(\bm{x}_s) \|^2-\|\bm{y}_t - \bm{y}^*(\bm{x}_s)\|^2 +  \eta_x^2 \eta_y \ell G^2\nonumber\\ 
     & \quad  + \frac{1}{2} \|\bm{y}_{t-1} - \bm{y}_{t-2}\|^2 -   \frac{1}{2} \|\bm{y}_t - \bm{y}_{t-1}\|^2  )  +   \langle \nabla_y f(\bm{x}_{t-1},\bm{y}_{t-1}) - \nabla_y f(\bm{x}_{t-2},\bm{y}_{t-2}), \bm{y}_{t-1} - \bm{y}^*(\bm{x}_s) \rangle\nonumber\\
    &\quad -   \langle \nabla_y f(\bm{x}_t,\bm{y}_t) - \nabla_y f(\bm{x}_{t-1},\bm{y}_{t-1}), \bm{y}_t - \bm{y}^*(\bm{x}_s) \rangle. 
    \label{eq: ogda recursion}
 \end{align}
If we let $s$ stay the same for some iterations, $({1}/{T+1})\sum_{t=0}^T \delta_t$ vanishes in a telescoping fashion.

In the following, we present the key lemmas, and complete convergence proof of OGDA. First let us introduce some useful lemmas for deterministic setting.
\subsubsection{Useful Lemmas}
\begin{lemma}\label{lemma: ogda dual}
For OGDA (Algorithm~\ref{alg:ogda}), under Theorem~\ref{thm:ncc_sogda}'s assumptions, the following statement holds for the generated sequence $\{\bm{y}_t\}$ during algorithm proceeding and any $\bm{y} \in \mathcal{Y}$:
\begin{equation}
\begin{split}
    \|\bm{y}_t - \bm{y}\|^2 &\leq \|\bm{y}_{t-1}-\bm{y} \|^2 - \frac{1}{2} \|\bm{y}_t - \bm{y}_{t-1}\|^2 + \frac{1}{2} \|\bm{y}_{t-1} - \bm{y}_{t-2}\|^2+ 2\eta_y \langle \bm{y}_t - \bm{y}, \nabla_y f(\bm{x}_{t},\bm{y}_{t})  \rangle \\
    &\quad +\eta_y \eta_x^2 \ell G^2  -2\eta_y \langle \nabla_y f(\bm{x}_t,\bm{y}_t) - \nabla_y f(\bm{x}_{t-1},\bm{y}_{t-1}), \bm{y}_t - \bm{y} \rangle \\
    &\quad +   2\eta_y \langle \nabla_y f(\bm{x}_{t-1},\bm{y}_{t-1}) - \nabla_y f(\bm{x}_{t-2},\bm{y}_{t-2}), \bm{y}_{t-1} - \bm{y} \rangle.
\end{split}
\end{equation}

\begin{proof}
According to updating rule of $\bm{y}$:
\begin{align*}
    \bm{y}_t = \mathcal{P}_{\mathcal{Y}}\left(\bm{y}_{t-1}+2\eta_y \nabla_y f(\bm{x}_{t-1},\bm{y}_{t-1})-\eta_y \nabla_y f(\bm{x}_{t-2},\bm{y}_{t-2})\right).
\end{align*}
Following the analysis in \cite{mokhtari2020convergence}, we let $\varepsilon_{t-1} =\eta_y (\nabla_y f(\bm{x}_{t},\bm{y}_{t})-\nabla_y f(\bm{x}_{t-1},\bm{y}_{t-1})) -\eta_y(\nabla_y f(\bm{x}_{t-1},\bm{y}_{t-1})-\nabla_y f(\bm{x}_{t-2},\bm{y}_{t-2}))$ and re-write the updating rule as:
\begin{align*}
    \bm{y}_t = \mathcal{P}_{\mathcal{Y}}\left(\bm{y}_{t-1}+ \eta_y \nabla_y f(\bm{x}_{t},\bm{y}_{t})-\varepsilon_{t-1}\right)
\end{align*}
Then, due to the property of projection onto convex set we have the following inequality that holds  for any $\bm{y}\in\mathcal{Y}$:
\begin{align*}
    (\bm{y}-\bm{y}_t)^\top ( \bm{y}_t - \bm{y}_{t-1} - \eta_y \nabla_y f(\bm{x}_t,\bm{y}_t)+\varepsilon_{t-1}) \geq 0.
\end{align*}

Using the identity that $\langle \bm{a}, \bm{b} \rangle = \frac{1}{2}(\|\bm{a}+\bm{b}\|^2 - \|\bm{a}\|^2 - \|\bm{b}\|^2)$ we have:
\begin{align*}
  0 &\leq  \|\bm{y} - \bm{y}_{t-1} - \eta_y\nabla_y f(\bm{x}_{t},\bm{y}_{t})+\varepsilon_{t-1} \|^2 - \|\bm{y} - \bm{y}_{t}\|^2 - \|\bm{y}_t - \bm{y}_{t-1} - \eta_y\nabla_y f(\bm{x}_{t},\bm{y}_{t})+\varepsilon_{t-1}\|^2 \\
  & \leq  \|\bm{y} - \bm{y}_{t-1}  \|^2 - \|\bm{y} - \bm{y}_{t}\|^2 - \|\bm{y}_t - \bm{y}_{t-1}\|^2 + 2\langle \bm{y}_t - \bm{y},\eta_y \nabla_y f(\bm{x}_t, \bm{y}_t) \rangle - 2 \langle \bm{y}_t - \bm{y},\varepsilon_{t-1} \rangle  
\end{align*}
Now we plug the definition of $\varepsilon_{t-1}$ into above inequality to get:

\begin{equation}
\begin{split}
  \|\bm{y} - \bm{y}_{t}\|^2 
  & \leq  \|\bm{y} - \bm{y}_{t-1}  \|^2  - \|\bm{y}_t - \bm{y}_{t-1}\|^2  + 2\eta_y \langle \bm{y}_t - \bm{y},\nabla_y f(\bm{x}_t, \bm{y}_t) \rangle \\&\quad- 2\eta_y \langle \bm{y}_t - \bm{y}, (\nabla_y f(\bm{x}_{t},\bm{y}_{t})-\nabla_y f(\bm{x}_{t-1},\bm{y}_{t-1}))  \rangle \\
  &\quad+ 2\eta_y \langle \bm{y}_t - \bm{y},  \nabla_y f(\bm{x}_{t-1},\bm{y}_{t-1})-\nabla_y f(\bm{x}_{t-2},\bm{y}_{t-2}) \rangle \\
  & \leq  \|\bm{y} - \bm{y}_{t-1}  \|^2  - \|\bm{y}_t - \bm{y}_{t-1}\|^2 + 2\eta_y \langle \bm{y}_t - \bm{y},\nabla_y f(\bm{x}_t, \bm{y}_t) \rangle \\
  &\quad- 2\eta_y \langle \bm{y}_t - \bm{y},(\nabla_y f(\bm{x}_{t},\bm{y}_{t})-\nabla_y f(\bm{x}_{t-1},\bm{y}_{t-1}))  \rangle \\
  &\quad+  2\eta_y \langle \bm{y}_{t-1} - \bm{y},  \nabla_y f(\bm{x}_{t-1},\bm{y}_{t-1})-\nabla_y f(\bm{x}_{t-2},\bm{y}_{t-2}) \rangle \\
  &\quad+  \eta_y\ell (\| \bm{y}_{t} - \bm{y}_{t-1}\|^2 + \|   \bm{x}_{t-1} - \bm{x}_{t-2}  \|^2 + \| \bm{y}_{t-1} - \bm{y}_{t-2}\|^2) \\
  & \leq  \|\bm{y} - \bm{y}_{t-1}  \|^2  - \frac{1}{2}\|\bm{y}_t - \bm{y}_{t-1}\|^2+ \frac{1}{2}\| \bm{y}_{t-1} - \bm{y}_{t-2}\|^2 + 2\eta_y \langle \bm{y}_t - \bm{y},\nabla_y f(\bm{x}_t, \bm{y}_t) \rangle  \\
  & - 2\eta_y \langle \bm{y}_t - \bm{y},(\nabla_y f(\bm{x}_{t},\bm{y}_{t})-\nabla_y f(\bm{x}_{t-1},\bm{y}_{t-1}))  \rangle \\
  &\quad+  2\eta_y \langle \bm{y}_{t-1} - \bm{y},  \nabla_y f(\bm{x}_{t-1},\bm{y}_{t-1})-\nabla_y f(\bm{x}_{t-2},\bm{y}_{t-2}) \rangle  +     \eta_y \eta_x^2 \ell G^2,   
\end{split}
\end{equation}
which concludes the proof.

\end{proof}
\end{lemma}

\begin{lemma}
\label{lemma:descent}
For OGDA (Algorithm~\ref{alg:ogda}), under the same assumptions made as in Theorem~\ref{thm:ncc_ogda}, the following statement holds for the generated sequence $\{\bm{x}_t\}, \{\bm{y}_t\}$ during algorithm proceeding:
\begin{align*} 
    \Phi_{1/2\ell}(\bm{x}_t)  
     & \leq \Phi_{1/2\ell}({\bm{x}}_{t-1}) + 2\eta_x \ell \left(\Phi(\bm{x}_{t-1}) - f(\bm{x}_{t-1},\bm{y}_{t-1})\right) - \frac{\eta_x}{8}\|\nabla \Phi_{1/2\ell}(\bm{x}_{t-1})\|^2+3 \ell \eta_x^2 G^2 \\
    & + \frac{\eta_x}{2} \| \nabla_x f(\bm{x}_{t-1},\bm{y}_{t-1}) - \nabla_x f(\bm{x}_{t-2},\bm{y}_{t-2})\|^2.
\end{align*}
 \begin{proof}
 Let $\hat{\bm{x}}_{t-1} = \arg\min_{\bm{x}\in \mathbb{R}^d} \Phi(\bm{x})+\ell \|\bm{x}-\bm{x}_{t-1}\|^2$. Notice that:
\begin{equation*}
\begin{split}
    \Phi_{1/2\ell}(\bm{x}_t) &\le \Phi_{1/2\ell}(\hat{\bm{x}}_{t-1})+\ell\|\hat{\bm{x}}_{t-1}-\bm{x}_{t}\|^2\\
    &\leq \Phi_{1/2\ell}(\hat{\bm{x}}_{t-1})+\ell (\|\hat{\bm{x}}_{t-1}-\bm{x}_{t-1} \|^2\\
    &\quad+  2\eta_x\langle 2\nabla_x f(\bm{x}_{t-1},\bm{y}_{t-1}) - \nabla_x f(\bm{x}_{t-2},\bm{y}_{t-2}), \hat{\bm{x}}_{t-1}-\bm{x}_{t-1} \rangle   +3 \eta_x^2 G^2 )
\end{split}
\end{equation*}
According to smoothness of $f(\cdot,\bm{y})$, we have:
\begin{align*}
    \langle \hat{\bm{x}}_{t-1}-\bm{x}_{t-1}, \nabla_x f(\bm{x}_{t-1},\bm{y}_{t-1})  \rangle &\leq f(\hat{\bm{x}}_{t-1},\bm{y}_{t-1}) - f(\bm{x}_{t-1},\bm{y}_{t-1}) + \frac{\ell}{2}\|\hat{\bm{x}}_{t-1}-\bm{x}_{t-1}\|^2 \\
    & \leq \Phi(\bm{x}_{t-1}) - f(\bm{x}_{t-1},\bm{y}_{t-1}) - \frac{\ell}{2}\|\hat{\bm{x}}_{t-1}-\bm{x}_{t-1}\|^2.
\end{align*}
So we have
\begin{align*}
    \Phi_{1/2\ell}(\bm{x}_t) &\le \Phi_{1/2\ell}(\hat{\bm{x}}_{t-1})+\ell\|\hat{\bm{x}}_{t-1}-\bm{x}_{t}\|^2\\
    &\leq \Phi_{1/2\ell}(\hat{\bm{x}}_{t-1})+\ell \|\bm{x}_{t-1}-\hat{\bm{x}}_{t-1} \|^2 \\
    &\quad+ 2\eta_x \ell \left(\Phi(\bm{x}_{t-1}) - f(\bm{x}_{t-1},\bm{y}_{t-1}) - \frac{\ell}{2}\|\hat{\bm{x}}_{t-1}-\bm{x}_{t-1}\|^2\right)+3 \ell \eta_x^2 G^2 \\
    & \quad+ \eta_x \ell \left(\frac{1}{2\ell} \| \nabla_x f(\bm{x}_{t-1},\bm{y}_{t-1}) - \nabla_x f(\bm{x}_{t-2},\bm{y}_{t-2})\|^2 +\frac{\ell}{2}\| \bm{x}_{t-1}-\hat{\bm{x}}_{t-1} \|^2   \right )\\
    & \leq \Phi_{1/2\ell}({\bm{x}}_{t-1}) + 2\eta_x \ell \left(\Phi(\bm{x}_{t-1}) - f(\bm{x}_{t-1},\bm{y}_{t-1})\right) - \frac{\eta_x\ell^2}{2}\|\hat{\bm{x}}_{t-1}-\bm{x}_{t-1}\|^2+3 \ell \eta_x^2 G^2 \\
    &\quad + \frac{\eta_x }{2} \| \nabla_x f(\bm{x}_{t-1},\bm{y}_{t-1}) - \nabla_x f(\bm{x}_{t-2},\bm{y}_{t-2})\|^2.  \\
\end{align*}
Using the fact that $\|\hat{\bm{x}}_{t-1}-\bm{x}_{t-1}\| = \frac{1}{2\ell}\|\nabla \Phi_{1/2\ell}(\bm{x}_{t-1})\|$ will conclude the proof.

\end{proof}

\end{lemma}

\begin{lemma}[Iterates gap]\label{lemma: ogda y gap}
For OGDA (Algorithm~\ref{alg:ogda}), under Theorem~\ref{thm:ncc_ogda}'s assumptions, the following statement holds for the generated sequence $\{\bm{y}_t\}$ during algorithm proceeding:
\begin{align*}
      \sum_{t=0}^{T}\|\bm{y}_t - \bm{y}_{t-1}\|^2   
      & \leq \sum_{t=0}^{T}  \left(\sum_{j=0}^{T}(2\eta_y^2 \ell^2)^{j} \right) 4\eta_y^2 \ell \left( \Phi(\bm{x}_{t})-f(\bm{x}_{t},\bm{y}_{t}) \right) \\
       & \quad +\sum_{t=0}^{T}\left(\sum_{j=0}^{T}(2\eta_y^2 \ell^2)^{j} \right)  2\eta_x^2 \eta_y^2 \ell^2 G^2.
 \end{align*}
 \begin{proof}
 Observe that 
 \begin{align*}
      \|\bm{y}_t - \bm{y}_{t-1}\|^2 &= \eta_y^2 \left \|  2  \nabla_y f(\bm{x}_{t-1},\bm{y}_{t-1})-  \nabla_y f(\bm{x}_{t-2},\bm{y}_{t-2})\right\|^2 \\
      & \leq 2\eta_y^2 \left \|    \nabla_y f(\bm{x}_{t-1},\bm{y}_{t-1}) \right\|^2 + 2\eta_y^2 \left \|    \nabla_y f(\bm{x}_{t-1},\bm{y}_{t-1})-  \nabla_y f(\bm{x}_{t-2},\bm{y}_{t-2})\right\|^2 \\
      & \leq 4\eta_y^2 \ell\left( \Phi(\bm{x}_{t-1})-f(\bm{x}_{t-1},\bm{y}_{t-1}) \right) + 2\eta_y^2 \ell^2 \left( \left \|  \bm{x}_{t-1} -   \bm{x}_{t-2} \right\|^2 +   \left \|  \bm{y}_{t-1} -   \bm{y}_{t-2} \right\|^2\right)\\
      & \leq 2\eta_y^2 \ell^2\left \|  \bm{y}_{t-1} -   \bm{y}_{t-2} \right\|^2+ 4\eta_y^2 \ell\left( \Phi(\bm{x}_{t-1})-f(\bm{x}_{t-1},\bm{y}_{t-1}) \right) + 2\eta_x^2 \eta_y^2 \ell^2 G^2.
 \end{align*}
 Unrolling the recursion yields:
 \begin{align*}
      \|\bm{y}_t - \bm{y}_{t-1}\|^2   
      & \leq (2\eta_y^2 \ell^2)^{t-1}\left \|  \bm{y}_{0} -   \bm{y}_{-1} \right\|^2+ \sum_{j=1}^{t}(2\eta_y^2 \ell^2)^{t-j} 4\eta_y^2 \ell \left( \Phi(\bm{x}_{j-1})-f(\bm{x}_{j-1},\bm{y}_{j-1}) \right) \\
      &\quad +\sum_{j=1}^{t}(2\eta_y^2 \ell^2)^{t-j} 2\eta_x^2 \eta_y^2 \ell^2 G^2.
 \end{align*}
 Since $\bm{y}_{0} =\bm{y}_{-1}$, we have:
  \begin{align*}
      \|\bm{y}_t - \bm{y}_{t-1}\|^2   
      & \leq   \sum_{j=1}^{t}(2\eta_y^2 \ell^2)^{t-j} 4\eta_y^2 \ell\left( \Phi(\bm{x}_{j-1})-f(\bm{x}_{j-1},\bm{y}_{j-1}) \right) +\sum_{j=1}^{t}(2\eta_y^2 \ell^2)^{t-j} 2\eta_x^2 \eta_y^2 \ell^2 G^2.
 \end{align*}
 Finally, summing the above inequality over $t = 0$ to $T$ yields:
\begin{align*}
      \sum_{t=0}^{T}\|\bm{y}_t - \bm{y}_{t-1}\|^2   
      & \leq \sum_{t=0}^{T}  \left(\sum_{j=0}^{T}(2\eta_y^2 \ell^2)^{j} \right) 4\eta_y^2 \ell\left( \Phi(\bm{x}_{t})-f(\bm{x}_{t},\bm{y}_{t}) \right) \\
      &\quad +\sum_{t=0}^{T}\left(\sum_{j=0}^{T}(2\eta_y^2 \ell^2)^{j} \right)  2\eta_x^2 \eta_y^2 \ell^2 G^2.
 \end{align*}
 \end{proof}

\end{lemma}

\begin{lemma} \label{lemma: ogda gap}
For OGDA (Algorithm~\ref{alg:ogda}), under Theorem~\ref{thm:ncc_ogda}'s assumptions, the following statement holds for the generated sequence $\{\bm{y}_t\}$ during algorithm proceeding and $\forall s \leq t$:
 \begin{align*}
     \Phi(\bm{x}_t) - f(\bm{x}_t,\bm{y}_t) &\leq 2\eta_x (t-s)G^2 + \frac{1}{2\eta_y} \left(\|\bm{y}_{t-1}-y^*(\bm{x}_s) \|^2-\|\bm{y}_t - \bm{y}^*(x_s)\|^2 -   \frac{1}{2} \|\bm{y}_t - \bm{y}_{t-1}\|^2 \right) \\
     &\quad+ \frac{1}{2 \eta_y }\left (\frac{1}{2} \|\bm{y}_{t-1} - \bm{y}_{t-2}\|^2 +\eta_y \eta_x^2 \ell G^2\right) -   \langle \nabla_y f(\bm{x}_t,\bm{y}_t) - \nabla_y f(\bm{x}_{t-1},\bm{y}_{t-1}), \\
     &\quad \bm{y}_t - \bm{y}^*(\bm{x}_s) \rangle +   \langle \nabla_y f(\bm{x}_{t-1},\bm{y}_{t-1}) - \nabla_y f(\bm{x}_{t-2},\bm{y}_{t-2}), \bm{y}_{t-1} - \bm{y}^*(\bm{x}_s) \rangle.
 \end{align*}
 
 \begin{proof}
 Observe that:
  \begin{align*}
     \Phi(\bm{x}_{t }) - f(\bm{x}_{t },\bm{y}_{t }) &\leq f(\bm{x}_{t },\bm{y}^*(\bm{x}_{t }))-f(\bm{x}_{s},\bm{y}^*(\bm{x}_{t }))+f(\bm{x}_{s},\bm{y}^*(\bm{x}_{s})) - f(\bm{x}_{t },\bm{y}^*(\bm{x}_{s}))\\
     & \quad + f(\bm{x}_{t },\bm{y}^*(\bm{x}_{s})) - f(\bm{x}_{t },\bm{y}_{t })\\
     &\leq 2(t-s )\eta_x G^2 -\langle \bm{y}_{t } - \bm{y}, \nabla_y f(\bm{x}_{t },\bm{y}_{t })\rangle,
 \end{align*} 
where in the last step we use the concavity of $f(\bm{x}_t,\cdot)$.

 Plugging in Lemma~\ref{lemma: ogda dual} will conclude the proof as follows:
  \begin{align*}
     \Phi(\bm{x}_{t})-f(\bm{x}_{t},\bm{y}_{t})  
     &\leq 2(t-s)\eta_x G^2\\
     &\quad +\frac{1}{2\eta_y} \left( \|\bm{y}_{t-1}-\bm{y} \|^2 -\|\bm{y}_t - \bm{y}\|^2  - \frac{1}{2} \|\bm{y}_t - \bm{y}_{t-1}\|^2 + \frac{1}{2} \|\bm{y}_{t-1} - \bm{y}_{t-2}\|^2 \right) \\
     &\quad+ \frac{1}{2 \eta_y}\eta_y \eta_x^2 \ell G^2  -  \langle \nabla_y f(\bm{x}_t,\bm{y}_t) - \nabla_y f(\bm{x}_{t-1},\bm{y}_{t-1}), \bm{y}_t - \bm{y} \rangle \\
     &\quad +    \langle \nabla_y f(\bm{x}_{t-1},\bm{y}_{t-1}) - \nabla_y f(\bm{x}_{t-2},\bm{y}_{t-2}), \bm{y}_{t-1} - \bm{y} \rangle.
 \end{align*}
 \end{proof}

\end{lemma}
 
\begin{lemma}\label{lemma: ogda avg primal gap}
For OGDA (Algorithm~\ref{alg:ogda}), under the same assumptions made in Theorem~\ref{thm:ncc_ogda}, the following statement holds for the generated sequence $\{\bm{x}_t\}, \{\bm{y}_t\}$ during algorithm proceeding:
 \begin{align*}
    \frac{1}{T+1}\sum_{t=0}^{T} \Phi(\bm{x}_t) - f(\bm{x}_t,\bm{y}_t) \leq  \frac{1}{B} \left(2\eta_x B^2 G^2 + \frac{1}{2\eta_y}\left(D^2 + \eta_y \ell D^2 \right)+2  (3\eta_x G^2+ D)D\right).
 \end{align*}
 
 \begin{proof}
 Let $S = (T+1)/B$, and we choose $s = jB$, $j = 0,...,S$. Then by summing over $t$ on the both side of Lemma~\ref{lemma: ogda gap} we have:
 \begin{align*}
     &\frac{1}{T+1}\sum_{t=0}^{T} \Phi(\bm{x}_t) - f(\bm{x}_t,\bm{y}_t) = \frac{1}{T+1}\sum_{j=0}^{S} \sum_{t=jB}^{(j+1)B-1} \Phi(\bm{x}_t) - f(\bm{x}_t,\bm{y}_t)\\
     &\quad \leq \frac{1}{T+1}\sum_{j=0}^{S} \left[2\eta_x B^2 G^2 + \frac{1}{2\eta_y}\left(\|\bm{y}_{jB-1}-\bm{y}^*(x_{jB}) \|^2 + \frac{1}{2} \|\bm{y}_{jB-1} - \bm{y}_{jB-2}\|^2 \right)\right] \\
     &\quad \quad + \frac{1}{T+1}\sum_{j=0}^{S} \left( - \langle \nabla_y f(\bm{x}_{(j+1)B-1},y_{(j+1)B-1}) - \nabla_y f(\bm{x}_{(j+1)B-2},y_{(j+1)B-2}), \right. \\ 
     &\quad \quad  \quad \left. \bm{y}_{(j+1)B-1} - y^*(\bm{x}_{jB}) \rangle +  \langle \nabla_y f(\bm{x}_{jB-1},\bm{y}_{jB-1}) - \nabla_y f(\bm{x}_{jB-2},\bm{y}_{jB-2}), \bm{y}_{jB-1} - y^*(\bm{x}_{jB} \rangle \right) \\
      &\quad \leq \frac{1}{T+1}\sum_{j=0}^{S} \left(2\eta_x B^2 G^2 + \frac{1}{2\eta_y}\left(D^2 + \frac{1}{2} D^2 \right)+2  (3\eta_x G^2+ D)D\right) \\ 
      &\quad \leq \frac{1}{B} \left(2\eta_x B^2 G^2 + \frac{1}{2\eta_y}\left(D^2 + \eta_y \ell D^2 \right)+2  (3\eta_x G^2+ D)D\right). 
 \end{align*}
 \end{proof}
\end{lemma}

\subsubsection{Proof of Theorem~\ref{thm:ncc_ogda} for OGDA}\label{app: ogda}
In this section we are going to provide the proof of Theorem~\ref{thm:ncc_ogda} on the convergence rate of OGDA in both deterministic and stochastic settings.

We start by establishing the convergence rate in deterministic setting. Before, we first state the formal version of Theorem~\ref{thm:ncc_ogda} here:
\begin{theorem}[OGDA Deterministic (Theorem~\ref{thm:ncc_ogda} restated)] \label{Theorem: OGDA formal}
 Under Assumption~\ref{asm:3}, if we choose $\eta_x = \Theta\left(\min\left\{\frac{\epsilon}{\ell G},\frac{\epsilon^2}{\ell G^2},\frac{\epsilon^4}{D^2 G^2 \ell^3}\right\}\right)$, $\eta_y = \frac{1}{2\ell}$, then OGDA (Algorithm~\ref{alg:ogda}) guarantees to find $\epsilon$-stationary point, i.e., $\frac{1}{T+1} \sum_{t=0}^{T}  \|\nabla \Phi_{1/2\ell}(\bm{x}_{t})\|^2 \leq \epsilon^2$, with the gradient complexity bounded by:
 \begin{align*}
     O\left( \frac{\ell G^2 \hat{\Delta}_{\Phi}}{\epsilon^4}\max \left\{ 1,\frac{D^2  \ell^2}{\epsilon^2}\right\}   \right).
 \end{align*}  
 
 \begin{proof}
From Lemma~\ref{lemma:descent} we have:

 \begin{align*}
   \frac{1}{T+1} \sum_{t=0}^{T}  \|\nabla \Phi_{1/2\ell}(\bm{x}_{t})\|^2   
    &\leq \frac{\Phi_{1/2\ell}({\bm{x}}_{0})-\Phi_{1/2\ell}(\bm{x}_t)}{\eta_x T} +16 \ell \frac{1}{T}\sum_{t=0}^{T-1}\left(\Phi({\bm{x}}_{t}) - f(\bm{x}_{t},\bm{y}_{t})   \right)  + 24\eta_x\ell   G^2  \\
    & + 4  \frac{1}{T+1} \sum_{t=0}^{T}\| \nabla_x f(\bm{x}_{t},\bm{y}_{t}) - \nabla_x f(\bm{x}_{t-1},\bm{y}_{t-1})\|^2,\\
    & \leq \frac{\Phi_{1/2\ell}({\bm{x}}_{0})-\Phi_{1/2\ell}(\bm{x}_t)}{\eta_x T} +16 \ell \frac{1}{T}\sum_{t=0}^{T-1}\left(\Phi({\bm{x}}_{t}) - f(\bm{x}_{t},\bm{y}_{t})   \right)  + 24\eta_x\ell   G^2  \\
    & + 4  \frac{1}{T+1} \sum_{t=0}^{T}\ell^2(3\eta_x^2G^2+ \|\bm{y}_{t} - \bm{y}_{t-1}\|^2).
 \end{align*}
 Plugging in Lemma~\ref{lemma: ogda y gap} yields:
 \begin{align*}
   &\frac{1}{T+1} \sum_{t=0}^{T}  \|\nabla \Phi_{1/2\ell}(\bm{x}_{t})\|^2 
     \leq \frac{\Phi_{1/2\ell}({\bm{x}}_{0})-\Phi_{1/2\ell}(\bm{x}_t)}{\eta_x T} \\
     &\qquad +16 \ell \frac{1}{T}\sum_{t=0}^{T-1}\left(\Phi({\bm{x}}_{t}) - f(\bm{x}_{t},\bm{y}_{t})   \right)  + 24\eta_x\ell   G^2 + 12\eta_x^2 \ell^2 G^2 \\
    &\qquad + 4  \frac{1}{T+1} \ell^2\left(  \sum_{t=0}^{T}  \left(\sum_{j=0}^{T}(2\eta_y^2 \ell^2)^{j} \right)4\eta_y^2 \ell \left( \Phi(\bm{x}_{t})-f(\bm{x}_{t},\bm{y}_{t}) \right) +\sum_{t=0}^{T}\left(\sum_{j=0}^{T}(2\eta_y^2 \ell^2)^{j} \right)  2\eta_x^2 \eta_y^2 \ell^2 G^2\right),\\
 \end{align*}
 since we choose $\eta_y \ell \leq \frac{1}{2}$, we know that:
 \begin{align*}
     \sum_{j=0}^{T}\left(2\eta_y^2\ell^2\right)^j \leq 2.
 \end{align*}
 Hence we have:
   \begin{align*}
   \frac{1}{T+1} \sum_{t=0}^{T}  \|\nabla \Phi_{1/2\ell}(\bm{x}_{t})\|^2 
    & \leq \frac{\Phi_{1/2\ell}({\bm{x}}_{0})-\Phi_{1/2\ell}(\bm{x}_t)}{\eta_x T} +(16 \ell+32\eta_y^2\ell^3) \frac{1}{T+1}\sum_{t=0}^{T}\left(\Phi({\bm{x}}_{t}) - f(\bm{x}_{t},\bm{y}_{t})   \right)  \\
    &\quad+ 24\eta_x\ell   G^2 + 12\eta_x^2 \ell^2 G^2  +   16\eta_x^2 \eta_y^2 \ell^4 G^2.     
 \end{align*}
 Now we plug in Lemma~\ref{lemma: ogda avg primal gap} to replace $\Phi({\bm{x}}_{t}) - f(\bm{x}_{t},\bm{y}_{t})$:
 \begin{align*}
        \frac{1}{T+1} \sum_{t=0}^{T}  \|\nabla \Phi_{1/2\ell}(\bm{x}_{t})\|^2 
    & \leq \frac{\Phi_{1/2\ell}({\bm{x}}_{0})-\Phi_{1/2\ell}(\bm{x}_t)}{\eta_x T} \\
    &\quad+(16 \ell+32\eta_y^2\ell^3)   \frac{1}{B} \left(2\eta_x B^2 G^2 + \frac{1}{2\eta_y}\left(D^2 + \eta_y \ell D^2 \right)+2  (3\eta_x G^2+ D)D\right) \\  
    &\quad  + 24\eta_x\ell   G^2 + 12\eta_x^2 \ell^2 G^2  
    +   16\eta_x^2 \eta_y^2 \ell^4 G^2.
 \end{align*}

Choose $B = O\left(\frac{D}{G\sqrt{\eta_x \eta_y }}\right),   \eta_x = O\left(\min\left\{\frac{\epsilon}{\ell G},\frac{\epsilon^2}{\ell G^2},\frac{\epsilon^4}{D^2 G^2 \ell^3}\right\}\right)$, $\eta_y = \frac{1}{2\ell}$, and then we guarantee that $\frac{1}{T+1} \sum_{t=0}^{T}  \|\nabla \Phi_{1/2\ell}(\bm{x}_{t})\|^2 \leq \epsilon^2$ with the gradient complexity is bounded by:
 \begin{align*}
     O\left( \frac{\ell G^2 \hat{\Delta}_{\Phi}}{\epsilon^4}\max \left\{ 1,\frac{D^2  \ell^2}{\epsilon^2}\right\}   \right).
 \end{align*} 
 
 \end{proof}
 \end{theorem}

\textbf{Stochastic setting.} 

We now turn to presenting the proof of OGDA in stochastic setting. First let us introduce some useful lemmas.
\subsubsection{Useful Lemmas}
\begin{lemma}\label{lemma: sogda dual}
For Stochastic OGDA (Algorithm~\ref{alg:ogda}), under the same assumptions made in Theorem~\ref{thm:ncc_sogda}, if we choose $\eta \leq 1/4\ell$ the following statement holds for the generated sequence $\{\bm{y}_t\}$ during algorithm proceeding and for any $\bm{y} \in \mathcal{Y}$:
\begin{align*}
   \mathbb{E}\|\bm{y} - \bm{y}_{t}\|^2  
   &  {\leq} \mathbb{E} \|\bm{y} - \bm{y}_{t-1}  \|^2  - \frac{1}{4}\mathbb{E}\|\bm{y}_t - \bm{y}_{t-1}\|^2+ \frac{1}{4}\mathbb{E}\| \bm{y}_{t-1} - \bm{y}_{t-2}\|^2 + 2\eta_y \langle \bm{y}_t - \bm{y},\nabla_y f(\bm{x}_t, \bm{y}_t) \rangle \\
   & +\eta_y \eta_x^2 \ell (G^2+\sigma^2)+ 6\eta_y^2\sigma^2 - 2\eta_y \langle \bm{y}_t - \bm{y},\nabla_y f(\bm{x}_{t},\bm{y}_{t})-\nabla_y f(\bm{x}_{t-1},\bm{y}_{t-1})  \rangle \\
   & +  2\eta_y \langle \bm{y}_{t-1} - \bm{y},  \nabla_y f(\bm{x}_{t-1},\bm{y}_{t-1})-\nabla_y f(\bm{x}_{t-2},\bm{y}_{t-2}) \rangle.
\end{align*}

\begin{proof}
The proof is similar to deterministic setting. Here we use $\xi_{t-1}$ to denote the random sample at iteration $t$. According to updating rule of $\bm{y}$:
\begin{align*}
    \bm{y}_t = \mathcal{P}_{\mathcal{Y}}\left(\bm{y}_{t-1}+2\eta_y \nabla_y f(\bm{x}_{t-1},\bm{y}_{t-1};\xi_{t-1})-\eta_y \nabla_y f(\bm{x}_{t-2},\bm{y}_{t-2};\xi_{t-1})\right)
\end{align*}
Similarly to deterministic setting, we let
\begin{align*}
   & \tilde{\varepsilon}_{t-1} =\eta_y (\nabla_y f(\bm{x}_{t},\bm{y}_{t})-\nabla_y f(\bm{x}_{t-1},\bm{y}_{t-1};\xi_{t-1})) -\eta_y(\nabla_y f(\bm{x}_{t-1},\bm{y}_{t-1};\xi_{t-1})-\nabla_y f(\bm{x}_{t-2},\bm{y}_{t-2};\xi_{t-1}))\\
   & {\varepsilon}_{t-1} =\eta_y (\nabla_y f(\bm{x}_{t},\bm{y}_{t})-\nabla_y f(\bm{x}_{t-1},\bm{y}_{t-1})) -\eta_y(\nabla_y f(\bm{x}_{t-1},\bm{y}_{t-1})-\nabla_y f(\bm{x}_{t-2},\bm{y}_{t-2}))
\end{align*}
 and re-write the updating rule as:
\begin{align*}
    \bm{y}_t = \mathcal{P}_{\mathcal{Y}}\left(\bm{y}_{t-1}+ \eta_y \nabla_y f(\bm{x}_{t},\bm{y}_{t})-\tilde{\varepsilon}_{t-1}\right)  
\end{align*}

Due to the property of projection we have:
\begin{align*}
    (\bm{y}-\bm{y}_t)^\top ( \bm{y}_t - \bm{y}_{t-1} - \eta_y \nabla_y f(\bm{x}_t,\bm{y}_t)+\tilde{\varepsilon}_{t-1}) \geq 0
\end{align*}

Using the identity that $\langle \bm{a}, \bm{b} \rangle = \frac{1}{2}(\|\bm{a}+\bm{b}\|^2 - \|\bm{a}\|^2 - \|\bm{b}\|^2)$ we have:
\begin{align*}
   0 &\leq  \|\bm{y} - \bm{y}_{t-1} - \eta_y\nabla_y f(\bm{x}_{t},\bm{y}_{t})+\tilde{\varepsilon}_{t-1} \|^2 - \|\bm{y} - \bm{y}_{t}\|^2 -  \|\bm{y}_t - \bm{y}_{t-1}- \eta_y \nabla_y f(\bm{x}_t,\bm{y}_t)+\tilde{\varepsilon}_{t-1}\|^2   \\
   & =   \|\bm{y} - \bm{y}_{t-1}  \|^2 -  \|\bm{y} - \bm{y}_{t}\|^2 -  \|\bm{y}_t - \bm{y}_{t-1}\|^2 + 2\langle \bm{y}_t - \bm{y},\eta_y \nabla_y f(\bm{x}_t, \bm{y}_t) \rangle\\
   &+2 \langle \bm{y} - \bm{y}_{t-1},\tilde{\varepsilon}_{t-1} \rangle-2 \langle \bm{y}_t - \bm{y}_{t-1},\tilde{\varepsilon}_{t-1} \rangle . 
\end{align*}
Notice that
\begin{align*}
    -2 \langle \bm{y}_t - \bm{y}_{t-1},\tilde{\varepsilon}_{t-1} \rangle&= -2 \langle \bm{y}_t - \bm{y}_{t-1}, {\varepsilon}_{t-1} \rangle-2 \langle \bm{y}_t - \bm{y}_{t-1}, \tilde{\varepsilon}_{t-1}-{\varepsilon}_{t-1} \rangle\\
    & \leq  -2 \langle \bm{y}_t - \bm{y}_{t-1}, {\varepsilon}_{t-1} \rangle+  \frac{1}{2} \|\bm{y}_t - \bm{y}_{t-1}\|^2 + 2\| \tilde{\varepsilon}_{t-1}-{\varepsilon}_{t-1} \|^2\\
\end{align*}
So we have:
\begin{align*}
   0 &\leq  \|\bm{y} - \bm{y}_{t-1} - \eta_y\nabla_y f(\bm{x}_{t},\bm{y}_{t})+\tilde{\varepsilon}_{t-1} \|^2 - \|\bm{y} - \bm{y}_{t}\|^2 -  \|\bm{y}_t - \bm{y}_{t-1}- \eta_y \nabla_y f(\bm{x}_t,\bm{y}_t)+\tilde{\varepsilon}_{t-1}\|^2   \\
   & =   \|\bm{y} - \bm{y}_{t-1}  \|^2 -  \|\bm{y} - \bm{y}_{t}\|^2 -  \|\bm{y}_t - \bm{y}_{t-1}\|^2 + 2\langle \bm{y}_t - \bm{y},\eta_y \nabla_y f(\bm{x}_t, \bm{y}_t) \rangle\\
   &\quad + 2 \langle \bm{y} - \bm{y}_{t-1},\tilde{\varepsilon}_{t-1} \rangle-2 \langle \bm{y}_t - \bm{y}_{t-1}, {\varepsilon}_{t-1} \rangle +  \frac{1}{2} \|\bm{y}_t - \bm{y}_{t-1}\|^2 + 2\| \tilde{\varepsilon}_{t-1}-{\varepsilon}_{t-1} \|^2. 
\end{align*}
Taking expectation over $\xi_{t-1}$ yields:
\begin{align*}
   0 &\leq    \mathbb{E}\|\bm{y} - \bm{y}_{t-1}  \|^2 - \mathbb{E}\|\bm{y} - \bm{y}_{t}\|^2 - \frac{1}{2}\mathbb{E}\|\bm{y}_t - \bm{y}_{t-1}\|^2 + 2\langle \bm{y}_t - \bm{y},\eta_y \nabla_y f(\bm{x}_t, \bm{y}_t) \rangle\\
   &\quad  -2 \langle \bm{y}_t - \bm{y} , {\varepsilon}_{t-1} \rangle   + 6\eta_y^2\sigma^2. 
\end{align*}

Now we plug the definition of $\varepsilon_{t-1}$ into above inequality:
\begin{align*}
   \mathbb{E}\|\bm{y} - \bm{y}_{t}\|^2  
   & \leq  \mathbb{E}\|\bm{y} - \bm{y}_{t-1}  \|^2  - \mathbb{E}\|\bm{y}_t - \bm{y}_{t-1}\|^2 + 2\eta_y \langle \bm{y}_t - \bm{y},\nabla_y f(\bm{x}_t, \bm{y}_t) \rangle \\
   &\quad- 2\eta_y \langle \bm{y}_t - \bm{y}, \nabla_y f(\bm{x}_{t},\bm{y}_{t})-\nabla_y f(\bm{x}_{t-1},\bm{y}_{t-1})   \rangle+ 6\eta_y^2\sigma^2\\
   &\quad+ 2\eta_y \langle \bm{y}_t - \bm{y},  \nabla_y f(\bm{x}_{t-1},\bm{y}_{t-1})-\nabla_y f(\bm{x}_{t-2},\bm{y}_{t-2}) \rangle \\
   & \leq  \mathbb{E}\|\bm{y} - \bm{y}_{t-1}  \|^2  - \mathbb{E}\|\bm{y}_t - \bm{y}_{t-1}\|^2 + 2\eta_y \langle \bm{y}_t - \bm{y},\nabla_y f(\bm{x}_t, \bm{y}_t) \rangle \\
   &\quad- 2\eta_y \langle \bm{y}_t - \bm{y}, \nabla_y f(\bm{x}_{t},\bm{y}_{t})-\nabla_y f(\bm{x}_{t-1},\bm{y}_{t-1})   \rangle + 6\eta_y^2\sigma^2\\
   & \quad + 2\eta_y \langle \bm{y}_{t-1} - \bm{y},  \nabla_y f(\bm{x}_{t-1},\bm{y}_{t-1})-\nabla_y f(\bm{x}_{t-2},\bm{y}_{t-2}) \rangle \\
   &\quad+  \eta_y\ell (\mathbb{E}\| \bm{y}_{t} - \bm{y}_{t-1}\|^2 +\mathbb{E} \|   \bm{x}_{t-1} - \bm{x}_{t-2}  \|^2 + \mathbb{E}\| \bm{y}_{t-1} - \bm{y}_{t-2}\|^2) \\
   & \stackrel{\text{\ding{192}}}{\leq} \mathbb{E} \|\bm{y} - \bm{y}_{t-1}  \|^2  - \frac{1}{4}\mathbb{E}\|\bm{y}_t - \bm{y}_{t-1}\|^2+ \frac{1}{4}\mathbb{E}\| \bm{y}_{t-1} - \bm{y}_{t-2}\|^2 \\
   &\quad+ 2\eta_y \langle \bm{y}_t - \bm{y},\nabla_y f(\bm{x}_t, \bm{y}_t) \rangle +\eta_y \eta_x^2 \ell (G^2+\sigma^2)+ 6\eta_y^2\sigma^2 \\
   &\quad - 2\eta_y \langle \bm{y}_t - \bm{y},\nabla_y f(\bm{x}_{t},\bm{y}_{t})-\nabla_y f(\bm{x}_{t-1},\bm{y}_{t-1})  \rangle\\
   &\quad+  2\eta_y \langle \bm{y}_{t-1} - \bm{y},  \nabla_y f(\bm{x}_{t-1},\bm{y}_{t-1})-\nabla_y f(\bm{x}_{t-2},\bm{y}_{t-2}) \rangle ,\\
\end{align*}
where in $\text{\ding{192}}$ we use the fact that $\eta_y \ell \leq \frac{1}{4}$ and hence can conclude the proof.

\end{proof}

\end{lemma}



\begin{lemma}
\label{lemma: sogda descent}
For Stochastic OGDA (Algorithm~\ref{alg:ogda}), under same assumptions as in Theorem~\ref{thm:ncc_sogda}, the following statement holds for the genserated sequence $\{\bm{x}_t\}, \{\bm{y}_t\}$ during algorithm proceeding:
\begin{align*} 
    \mathbb{E}[\Phi_{1/2\ell}(\bm{x}_t)]  
    & \leq  \mathbb{E}[\Phi_{1/2\ell}({\bm{x}}_{t-1})] + 2\eta_x \ell \mathbb{E}\left(\Phi(\bm{x}_{t-1}) - f(\bm{x}_{t-1},\bm{y}_{t-1})\right) - \frac{\eta_x}{8} \mathbb{E} \|\nabla \Phi_{1/2\ell}(\bm{x}_{t-1})\|^2\\
    &\quad+3 \ell \eta_x^2(G^2+\sigma^2) + \frac{\eta_x }{2}\mathbb{E} \| \nabla_x f(\bm{x}_{t-1},\bm{y}_{t-1}) - \nabla_x f(\bm{x}_{t-2},\bm{y}_{t-2})\|^2.  
\end{align*}
 \begin{proof}
 Let $\hat{\bm{x}}_{t-1} = \arg\min_{\bm{x}\in \mathbb{R}^d} \Phi(\bm{x})+\ell \|\bm{x}-\bm{x}_{t-1}\|^2$. Notice that:
\begin{align*}
   \mathbb{E}[\Phi_{1/2\ell}(\bm{x}_t)] &\le \mathbb{E}[\Phi_{1/2\ell}(\hat{\bm{x}}_{t-1})]+\ell\mathbb{E}\|\hat{\bm{x}}_{t-1}-\bm{x}_{t}\|^2\\
    &\leq \mathbb{E}[\Phi_{1/2\ell}(\hat{\bm{x}}_{t-1})]\\
    &\quad +\ell ( \mathbb{E}\|\bm{x}_{t-1}-\hat{\bm{x}}_{t-1} \|^2 + 2\eta_x\langle 2\nabla_x f(\bm{x}_{t-1},\bm{y}_{t-1}) - \nabla_x f(\bm{x}_{t-2},\bm{y}_{t-2}), \bm{x}_{t-1}-\hat{\bm{x}}_{t-1} \rangle  \\
    &\quad +3 \eta_x^2 (G^2+\sigma^2) )
\end{align*}
According to smoothness of $f(\cdot,\bm{y})$, we have:
\begin{align*}
    \langle \hat{\bm{x}}_{t-1} - {\bm{x}}_{t-1}, \nabla_x f(\bm{x}_{t-1},\bm{y}_{t-1})  \rangle &\leq f(\hat{\bm{x}}_{t-1},\bm{y}_{t-1}) - f(\bm{x}_{t-1},\bm{y}_{t-1}) + \frac{\ell}{2}\|\hat{\bm{x}}_{t-1}-\bm{x}_{t-1}\|^2 \\
    & \leq \Phi(\bm{x}_{t-1}) - f(\bm{x}_{t-1},\bm{y}_{t-1}) - \frac{\ell}{2}\|\hat{\bm{x}}_{t-1}-\bm{x}_{t-1}\|^2.
\end{align*}
So we have
\begin{align*}
     \mathbb{E}[\Phi_{1/2\ell}(\bm{x}_t)] &\le  \mathbb{E}[\Phi_{1/2\ell}(\hat{\bm{x}}_{t-1})]+\ell \mathbb{E}\|\hat{\bm{x}}_{t-1}-\bm{x}_{t}\|^2\\
    &\leq  \mathbb{E}[\Phi_{1/2\ell}(\hat{\bm{x}}_{t-1})]+\ell  \mathbb{E}\|\bm{x}_{t-1}-\hat{\bm{x}}_{t-1} \|^2 \\
    &\quad+ 2\eta_x \ell  \mathbb{E}\left(\Phi(\bm{x}_{t-1}) - f(\bm{x}_{t-1},\bm{y}_{t-1}) - \frac{\ell}{2} \mathbb{E}\|\hat{\bm{x}}_{t-1}-\bm{x}_{t-1}\|^2\right)+3 \ell \eta_x^2(G^2+\sigma^2) \\
    &\quad+ \eta_x \ell \left(\frac{1}{2\ell}  \mathbb{E}\| \nabla_x f(\bm{x}_{t-1},\bm{y}_{t-1}) - \nabla_x f(\bm{x}_{t-2},\bm{y}_{t-2})\|^2 +\frac{\ell}{2} \mathbb{E} \| \bm{x}_{t-1}-\hat{\bm{x}}_{t-1} \|^2   \right )\\
    & \leq  \mathbb{E}[\Phi_{1/2\ell}({\bm{x}}_{t-1})] + 2\eta_x \ell \mathbb{E}\left(\Phi(\bm{x}_{t-1}) - f(\bm{x}_{t-1},\bm{y}_{t-1})\right) - \frac{\eta_x\ell^2}{2} \mathbb{E} \|\hat{\bm{x}}_{t-1}-\bm{x}_{t-1}\|^2 \\
    &\quad+3 \ell \eta_x^2(G^2+\sigma^2) + \frac{\eta_x }{2}\mathbb{E} \| \nabla_x f(\bm{x}_{t-1},\bm{y}_{t-1}) - \nabla_x f(\bm{x}_{t-2},\bm{y}_{t-2})\|^2.  \\
\end{align*}

\end{proof}

\end{lemma}

\begin{lemma}
For Stochastic OGDA (Algorithm~\ref{alg:ogda}), under Theorem~\ref{thm:ncc_sogda}'s assumptions, the following statement holds for the generated sequence $\{\bm{y}_t\}$ during algorithm proceeding:
  \begin{align*}
     \sum_{t=0}^{T} \mathbb{E}\|\bm{y}_t - \bm{y}_{t-1}\|^2 &\leq 4\eta_y^2 \ell   \sum_{t=0}^{T}\left(\sum_{j=0}^{T}\left(2\eta_y^2\ell^2\right)^{j}\right)\mathbb{E}[\Phi(\bm{x}_t)-f(\bm{x}_t,\bm{y}_t)] \\
     &\quad+ \sum_{t=0}^{T} \left(\sum_{j=0}^{T}\left(2\eta_y^2\ell^2\right)^{j}\right) \left(6\eta_x^2 \eta_y^2 \ell^2 (G^2+\sigma^2) + {6\eta_y^2\sigma^2} \right)\\ 
 \end{align*}
\end{lemma}

 \begin{proof}
 According to updating rule of stochastic OGDA:
 \begin{align*}
     &\mathbb{E}\|\bm{y}_t - \bm{y}_{t-1}\|^2  \\
     & \leq \eta_y^2 \mathbb{E}\|2\nabla_y f(\bm{x}_{t-1},\bm{y}_{t-1};\xi_{t-1})-f(\bm{x}_{t-2},\bm{y}_{t-2};\xi_{t-1})\|^2\\
     & \leq 2\eta_y^2\mathbb{E}\|\nabla_y f(\bm{x}_{t-1},\bm{y}_{t-1}) \|^2+2\eta_y^2 {\sigma^2}+2\eta_y^2\mathbb{E}\|\nabla_y f(\bm{x}_{t-1},\bm{y}_{t-1} )-f(\bm{x}_{t-2},\bm{y}_{t-2} )\|^2+4\eta_y^2 {\sigma^2} \\
     & \leq 4\eta_y^2 \ell \mathbb{E}[\Phi(\bm{x}_{t-1})-f(\bm{x}_{t-1},\bm{y}_{t-1})] + 2\eta_y^2\ell^2(\mathbb{E}\|\bm{x}_{t-1}-\bm{x}_{t-2}\|^2+\mathbb{E}\|\bm{y}_{t-1}-\bm{y}_{t-2}\|^2) + {6\eta_y^2\sigma^2}\\
     &\leq 4\eta_y^2 \ell \mathbb{E}[\Phi(\bm{x}_{t-1})-f(\bm{x}_{t-1},\bm{y}_{t-1})] + 2\eta_y^2\ell^2(3\eta_x^2(G^2+\sigma^2)+\mathbb{E}\|\bm{y}_{t-1}-\bm{y}_{t-2}\|^2) + {6\eta_y^2\sigma^2}.
 \end{align*}
 Unrolling the recursion yields:
 \begin{align*}
      \mathbb{E}\|\bm{y}_t - \bm{y}_{t-1}\|^2 &\leq 4\eta_y^2 \ell   \sum_{j=0}^{t-1}\left(2\eta_y^2\ell^2\right)^{t-1-j}\mathbb{E}[\Phi(x_j)-f(x_j,y_j)] \\
      &\quad+ \sum_{j=0}^{t-1}\left(2\eta_y^2\ell^2\right)^{t-1-j} \left(6\eta_x^2 \eta_y^2 \ell^2 (G^2+\sigma^2) + {6\eta_y^2\sigma^2} \right) + (2\eta_y^2\ell^2)\mathbb{E}\|\bm{y}_0 - \bm{y}_{-1}\|^2.
 \end{align*}
 Since $\bm{y}_0 = \bm{y}_{-1}$, we can conclude the proof via summing $t$ from $0$ to $T-1$:
  \begin{align*}
     \sum_{t=0}^{T} \mathbb{E}\|\bm{y}_t - \bm{y}_{t-1}\|^2 &\leq 4\eta_y^2 \ell   \sum_{t=0}^{T}\left(\sum_{j=0}^{T}\left(2\eta_y^2\ell^2\right)^{j}\right)\mathbb{E}[\Phi(x_t)-f(x_t,y_t)] \\
     &\quad + \sum_{t=0}^{T} \left(\sum_{j=0}^{T}\left(2\eta_y^2\ell^2\right)^{j}\right) \left(6\eta_x^2 \eta_y^2 \ell^2 (G^2+\sigma^2) + {6\eta_y^2\sigma^2} \right).
 \end{align*}
 \end{proof}

\begin{lemma}
For Stochastic OGDA (Algorithm~\ref{alg:ogda}), under assumptions made in Theorem~\ref{thm:ncc_sogda}, the following statement holds for the generated sequence $\{\bm{y}_t\}$ during algorithm proceeding and $\forall s \leq t$:
 \begin{align*}
     &\mathbb{E}[\Phi(\bm{x}_t) - f(\bm{x}_t,\bm{y}_t)] \leq 2\eta_x (t-s)G\sqrt{G^2+\sigma^2}  +\frac{\eta_y \eta_x^2 \ell}{2} (G^2+\sigma^2)+ 3\eta_y \sigma^2\\
     &\quad+ \frac{1}{2\eta_y} \left(\mathbb{E}\|\bm{y}_{t-1}-\bm{y}^*(\bm{x}_s) \|^2-\mathbb{E}\|\bm{y}_t - \bm{y}^*(\bm{x}_s)\|^2 -   \frac{1}{4}  \mathbb{E}\|\bm{y}_t - \bm{y}_{t-1}\|^2 +  \frac{1}{4} \mathbb{E}\|\bm{y}_{t-1} - \bm{y}_{t-2}\|^2 \right)  \\
    &\quad +   \langle \nabla_y f(\bm{x}_t,\bm{y}_t) - \nabla_y f(\bm{x}_{t-1},\bm{y}_{t-1}), \bm{y}_t - \bm{y}^*(\bm{x}_s) \rangle \\
    &\quad-    \langle \nabla_y f(\bm{x}_{t-1},\bm{y}_{t-1}) - \nabla_y f(\bm{x}_{t-2},\bm{y}_{t-2}), \bm{y}_{t-1} - \bm{y}^*(\bm{x}_s) \rangle.
 \end{align*}
 
 \begin{proof}
 Observe that:
  \begin{align*}
     \mathbb{E}[\Phi(\bm{x}_t) - f(\bm{x}_t,\bm{y}_t)] &\leq \mathbb{E}[f(\bm{x}_{t },\bm{y}^*(\bm{x}_{t }))-f(\bm{x}_{s},\bm{y}^*(\bm{x}_{t }))]+\mathbb{E}[f(\bm{x}_{s},\bm{y}^*(\bm{x}_{s})) - f(\bm{x}_{t },\bm{y}^*(\bm{x}_{s}))]\\
     & \quad + \mathbb{E}[f(\bm{x}_{t },\bm{y}^*(\bm{x}_{s})) - f(\bm{x}_{t },\bm{y}_{t })]\\
     &\leq 2(t-s)\eta_x G\sqrt{G^2+\sigma^2} -\mathbb{E}\langle \bm{y}_{t } - \bm{y}, \nabla_y f(\bm{x}_{t },\bm{y}_{t })\rangle.
 \end{align*}

 Plugging in Lemma~\ref{lemma: sogda dual} will conclude the proof. 
 \end{proof}
 
\end{lemma}
 
\begin{lemma}\label{lemma: sogda primal gap}
For Stochastic OGDA (Algorithm~\ref{alg:ogda}), under Theorem~\ref{thm:ncc_sogda}'s assumptions, the following statement holds for the generated sequence $\{\bm{x}_t\}, \{\bm{y}_t\}$ during algorithm proceeding:
 \begin{align*}
    \frac{1}{T+1}\sum_{t=0}^{T}  \mathbb{E}[\Phi(\bm{x}_t) - f(\bm{x}_t,\bm{y}_t)]  &\leq \frac{1}{B} \left(2\eta_x B^2 G\sqrt{G^2+\sigma^2} + \frac{5 D^2}{8\eta_y} +2  (3\eta_x G\sqrt{G^2+\sigma^2}+ D)D\right) \\
    &\quad +\frac{\eta_y \eta_x^2 \ell}{2} (G^2+\sigma^2)+ 3\eta_y \sigma^2.
 \end{align*}
 
 \begin{proof}
  Let $S = (T+1)/B$, and we choose $s = jB$, $j = 0,...,S$. Then by summing over $t$ on the both side of Lemma~\ref{lemma: sogda primal gap} we have:
 \begin{align*}
     &\frac{1}{T+1}\sum_{t=0}^{T-1} \mathbb{E}[\Phi(\bm{x}_t) - f(\bm{x}_t,\bm{y}_t)] = \frac{1}{T}\sum_{j=0}^{S} \sum_{t=jB}^{(j+1)B-1} \mathbb{E}[\Phi(\bm{x}_t) - f(\bm{x}_t,\bm{y}_t)]\\
     & \leq \frac{1}{T}\sum_{j=0}^{S} \left[2\eta_x B^2 G\sqrt{G^2+\sigma^2} + \frac{1}{2\eta_y}\left(\|y_{jB}-y^*(x_{jB}) \|^2 + \frac{1}{4} \|y_{jB} - y_{jB-1}\|^2 \right)\right] \\
     & \quad+\frac{\eta_y \eta_x^2 \ell}{2} (G^2+\sigma^2)+ 3\eta_y \sigma^2 \\
     & \quad+ \frac{1}{T}\sum_{j=0}^{S} \left[  - \langle \nabla_y f(\bm{x}_{(j+1)B-1},y_{(j+1)B-1}) - \nabla_y f(\bm{x}_{(j+1)B-2},y_{(j+1)B-2}), \bm{y}_{(j+1)B-1} - y^*(\bm{x}_{jB}) \rangle\right. \\
     & \qquad \qquad \left.+ \langle \nabla_y f(\bm{x}_{jB-1},\bm{y}_{jB-1}) - \nabla_y f(\bm{x}_{jB-2},\bm{y}_{jB-2}), \bm{y}_{jB-1} - y^*(\bm{x}_{jB} \rangle  \right] \\
      & \leq \frac{1}{T}\sum_{j=0}^{S} \left[2\eta_x B^2 G\sqrt{G^2+\sigma^2} + \frac{5 D^2}{8\eta_y} +2  (3\eta_x G\sqrt{G^2+\sigma^2}+ D)D\right]\\
      &\quad+\frac{\eta_y \eta_x^2 \ell}{2} (G^2+\sigma^2)+ 3\eta_y \sigma^2 \\ 
      & \leq \frac{1}{B} \left[2\eta_x B^2 G\sqrt{G^2+\sigma^2} + \frac{5 D^2}{8\eta_y}+2  (3\eta_x G\sqrt{G^2+\sigma^2}+ D)D\right]\\
      &\quad +\frac{\eta_y \eta_x^2 \ell}{2} (G^2+\sigma^2)+ 3\eta_y \sigma^2.
 \end{align*}
 \end{proof}
\end{lemma}

\subsubsection{Proof of Theorem~\ref{thm:ncc_sogda} for OGDA}\label{app: sogda}
 In this section we are going to provide the proof for Theorem~\ref{thm:ncc_sogda}, the convergence rate of OGDA in stochastic setting. We first introduce the formal version of Theorem~\ref{thm:ncc_sogda} here:
\begin{theorem}[OGDA Stochastic (Theorem~\ref{thm:ncc_sogda} restated)] \label{Theorem: SOGDA formal}
  Under Assumption~\ref{asm:2} and \ref{asm:3}, if we choose $\eta_x=O(\min\{\frac{\epsilon^2}{\ell (G^2+\sigma^2)},\frac{\epsilon^4}{D^2\ell^3  G\sqrt{G^2+\sigma^2} },\frac{\epsilon^6}{D^2\ell^3\sigma^2 G\sqrt{G^2+\sigma^2} }\})$, $\eta_y = O(\min\{\frac{1}{4\ell},\frac{\epsilon^2}{\ell\sigma^2}\})$, then Stochastic OGDA (Algorithm~\ref{alg:ogda}) guarantees to find $\epsilon$-stationary point, i.e., $\frac{1}{T+1} \sum_{t=0}^{T} \mathbb{E} \|\nabla \Phi_{1/2\ell}(\bm{x}_{t})\|^2 \leq \epsilon^2$, with the gradient complexity bounded by:
 \begin{align*}
     O\left(\frac{D^2\ell^3  G\sqrt{G^2+\sigma^2}}{\epsilon^6} \max\left\{ 1, \frac{  \sigma^2   }{\epsilon^2} \right\}\right).
 \end{align*} 
 
 \begin{proof}
 
Similar to the proof in deterministic setting, first according to Lemma~\ref{lemma: sogda descent} we have:

 \begin{align*}
   \frac{1}{T+1} \sum_{t=0}^{T} \mathbb{E} \|\nabla \Phi_{1/2\ell}(\bm{x}_{t})\|^2  
    &\leq \frac{\Phi_{1/2\ell}(\bm{x}_{0})-\Phi_{1/2\ell}(\bm{x}_{T+1})}{\eta_x (T+1)} \\
    &\quad+16 \ell \frac{1}{T+1}\sum_{t=0}^{T }\left(\Phi(\bm{x}_{t}) - f(\bm{x}_{t},\bm{y}_{t})\right)  + 12\eta_x^2\ell^2 (G^2 + \sigma^2) + 24\ell \eta_x (G^2 + \sigma^2)  \\
    &\quad + 4\ell^2\frac{1}{T+1} \left(4\eta_y^2 \ell   \sum_{t=0}^{T+1}\left(\sum_{j=0}^{T}\left(2\eta_y^2\ell^2\right)^{j}\right)\mathbb{E}[\Phi(\bm{x}_t)-f(\bm{x}_t,\bm{y}_t)] \right. \\
    &\quad+ \left. \sum_{t=0}^{T} \left(\sum_{j=0}^{T}\left(2\eta_y^2\ell^2\right)^{j}\right) \left(6\eta_x^2 \eta_y^2 \ell^2 (G^2+\sigma^2) + 6\eta_y^2\sigma^2 \right)\right).
\end{align*}

Since we choose $\eta_y \ell \leq \frac{1}{4}$, it follows that:
 \begin{align*}
     \sum_{j=0}^{T }\left(2\eta_y^2\ell^2\right)^j \leq 2 .
 \end{align*}
 As a result, we can further simplify the bound as:
 \begin{align*}
   \frac{1}{T+1} \sum_{t=0}^{T} \mathbb{E} \|\nabla \Phi_{1/2\ell}(\bm{x}_{t })\|^2 
    & \leq \frac{\Phi_{1/2\ell}({x}_{0})-\Phi_{1/2\ell}(\bm{x}_{T+1})}{\eta_x (T+1)}\\
    &\quad +(16 \ell+ 32\eta_y^2\ell^3) \frac{1}{T+1}\sum_{t=0}^{T }\left(\Phi({x}_{t}) - f(\bm{x}_{t},\bm{y}_{t})   \right)  \\
    &  \quad  + 12\eta_x^2\ell^2 (G^2 + \sigma^2) + 24\ell \eta_x (G^2 + \sigma^2)+ 8\ell^2  \left(6\eta_x^2 \eta_y^2 \ell^2 (G^2+\sigma^2) + {6\eta_y^2\sigma^2} \right) .
 \end{align*}
 Plugging in Lemma~\ref{lemma: sogda primal gap} yields:
  \begin{align*}
   \frac{1}{T+1} \sum_{t=0}^{T} \mathbb{E} \|\nabla \Phi_{1/2\ell}(\bm{x}_{t})\|^2  &\leq \frac{\Phi_{1/2\ell}({x}_{0})-\Phi_{1/2\ell}(\bm{x}_{T+1})}{\eta_x (T+1)} \\
    &\quad  +(16 \ell+ 32\eta_y^2\ell^3)  \frac{1}{B} \left(2\eta_x B^2 G\sqrt{G^2+\sigma^2} + \frac{5 D^2}{8\eta_y} +2  (3\eta_x G\sqrt{G^2+\sigma^2}+ D)D\right)  \\
    &\quad+ (16 \ell+ 32\eta_y^2\ell^3) (\frac{\eta_y \eta_x^2 \ell}{2} (G^2+\sigma^2)+ 3\eta_y \sigma^2 )  \\
    & \quad  + 12\eta_x^2\ell^2 (G^2 + \sigma^2) + 24\ell \eta_x (G^2 + \sigma^2)+ 8\ell^2  \left(6\eta_x^2 \eta_y^2 \ell^2 (G^2+\sigma^2) + {6\eta_y^2\sigma^2} \right) .
 \end{align*}

 Choose $B = O(\frac{D}{\sqrt{\eta_x \eta_y G \sqrt{G^2+\sigma^2}}})$, $\eta_x = O(\min\{\frac{\epsilon^2}{\ell (G^2+\sigma^2)},\frac{\epsilon^4}{D^2\ell^3  G\sqrt{G^2+\sigma^2} },\frac{\epsilon^6}{D^2\ell^3\sigma^2 G\sqrt{G^2+\sigma^2} }\})$, $\eta_y = O(\min\{\frac{1}{4\ell},\frac{\epsilon^2}{\ell\sigma^2}\})$,  and then it is guaranteed that $\frac{1}{T+1} \sum_{t=0}^{T} \mathbb{E} \|\nabla \Phi_{1/2\ell}(\bm{x}_{t})\|^2 \leq \epsilon^2$ with the gradient complexity is bounded by
 \begin{align*}
     O\left(\frac{D^2\ell^3  G\sqrt{G^2+\sigma^2}}{\epsilon^6} \max\left\{ 1, \frac{  \sigma^2}{\epsilon^2} \right\}\right).
 \end{align*} 
 \end{proof}
 \end{theorem}

\subsection{Proof of convergence of EG}\label{app:NCC_Upper_EG}

In this section, the convergence of EG in NC-C setting has been established. Before presenting the complete proofs, here we briefly discuss the proof sketch. 

\paragraph{Proof sketch}~Similar to OGDA, we have the following  lemma on $\Phi_{1/2\ell}$:
\begin{align*} 
  \frac{1}{T+1}\sum_{t=0}^T \| \nabla \Phi_{1/2\ell}({\bm{x}}_{t-\frac{1}{2}})\|^2    &\leq \Phi_{1/2\ell}({\bm{x}}_{-\frac{1}{2}})-\Phi_{1/2\ell}(\bm{x}_{T+\frac{1}{2}}) \\
  &\quad +  O( \ell+  \eta_y^2\ell^3) \frac{1}{T+1}\sum_{t=0}^T \delta_{t-\frac{1}{2}}  +O(\ell \eta_x^2 G^2).
\end{align*}
Now we need to examine $\delta_{t-\frac{1}{2}}$. To bound this term, we have the following recursion:
  \begin{align*}
     \Phi(\bm{x}_{t+\frac{1}{2}}) - f(\bm{x}_{t+\frac{1}{2}},\bm{y}_{t+\frac{1}{2}}) &\leq O((t-s)\eta_x G^2)  \\
     &\quad + \frac{1}{2\eta_y} \left(     \|\bm{y}_{t} - \bm{y}^*(\bm{x}_{s})\|^2-\|\bm{y}_{t+1}-\bm{y}^*(\bm{x}_{s}) \|^2     + \frac{\eta_x^2G^2}{2} \right),
 \end{align*}
 which is derived by the descent property of EG on concave function.
 Similar to OGDA, here we also obtain neat recursion, which will yield our desired complexity bound.

In the following, we present the key lemmas, and complete convergence proof of EG. First let us introduce some useful lemmas for the deterministic setting.

\subsubsection{Useful Lemmas}
 \begin{proposition}[\cite{chen2017accelerated}, Proposition 4.2]\label{prop: 3 point theorem}
 If $\bm{p} =\mathcal{P}_{\mathcal{Y}}(\bm{r} - \bm{u})$, $\bm{q} = \mathcal{P}_{\mathcal{Y}}(\bm{r} - \bm{v})$, and
 \begin{align*}
     \|\bm{u} - \bm{v}\|^2 \leq C_1^2 \|\bm{p} - \bm{r}\|^2 + C_2^2,
 \end{align*}
 then for any $\bm{z} \in \mathbb{R}^d$ we have:
 \begin{align*}
     \langle \bm{v}, \bm{p} - \bm{z} \rangle \leq \|\bm{r} - \bm{z}\|^2 - \|\bm{q} - \bm{z}\|^2 - \left(\frac{1}{2} - \frac{C_1^2}{2}\right)\|\bm{r} - \bm{p}\|^2 + \frac{C_2^2}{2}.
 \end{align*}
 \end{proposition}

\begin{lemma} \label{lemma: EG dual}
For EG (Algorithm~\ref{alg:eg}), under Theorem~\ref{thm:ncc_sogda}'s assumptions, the following statement holds for the generated sequence $\{\bm{y}_t\},\{\bm{y}_{t+\frac{1}{2}}\}$ during algorithm proceeding and any $\bm{y} \in \mathcal{Y}$:
\begin{align*}
    \|\bm{y}_{t+1}-\bm{y} \|^2  &\leq  \|\bm{y}_{t} - \bm{y}\|^2  + 2\eta_y \langle \bm{y}_{t+\frac{1}{2}} - \bm{y}, \nabla_y f(\bm{x}_{t+\frac{1}{2}},\bm{y}_{t+\frac{1}{2}})\rangle- \left(\frac{1}{2} - \frac{\eta_y^2\ell^2}{2}\right)\|\bm{y}_t- \bm{y}_{t+\frac{1}{2}}\|^2 \\
    &\quad+ \frac{\eta_x^2\eta_y^2\ell^2 G^2}{2}.
\end{align*} 
\begin{proof}
According to Proposition~\ref{prop: 3 point theorem}, we set $\bm{r} = \bm{y}_t$, $\bm{q} = \bm{y}_{t+1}$, $\bm{p} = \bm{y}_{t+\frac{1}{2}}$ and $\bm{v} = -\eta_y\nabla_y f(\bm{x}_{t+\frac{1}{2}},\bm{y}_{t+\frac{1}{2}})$, $\bm{u} = -\eta_y\nabla_y f(\bm{x}_{t},\bm{y}_{t})$. We can verify that:
\begin{align*}
    \|\bm{u} - \bm{v} \|^2 &= \eta_y^2\|\nabla_y f(\bm{x}_{t+\frac{1}{2}},\bm{y}_{t+\frac{1}{2}}) - \nabla_y f(\bm{x}_{t},\bm{y}_{t})\|^2 \\
    &\leq  \eta_y^2(\ell^2\|\bm{y}_{t+\frac{1}{2}}-\bm{y}_{t}\|^2 + \ell^2\|\bm{x}_{t+\frac{1}{2}}-\bm{x}_{t}\|^2)\\
    &\leq  \eta_y^2(\ell^2\|\bm{p}-\bm{r}\|^2 + \ell^2\eta_x^2G^2),
\end{align*}
so if we set $C_1^2 =  \eta_y^2\ell^2$ and $C_2^2 = \eta_x^2\eta_y^2\ell^2 G^2   $, we have the following inequality holding for any $\bm{y} \in \mathcal{Y}$:
\begin{align*}
    \langle -\eta_y\nabla_y f(\bm{x}_{t+\frac{1}{2}},\bm{y}_{t+\frac{1}{2}}), \bm{y}_{t+\frac{1}{2}}-\bm{y} \rangle &\leq \|\bm{y}_t - \bm{y}\|^2 - \|\bm{y}_{t+1} - \bm{y}\|^2 - \left(\frac{1}{2} - \frac{\eta_y^2\ell^2}{2}\right)\|\bm{y}_t- \bm{y}_{t+\frac{1}{2}}\|^2 \\
    &\quad+ \frac{\eta_x^2\eta_y^2\ell^2 G^2}{2}.
\end{align*} 
\end{proof}
\end{lemma}

\begin{lemma}
\label{lemma: EG descent}
For EG (Algorithm~\ref{alg:eg}), under Theorem~\ref{thm:ncc_sogda}'s assumptions, the following statement holds for the generated sequence $\{\bm{x}_t\},\{\bm{y}_t\}, \{\bm{x}_{t+\frac{1}{2}}\},\{\bm{y}_{t+\frac{1}{2}}\}$ during algorithm proceeding:
\begin{align*} 
    \Phi_{1/2\ell}(\bm{x}_{t+\frac{1}{2}})  
    &  \leq \Phi_{1/2\ell}({\bm{x}}_{t-\frac{1}{2}}) + 2\eta_x \ell \left(\Phi(\bm{x}_{t-\frac{1}{2}}) - f(\bm{x}_{t-\frac{1}{2}},\bm{y}_{t-\frac{1}{2}})\right) - \frac{\eta_x}{8}\| \nabla \Phi_{1/2\ell}({\bm{x}}_{t-\frac{1}{2}})\|^2+3 \ell \eta_x^2 G^2 \\
    & + \frac{\eta_x }{2} \| \nabla_x f(\bm{x}_{t },\bm{y}_{t }) - \nabla_x f(\bm{x}_{t-1},\bm{y}_{t-1})\|^2.  \\
\end{align*}
  \begin{proof}
 Let $\hat{\bm{x}}_{t-\frac{1}{2}} = \arg\min_{\bm{x}\in \mathbb{R}^d} \Phi(\bm{x})+\ell \|\bm{x}-\bm{x}_{t-\frac{1}{2}}\|^2$. Notice that:
\begin{align*}
    \Phi_{1/2\ell}(\bm{x}_{t+\frac{1}{2}})  
    &\leq \Phi_{1/2\ell}(\hat{\bm{x}}_{t-\frac{1}{2}})+\ell\|\hat{\bm{x}}_{t-\frac{1}{2}}-\bm{x}_{t+\frac{1}{2}}\|^2\\
    &\leq \Phi_{1/2\ell}(\hat{\bm{x}}_{t-\frac{1}{2}}) +\ell \|\hat{\bm{x}}_{t-\frac{1}{2}}-\bm{x}_{t+\frac{1}{2}} \|^2 \\
    &\quad+\ell ( 2\eta_x\langle  \nabla_x f(\bm{x}_{t-\frac{1}{2}},\bm{y}_{t-\frac{1}{2}})+(\nabla_x f(\bm{x}_{t },\bm{y}_{t }-\nabla_x f(\bm{x}_{t-1},\bm{y}_{t-1})  , \hat{\bm{x}}_{t-\frac{1}{2}}-\bm{x}_{t-\frac{1}{2}} \rangle   + \eta_x^2 G^2 )\\
    &= \Phi_{1/2\ell}(\hat{\bm{x}}_{t-\frac{1}{2}})+\ell (\|\hat{\bm{x}}_{t-\frac{1}{2}}-\bm{x}_{t+\frac{1}{2}} \|^2 + 2\eta_x\langle  \nabla_x f(\bm{x}_{t-\frac{1}{2}},\bm{y}_{t-\frac{1}{2}})  , \hat{\bm{x}}_{t-\frac{1}{2}}-\bm{x}_{t-\frac{1}{2}}\rangle) \\
    & \quad + 2\ell\eta_x\langle  \nabla_x f(\bm{x}_{t },\bm{y}_{t })-\nabla_x f(\bm{x}_{t-1},\bm{y}_{t-1})  , \hat{\bm{x}}_{t-\frac{1}{2}}-\bm{x}_{t-\frac{1}{2}} \rangle   + \eta_x^2\ell G^2 
\end{align*}
According to smoothness of $f(\cdot,\bm{y})$, we have:
\begin{align*}
    \langle \hat{\bm{x}}_{t-\frac{1}{2}}-\bm{x}_{t-\frac{1}{2}} , \nabla_x f(\bm{x}_{t-\frac{1}{2}},\bm{y}_{t-\frac{1}{2}}) \rangle &\leq f(\hat{\bm{x}}_{t-\frac{1}{2}},\bm{y}_{t-\frac{1}{2}}) - f(\bm{x}_{t-\frac{1}{2}},\bm{y}_{t-\frac{1}{2}}) + \frac{\ell}{2}\|\hat{\bm{x}}_{t-\frac{1}{2}}-\bm{x}_{t-\frac{1}{2}}\|^2 \\
    & \leq \Phi(\bm{x}_{t-\frac{1}{2}}) - f(\bm{x}_{t-\frac{1}{2}},\bm{y}_{t-\frac{1}{2}}) - \frac{\ell}{2}\|\hat{\bm{x}}_{t-\frac{1}{2}}-\bm{x}_{t-\frac{1}{2}}\|^2.
\end{align*}
So we have
\begin{align*}
     \Phi_{1/2\ell}(\bm{x}_{t+\frac{1}{2}}) 
    &\leq \Phi_{1/2\ell}( {\bm{x}}_{t-\frac{1}{2}})+\ell \|\bm{x}_{t-\frac{1}{2}}-\hat{\bm{x}}_{t-\frac{1}{2}} \|^2 \\
    &\quad+ 2\eta_x \ell \left(\Phi(\bm{x}_{t-\frac{1}{2}}) - f(\bm{x}_{t-\frac{1}{2}},\bm{y}_{t-\frac{1}{2}}) - \frac{\ell}{2}\|\hat{\bm{x}}_{t-\frac{1}{2}}-\bm{x}_{t-\frac{1}{2}}\|^2\right)+3 \ell \eta_x^2 G^2 \\
    & + \eta_x \ell \left(\frac{1}{2\ell} \|\nabla_x f(\bm{x}_{t },\bm{y}_{t }) - \nabla_x f(\bm{x}_{t-1},\bm{y}_{t-1})\|^2 +\frac{\ell}{2}\| \bm{x}_{t-\frac{1}{2}}-\hat{\bm{x}}_{t-\frac{1}{2}} \|^2   \right )\\
    & \leq \Phi_{1/2\ell}({\bm{x}}_{t-\frac{1}{2}}) + 2\eta_x \ell \left(\Phi(\bm{x}_{t-\frac{1}{2}}) - f(\bm{x}_{t-\frac{1}{2}},\bm{y}_{t-\frac{1}{2}})\right) - \frac{\eta_x\ell^2}{2}\|\hat{\bm{x}}_{t-\frac{1}{2}}-\bm{x}_{t-\frac{1}{2}}\|^2+3 \ell \eta_x^2 G^2 \\
    & + \frac{\eta_x }{2} \| \nabla_x f(\bm{x}_{t },\bm{y}_{t }) - \nabla_x f(\bm{x}_{t-1},\bm{y}_{t-1})\|^2.  \\
\end{align*}

\end{proof}

\end{lemma}

\begin{lemma}\label{lemma: EG gap}
For EG (Algorithm~\ref{alg:eg}), under Theorem~\ref{thm:ncc_sogda}'s assumptions, the following statement holds for the generated sequence $\{\bm{x}_t\},\{\bm{y}_t\}, \{\bm{x}_{t+\frac{1}{2}}\},\{\bm{y}_{t+\frac{1}{2}}\}$ during algorithm proceeding and $\forall s \leq t$:
 \begin{align*}
\Phi(\bm{x}_{t+\frac{1}{2}}) - f(\bm{x}_{t+\frac{1}{2}},\bm{y}_{t+\frac{1}{2}})  
     &\leq 2(t-s+1)\eta_x G^2 \\
     &\quad+\frac{1}{2\eta_y} \left(     \|\bm{y}_{t} - \bm{y}^*(\bm{x}_{s})\|^2-\|\bm{y}_{t+1}-\bm{y}^*(\bm{x}_{s}) \|^2     + \frac{\eta_x^2G^2}{2} \right).
 \end{align*}
 \begin{proof}
 Observe that:
 \begin{align*}
     \Phi(\bm{x}_{t+\frac{1}{2}}) - f(\bm{x}_{t+\frac{1}{2}},\bm{y}_{t+\frac{1}{2}}) &\leq f(\bm{x}_{t+\frac{1}{2}},\bm{y}^*(\bm{x}_{t+\frac{1}{2}}))-f(\bm{x}_{s},\bm{y}^*(\bm{x}_{t+\frac{1}{2}}))+f(\bm{x}_{s},\bm{y}^*(\bm{x}_{s})) \\
     &\quad- f(\bm{x}_{t+\frac{1}{2}},\bm{y}^*(\bm{x}_{s})) + f(\bm{x}_{t+\frac{1}{2}},\bm{y}^*(\bm{x}_{s})) - f(\bm{x}_{t+\frac{1}{2}},\bm{y}_{t+\frac{1}{2}})\\
     &\leq 2(t-s+1)\eta_x G^2 -\langle \bm{y}_{t+\frac{1}{2}} - \bm{y}, \nabla_y f(\bm{x}_{t+\frac{1}{2}},\bm{y}_{t+\frac{1}{2}})\rangle
 \end{align*}
 Plugging in Lemma~\ref{lemma: EG dual} will conclude the proof:
  \begin{align*}
     \Phi(\bm{x}_{t+\frac{1}{2}}) - f(\bm{x}_{t+\frac{1}{2}},\bm{y}_{t+\frac{1}{2}})  
     &\leq 2(t-s+1)\eta_x G^2 \\
     &\quad+\frac{1}{2\eta_y} \left(     \|\bm{y}_{t} - \bm{y}^*(\bm{x}_{s})\|^2-\|\bm{y}_{t+1}-\bm{y}^*(\bm{x}_{s}) \|^2     + \frac{\eta_x^2G^2}{2} \right).
 \end{align*}

 \end{proof}

\end{lemma}
 
\begin{lemma}\label{lemma: EG avg gap}
For EG (Algorithm~\ref{alg:eg}), under Theorem~\ref{thm:ncc_sogda}'s assumptions, the following statement holds for the generated sequence $\{\bm{x}_t\},\{\bm{y}_t\}, \{\bm{x}_{t+\frac{1}{2}}\},\{\bm{y}_{t+\frac{1}{2}}\}$ during algorithm proceeding:
 \begin{align*}
    \frac{1}{T+1}\sum_{t=0}^{T} \Phi(\bm{x}_{t-\frac{1}{2}}) - f(\bm{x}_{t-\frac{1}{2}},\bm{y}_{t-\frac{1}{2}})  \leq \frac{1}{B} \left(2\eta_x B^2 G^2+ \frac{D^2}{2\eta_y} + \frac{B \eta_x^2 G^2}{2}\right)
 \end{align*}
                        
 \begin{proof}
 According to Lemma~\ref{lemma: EG gap}:
 \begin{align*}                                                 
     &\frac{1}{T+1}\sum_{t=0}^{T} \Phi(\bm{x}_{t-\frac{1}{2}}) - f(\bm{x}_{t-\frac{1}{2}},\bm{y}_{t-\frac{1}{2}}) \\
     &= \frac{1}{T+1}\sum_{j=0}^{S} \sum_{t=kB}^{(k+1)B-1} \Phi(\bm{x}_{t-\frac{1}{2}}) - f(\bm{x}_{t-\frac{1}{2}},\bm{y}_{t-\frac{1}{2}})\\
     & \leq \frac{1}{T+1}\sum_{j=0}^{S} \left[2B^2\eta_x G^2 +\frac{1}{2\eta_y} \left(     \|\bm{y}_{kB} - \bm{y}^*(\bm{x}_{s})\|^2-\|\bm{y}_{(k+1)B-1}-\bm{y}^*(\bm{x}_{s}) \|^2     + \frac{\eta_x^2G^2}{2} \right) \right] \\ 
     &\leq \frac{1}{B} \left[2\eta_x B^2 G^2+ \frac{D^2}{2\eta_y} + \frac{B \eta_x^2 G^2}{2}\right].
 \end{align*}
 \end{proof}
\end{lemma}

 \subsubsection{Proof of Theorem~\ref{thm:ncc_ogda} for EG} \label{app: eg}
In this section we are going to provide the proof for Theorem~\ref{thm:ncc_ogda}, EG part, the convergence rate of EG in deterministic setting. We first introduce the formal version of Theorem~\ref{thm:ncc_ogda}, EG part here:
\begin{theorem}[EG Deterministic, formal] \label{Theorem: EG formal}
 Under Assumption~\ref{asm:3}, if we choose $\eta_x = O\left(\min\left\{\frac{\epsilon}{\ell G},\frac{\epsilon^2}{\ell G^2},\frac{\epsilon^4}{D^2 G^2 \ell^3}\right\}\right)$, $\eta_y = \frac{1}{2\ell}$, then EG (Algorithm~\ref{alg:eg}) guarantees to find $\epsilon$-stationary point, i.e., $\frac{1}{T+1} \sum_{t=0}^{T}  \|\nabla \Phi_{1/2\ell}(\bm{x}_{t})\|^2 \leq \epsilon^2$, with the gradient complexity bounded by:
 \begin{align*}
     O\left( \frac{\ell G^2 \hat{\Delta}_{\Phi}}{\epsilon^4}\max \left\{ 1,\frac{D^2  \ell^2}{\epsilon^2}\right\}   \right).
 \end{align*}  
 
 \begin{proof}
 According to Lemma~\ref{lemma: EG descent}:
 \begin{align*} 
    \frac{1}{T+1}\sum_{t=0}^T\|\nabla \Phi_{1/2\ell}(\bm{x}_{t-\frac{1}{2}})\|^2   
    &\leq \frac{\Phi_{1/2\ell}(\bm{x}_{-\frac{1}{2}})  }{\eta_x(T+1)}+\frac{1}{T+1}\sum_{t=0}^T8\ell \left(\Phi({x}_{t-\frac{1}{2}}) - f(\bm{x}_{t-\frac{1}{2}},\bm{y}_{t-\frac{1}{2}})  \right) + 12\eta_x\ell G^2\\
    & + 8\frac{1}{T+1}\sum_{t=0}^T \|\nabla_x f(\bm{x}_{t},\bm{y}_{t}) - \nabla_x f(\bm{x}_{t-1},\bm{y}_{t-1})\|^2. 
\end{align*}

For $ \|\nabla_x f(\bm{x}_{t},\bm{y}_{t}) - \nabla_x f(\bm{x}_{t-1},\bm{y}_{t-1})\|^2$, notice that:
\begin{align*}
    \|\nabla_x f(\bm{x}_{t},\bm{y}_{t}) - \nabla_x f(\bm{x}_{t-1},\bm{y}_{t-1})\|^2 &\leq \ell^2 \|\bm{x}_{t}-\bm{x}_{t-1}\|^2 + \ell^2\|\bm{y}_{t}-\bm{y}_{t-1}\|^2\\
    & \leq \eta_x^2\ell^2G^2 + \eta_y^2\ell^2\|\nabla_y f(\bm{x}_{t-\frac{1}{2}},\bm{y}_{t-\frac{1}{2}}) \|^2\\
    & \leq \eta_x^2\ell^2G^2 + 2\eta_y^2\ell^3   \left( \Phi(\bm{x}_{t-\frac{1}{2}})-f(\bm{x}_{t-\frac{1}{2}},\bm{y}_{t-\frac{1}{2}}) \right)\\
\end{align*}
So we have:
 \begin{align*}
       &\frac{1}{T+1}\sum_{t=0}^T\|\nabla \Phi_{1/2\ell}(\bm{x}_{t-\frac{1}{2}})\|^2   
    \leq \frac{\Phi_{1/2\ell}(\bm{x}_{-\frac{1}{2}})  }{\eta_x(T+1)}\\
    &\qquad+\frac{1}{T+1}\sum_{t=0}^T(8\ell+ 2\eta_y^2\ell^3) \left(\Phi({x}_{t-\frac{1}{2}}) - f(\bm{x}_{t-\frac{1}{2}},\bm{y}_{t-\frac{1}{2}})  \right) + 12\eta_x G^2 + 8   \eta_x^2\ell^2G^2  
 \end{align*}
 Now we plug in Lemma~\ref{lemma: EG avg gap}:
  \begin{align*}
  &\frac{1}{T+1}\sum_{t=0}^T\|\nabla \Phi_{1/2\ell}(\bm{x}_{t-\frac{1}{2}})\|^2   
    \leq \frac{\Phi_{1/2\ell}(\bm{x}_{-\frac{1}{2}})-\Phi_{1/2\ell}(\bm{x}_{T-\frac{1}{2}})  }{\eta_x(T+1)} \\
    &\qquad+  {(8\ell+ 2\eta_y^2\ell^3)}  \left(2\eta_x B  G^2+ \frac{D^2}{2\eta_yB} + \frac{  \eta_x^2 G^2}{2}\right) + 12\ell \eta_x G^2 + 8   \eta_x^2\ell^2G^2  
 \end{align*}
 Choose $B = O\left(\frac{D}{G\sqrt{\eta_x \eta_y }}\right),   \eta_x = O\left(\min\left\{\frac{\epsilon}{\ell G},\frac{\epsilon^2}{\ell G^2},\frac{\epsilon^4}{D^2 G^2 \ell^3}\right\}\right)$, $\eta_y = \frac{1}{2\ell}$, and then we guarantee that $\frac{1}{T+1} \sum_{t=0}^{T}  \|\nabla \Phi_{1/2\ell}(\bm{x}_{t-\frac{1}{2}})\|^2 \leq \epsilon^2$ with the gradient complexity is bounded by:
 \begin{align*}
     O\left( \frac{\ell G^2 \hat{\Delta}_{\Phi}}{\epsilon^4}\max \left\{ 1,\frac{D^2  \ell^2}{\epsilon^2}\right\}   \right).
 \end{align*} 
 \end{proof}
 \end{theorem}

\textbf{Stochastic setting.} 

In this part, we are going to present proof of EG in stochastic setting. First let us introduce some useful lemmas.
\subsubsection{Useful Lemmas}
\begin{lemma} \label{lemma: SEG dual}
For Stochastic EG (Algorithm~\ref{alg:eg}), under Theorem~\ref{thm:ncc_sogda}'s assumptions, the following statement holds for the generated sequence $\{\bm{y}_t\},\{\bm{y}_{t+\frac{1}{2}}\}$ during algorithm proceeding and any $\bm{y} \in \mathcal{Y}$:
\begin{align*}
    \|\bm{y}_{t+1}-\bm{y} \|^2  &\leq  \|\bm{y}_{t} - \bm{y}\|^2  + 2\eta_y \langle \bm{y}_{t+\frac{1}{2}} - \bm{y}, \nabla_y f(\bm{x}_{t+\frac{1}{2}},\bm{y}_{t+\frac{1}{2}})\rangle -\left(\frac{1}{2} - \frac{3\eta_x^2L^2}{2}\right)\|\bm{y}_t - \bm{y}_{t+\frac{1}{2}}\|^2 \\
    &\quad+  \frac{1}{2}(3\eta_x^2\eta_y^2\ell^2 (G^2+\sigma^2) + 6\eta_y^2\sigma^2).
\end{align*} 
\begin{proof}
According to Proposition~\ref{prop: 3 point theorem}, we set $\bm{r} = \bm{y}_t$, $\bm{q} = \bm{y}_{t+1}$, $\bm{p} = \bm{y}_{t+\frac{1}{2}}$ and $\bm{v} = -\eta_y\nabla_y f(\bm{x}_{t+\frac{1}{2}},\bm{y}_{t+\frac{1}{2}};\xi)$, $\bm{u} = -\eta_y\nabla_y f(\bm{x}_{t},\bm{y}_{t};\xi)$. We can verify that:
\begin{align*}
    &\|\bm{u} - \bm{v} \|^2 = \eta_y^2\|\nabla_y f(\bm{x}_{t+\frac{1}{2}},\bm{y}_{t+\frac{1}{2}};\xi) - \nabla_y f(\bm{x}_{t},\bm{y}_{t};\xi)\|^2 \\
    &\leq 3\eta_y^2\|\nabla_y f(\bm{x}_{t+\frac{1}{2}},\bm{y}_{t+\frac{1}{2}}) - \nabla_y f(\bm{x}_{t},\bm{y}_{t})\|^2 + 3\eta_y^2\|\nabla_y f(\bm{x}_{t+\frac{1}{2}},\bm{y}_{t+\frac{1}{2}};\xi) - \nabla_y f(\bm{x}_{t+\frac{1}{2}},\bm{y}_{t+\frac{1}{2}})\|^2\\
    & \quad + 3\eta_y^2\|\nabla_y f(\bm{x}_{t},\bm{y}_{t};\xi) - \nabla_y f(\bm{x}_{t},\bm{y}_{t})\|^2\\
    & \leq 3(\ell^2\|\bm{y}_{t+\frac{1}{2}}-\bm{y}_{t}\|^2 + \ell^2\|\bm{x}_{t+\frac{1}{2}}-\bm{x}_{t}\|^2) + 3\eta_y^2Var(\nabla_y f(\bm{x}_{t},\bm{y}_{t};\xi) ) \\
    &\quad+ 3\eta_y^2Var(\nabla_y f(\bm{x}_{t+\frac{1}{2}},\bm{y}_{t+\frac{1}{2}};\xi))\\
    & \leq 3\eta_y^2(\ell^2\|\bm{y}_{t+\frac{1}{2}}-\bm{y}_{t}\|^2 + \eta_x^2\ell^2 (G^2+\sigma^2)) + 3\eta_y^2Var(\nabla_y f(\bm{x}_{t},\bm{y}_{t};\xi) ) \\
    &\quad+ 3\eta_y^2Var(\nabla_y f(\bm{x}_{t+\frac{1}{2}},\bm{y}_{t+\frac{1}{2}};\xi))
\end{align*}
so if we set $C_1^2 = 3\eta_y^2\ell^2$ and $C_2^2 =3\eta_x^2\eta_y^2\ell^2 (G^2+\sigma^2) + 3\eta_y^2Var(\nabla_y f(\bm{x}_{t},\bm{y}_{t};\xi) ) + 3\eta_y^2Var(\nabla_y f(\bm{x}_{t+\frac{1}{2}},\bm{y}_{t+\frac{1}{2}};\xi))$, we have the following inequality holding for any $\bm{y} \in \mathcal{Y}$:
\begin{align*}
    \langle -\eta_y\nabla_y f(\bm{x}_{t+\frac{1}{2}},\bm{y}_{t+\frac{1}{2}};\xi), \bm{y}_{t+\frac{1}{2}}-\bm{y} \rangle &\leq \|\bm{y}_t - \bm{y}\|^2 - \|\bm{y}_{t+1} - \bm{y}\|^2 \\
    &\quad- \left(\frac{1}{2} - \frac{C_1^2}{2}\right)\|\bm{y}_t- \bm{y}_{t+\frac{1}{2}}\|^2 + \frac{C_2^2}{2}.
\end{align*}
Taking expectation on both sides yields:
\begin{align*}
    &\langle -\eta_y\nabla_y f(\bm{x}_{t+\frac{1}{2}},\bm{y}_{t+\frac{1}{2}}), \bm{y}_{t+\frac{1}{2}}-\bm{y} \rangle \leq \mathbb{E}\|\bm{y}_t - \bm{y}\|^2 - \mathbb{E}\|\bm{y}_{t+1} - \bm{y}\|^2 \\
    &\qquad- \left(\frac{1}{2} - \frac{3\eta_y^2\ell^2}{2}\right)\mathbb{E}\|\bm{y}_t- \bm{y}_{t+\frac{1}{2}}\|^2  + \frac{1}{2}(3\eta_x^2\eta_y^2\ell^2 (G^2+\sigma^2) + 6\eta_y^2\sigma^2).
\end{align*}
\end{proof}
\end{lemma}

\begin{lemma} 
\label{lemma: seg descent}
For Stochastic EG (Algorithm~\ref{alg:eg}), under Theorem~\ref{thm:ncc_sogda}'s assumptions, the following statement holds for the generated sequence $\{\bm{x}_t\},\{\bm{y}_t\}, \{\bm{x}_{t+\frac{1}{2}}\},\{\bm{y}_{t+\frac{1}{2}}\}$ during algorithm proceeding:
\begin{align*} 
   \mathbb{E}[\Phi_{1/2\ell}(\bm{x}_{t+\frac{1}{2}})]  
    &\leq \mathbb{E}[\Phi_{1/2\ell}(\bm{x}_{t-\frac{1}{2}})] +2\eta\ell  \mathbb{E}[\Phi({x}_{t-\frac{1}{2}}) - f(\bm{x}_{t-\frac{1}{2}},\bm{y}_{t-\frac{1}{2}})]  - \frac{\eta_x}{8 }\mathbb{E} \|\nabla \Phi_{1/2\ell}(\bm{x}_{t-\frac{1}{2}})\|^2 \\
    &\quad+ 3\eta_x^2\ell (G^2+\sigma^2)+ 2\eta_x  \mathbb{E}\|\nabla_x f(\bm{x}_{t},\bm{y}_{t}) - \nabla_x f(\bm{x}_{t-1},\bm{y}_{t-1})\|^2  .
\end{align*}
   \begin{proof}
 Let $\hat{\bm{x}}_{t-\frac{1}{2}} = \arg\min_{\bm{x}\in \mathbb{R}^d} \Phi(\bm{x})+\ell \|\bm{x}-\bm{x}_{t-\frac{1}{2}}\|^2$. Notice that:
\begin{align*}
    &\mathbb{E}[\Phi_{1/2\ell}(\bm{x}_{t+\frac{1}{2}})]  
    \leq \mathbb{E}[\Phi (\hat{\bm{x}}_{t-\frac{1}{2}})]+\ell\mathbb{E}\|\hat{\bm{x}}_{t-\frac{1}{2}}-\bm{x}_{t+\frac{1}{2}}\|^2\\
    &\leq \Phi_{1/2\ell}(\hat{\bm{x}}_{t-\frac{1}{2}})+3\eta_x^2\ell(\sigma^2 +   G^2) +\ell \mathbb{E}\|\hat{\bm{x}}_{t-\frac{1}{2}}-\bm{x}_{t+\frac{1}{2}} \|^2 \\
    &\quad+ 2\eta_x \ell \mathbb{E}\langle  \nabla_x f(\bm{x}_{t-\frac{1}{2}},\bm{y}_{t-\frac{1}{2}})+ \nabla_x f(\bm{x}_{t },\bm{y}_{t })-\nabla_x f(\bm{x}_{t-1},\bm{y}_{t-1})  , \hat{\bm{x}}_{t-\frac{1}{2}}-\bm{x}_{t-\frac{1}{2}} \rangle  )\\\
    &= \mathbb{E}[\Phi (\hat{\bm{x}}_{t-\frac{1}{2}})]+\ell (\mathbb{E}\|\hat{\bm{x}}_{t-\frac{1}{2}}-\bm{x}_{t+\frac{1}{2}} \|^2 + 2\eta_x\mathbb{E}\langle  \nabla_x f(\bm{x}_{t-\frac{1}{2}},\bm{y}_{t-\frac{1}{2}})  , \hat{\bm{x}}_{t-\frac{1}{2}}-\bm{x}_{t-\frac{1}{2}}\rangle) \\
    & \quad + 2\ell\eta_x\mathbb{E}\langle  \nabla_x f(\bm{x}_{t },\bm{y}_{t })-\nabla_x f(\bm{x}_{t-1},\bm{y}_{t-1})  , \hat{\bm{x}}_{t-\frac{1}{2}}-\bm{x}_{t-\frac{1}{2}} \rangle+3\eta_x^2\ell(\sigma^2 +   G^2)
\end{align*}
According to smoothness of $f(\cdot,\bm{y})$, we have:
\begin{align*}
    \langle \hat{\bm{x}}_{t-\frac{1}{2}}-\bm{x}_{t-\frac{1}{2}} , \nabla_x f(\bm{x}_{t-\frac{1}{2}},\bm{y}_{t-\frac{1}{2}}) \rangle &\leq f(\hat{\bm{x}}_{t-\frac{1}{2}},\bm{y}_{t-\frac{1}{2}}) - f(\bm{x}_{t-\frac{1}{2}},\bm{y}_{t-\frac{1}{2}}) + \frac{\ell}{2}\|\hat{\bm{x}}_{t-\frac{1}{2}}-\bm{x}_{t-\frac{1}{2}}\|^2 \\
    & \leq \Phi(\bm{x}_{t-\frac{1}{2}}) - f(\bm{x}_{t-\frac{1}{2}},\bm{y}_{t-\frac{1}{2}}) - \frac{\ell}{2}\|\hat{\bm{x}}_{t-\frac{1}{2}}-\bm{x}_{t-\frac{1}{2}}\|^2.
\end{align*}
So we have:
\begin{align*}
    &\mathbb{E}[\Phi_{1/2\ell}(\bm{x}_{t+\frac{1}{2}}) ]
    \leq \mathbb{E}[\Phi ( \hat{\bm{x}}_{t-\frac{1}{2}})]+\ell \mathbb{E}\|\bm{x}_{t-\frac{1}{2}}-\hat{\bm{x}}_{t-\frac{1}{2}} \|^2 \\
    &\quad + 2\eta_x \ell \mathbb{E}\left(\Phi(\bm{x}_{t-\frac{1}{2}}) - f(\bm{x}_{t-\frac{1}{2}},\bm{y}_{t-\frac{1}{2}}) - \frac{\ell}{2}\|\hat{\bm{x}}_{t-\frac{1}{2}}-\bm{x}_{t-\frac{1}{2}}\|^2\right)+3\eta_x^2\ell(G^2+\sigma^2 ) \\
    & \quad+ \eta_x \ell \left(\frac{1}{2\ell} \mathbb{E}\|\nabla_x f(\bm{x}_{t },\bm{y}_{t }) - \nabla_x f(\bm{x}_{t-1},\bm{y}_{t-1})\|^2 +\frac{\ell}{2}\mathbb{E}\| \bm{x}_{t-\frac{1}{2}}-\hat{\bm{x}}_{t-\frac{1}{2}} \|^2   \right )\\
    & \leq \mathbb{E}[\Phi_{1/2\ell}( {\bm{x}}_{t-\frac{1}{2}})] + 2\eta_x \ell \left(\Phi(\bm{x}_{t-\frac{1}{2}}) - f(\bm{x}_{t-\frac{1}{2}},\bm{y}_{t-\frac{1}{2}})\right) - \frac{\eta_x\ell^2}{2}\mathbb{E}\|\hat{\bm{x}}_{t-\frac{1}{2}}-\bm{x}_{t-\frac{1}{2}}\|^2\\
    &\quad+3\eta_x^2\ell(G^2+\sigma^2 ) + \frac{\eta_x }{2} \mathbb{E}\| \nabla_x f(\bm{x}_{t },\bm{y}_{t }) - \nabla_x f(\bm{x}_{t-1},\bm{y}_{t-1})\|^2.  \\
\end{align*}
Using the fact that $\|  \bm{x}_{t-\frac{1}{2}}-\hat{\bm{x}}_{t-\frac{1}{2}} \| = \frac{1}{2\ell}\|\nabla \Phi_{1/2\ell}( {\bm{x}}_{t-\frac{1}{2}})\|$ will conclude the proof.

\end{proof}

\end{lemma}

\begin{lemma}\label{lemma: SEG gap} 
For Stochastic EG (Algorithm~\ref{alg:eg}), under Theorem~\ref{thm:ncc_sogda}'s assumptions, the following statement holds for the generated sequence $\{\bm{x}_t\},\{\bm{y}_t\}, \{\bm{x}_{t+\frac{1}{2}}\},\{\bm{y}_{t+\frac{1}{2}}\}$ during algorithm proceeding and $\forall s \leq t$:
 \begin{align*}
&\Phi(\bm{x}_{t+\frac{1}{2}}) - f(\bm{x}_{t+\frac{1}{2}},\bm{y}_{t+\frac{1}{2}})  
     \leq 2(t-s+1)\eta_x G^2 \\
     &\qquad+\frac{1}{2\eta_y} \left(     \|\bm{y}_{t} - \bm{y}^*(\bm{x}_{s})\|^2-\|\bm{y}_{t+1}-\bm{y}^*(\bm{x}_{s}) \|^2    +  \frac{1}{2}(3\eta_x^2\eta_y^2\ell^2 (G^2+\sigma^2) + 6\eta_y^2\sigma^2)\right).
 \end{align*}
 \begin{proof}
 According to Lemma~\ref{lemma: SEG gap}:
 \begin{align*}
     \Phi(\bm{x}_{t+\frac{1}{2}}) - f(\bm{x}_{t+\frac{1}{2}},\bm{y}_{t+\frac{1}{2}}) &\leq f(\bm{x}_{t+\frac{1}{2}},\bm{y}^*(\bm{x}_{t+\frac{1}{2}}))-f(\bm{x}_{s},\bm{y}^*(\bm{x}_{t+\frac{1}{2}}))+f(\bm{x}_{s},\bm{y}^*(\bm{x}_{s})) \\
     &\quad- f(\bm{x}_{t+\frac{1}{2}},\bm{y}^*(\bm{x}_{s})) + f(\bm{x}_{t+\frac{1}{2}},\bm{y}^*(\bm{x}_{s})) - f(\bm{x}_{t+\frac{1}{2}},\bm{y}_{t+\frac{1}{2}})\\
     &\leq 2(t-s+1)\eta_x G^2 -\langle \bm{y}_{t+\frac{1}{2}} - \bm{y}, \nabla_y f(\bm{x}_{t+\frac{1}{2}},\bm{y}_{t+\frac{1}{2}})\rangle
 \end{align*}
 Plugging in Lemma~\ref{lemma: SEG dual} will conclude the proof:
  \begin{align*}
     &\Phi(\bm{x}_{t+\frac{1}{2}}) - f(\bm{x}_{t+\frac{1}{2}},\bm{y}_{t+\frac{1}{2}})  
     \leq 2(t-s+1)\eta_x G^2 \\
     &\quad+\frac{1}{2\eta_y} \left(     \|\bm{y}_{t} - \bm{y}^*(\bm{x}_{s})\|^2-\|\bm{y}_{t+1}-\bm{y}^*(\bm{x}_{s}) \|^2 +  \frac{1}{2}(3\eta_x^2\eta_y^2\ell^2 (G^2+\sigma^2) + 6\eta_y^2\sigma^2) \right).
 \end{align*}

 \end{proof}

\end{lemma}
 
\begin{lemma}\label{lemma: SEG avg gap}
For Stochastic EG (Algorithm~\ref{alg:eg}), under Theorem~\ref{thm:ncc_sogda}'s assumptions, the following statement holds for the generated sequence $\{\bm{x}_t\},\{\bm{y}_t\}, \{\bm{x}_{t+\frac{1}{2}}\},\{\bm{y}_{t+\frac{1}{2}}\}$ during algorithm proceeding:
 \begin{align*}
    \frac{1}{T+1}\sum_{t=0}^{T} \Phi(\bm{x}_{t-\frac{1}{2}}) - f(\bm{x}_{t-\frac{1}{2}},\bm{y}_{t-\frac{1}{2}})  \leq \frac{1}{B} \left(2\eta_x B^2 G^2+ \frac{D^2}{2\eta_y} + \frac{B \eta_x^2 G^2}{2}\right).
 \end{align*}
                        
 \begin{proof}
 Summing over $t = 0$ to $T$ on both side of Lemma~\ref{lemma: SEG gap}  yields:
 \begin{align*}                                                 
     &\frac{1}{T+1}\sum_{t=0}^{T} \Phi(\bm{x}_{t-\frac{1}{2}}) - f(\bm{x}_{t-\frac{1}{2}},\bm{y}_{t-\frac{1}{2}})\\
     &= \frac{1}{T+1}\sum_{j=0}^{S} \sum_{t=kB}^{(k+1)B-1} \Phi(\bm{x}_{t-\frac{1}{2}}) - f(\bm{x}_{t-\frac{1}{2}},\bm{y}_{t-\frac{1}{2}})\\
     & \leq \frac{1}{T+1}\sum_{j=0}^{S} \left[2B^2\eta_x G^2 +\frac{1}{2\eta_y} \left(     \|\bm{y}_{kB} - \bm{y}^*(\bm{x}_{s})\|^2-\|\bm{y}_{(k+1)B-1}-\bm{y}^*(\bm{x}_{s}) \|^2 \right. \right.  \\
     &\left. \left. \quad+  \frac{1}{2}(3\eta_x^2\eta_y^2\ell^2 (G^2+\sigma^2) + 6\eta_y^2\sigma^2) \right) \right] \\ 
     &\leq \frac{1}{B} \left(2\eta_x B^2 G^2+ \frac{D^2}{2\eta_y} +  \frac{1}{2}(3\eta_x^2\eta_y^2\ell^2 (G^2+\sigma^2) + 6\eta_y^2\sigma^2)\right),
 \end{align*}
 which concludes the proof.
 \end{proof}
\end{lemma}

 \subsubsection{Proof of Theorem~\ref{thm:ncc_sogda} for EG} \label{app: seg}
In this section we  provide the proof for Theorem~\ref{thm:ncc_sogda} on the convergence rate of EG in stochastic setting. We first introduce the formal version of theorem  here:
\begin{theorem}[EG Stochastic, formal] \label{Theorem: SEG formal}
Under Assumption~\ref{asm:2}, and~\ref{asm:3}, if we choose $\eta_x = O(\min\{\frac{\epsilon^2}{\ell (G^2+\sigma^2)},\frac{\epsilon^4}{D^2\ell^3  G\sqrt{G^2+\sigma^2} },\frac{\epsilon^6}{D^2\ell^3\sigma^2 G\sqrt{G^2+\sigma^2} }\})$, $\eta_y = O(\min\{\frac{1}{2\ell},\frac{\epsilon^2}{\ell\sigma^2}\})$, then Stochastic EG (Algorithm~\ref{alg:eg}) guarantees to find $\epsilon$-stationary point, i.e., $\frac{1}{T+1} \sum_{t=0}^{T} \mathbb{E} \|\nabla \Phi_{1/2\ell}(\bm{x}_{t})\|^2 \leq \epsilon^2$, with the gradient complexity bounded by:
 \begin{align*}
     O\left(\frac{D^2\ell^3  G\sqrt{G^2+\sigma^2}}{\epsilon^6} \max\left\{ 1, \frac{  \sigma^2   }{\epsilon^2} \right\}\right).
 \end{align*} 
 
 \begin{proof}
 According to Lemma~\ref{lemma: seg descent}:

\begin{align*} 
 &\frac{1}{T+1}\sum_{t=0}^{T } \mathbb{E} \|\nabla \Phi_{1/2\ell}(\bm{x}_{t-\frac{1}{2}})\|^2\leq    
      \frac{\mathbb{E}[\Phi_{1/2\ell}(\bm{x}_{-\frac{1}{2}}) - \Phi_{1/2\ell}(\bm{x}_{T+\frac{1}{2}})] }{T} \\
    &\qquad+16 \ell \frac{1}{T+1}\sum_{t=0}^{T } \mathbb{E}[\Phi({x}_{t-\frac{1}{2}}) - f(\bm{x}_{t-\frac{1}{2}},\bm{y}_{t-\frac{1}{2}})]    + 24\eta_x \ell (G^2+\sigma^2)\\
    &\qquad +16 \frac{1}{T+1}\sum_{t=0}^{T } \mathbb{E}\|\nabla_x f(\bm{x}_{t},\bm{y}_{t}) - \nabla_x f(\bm{x}_{t-1},\bm{y}_{t-1})\|^2  .
\end{align*}
Observe that:
\begin{align*}
    \mathbb{E}\|\nabla_x f(\bm{x}_{t},\bm{y}_{t}) - \nabla_x f(\bm{x}_{t-1},\bm{y}_{t-1})\|^2 &\leq \ell^2 \mathbb{E}\| (\bm{x}_{t},\bm{y}_{t}) -  (\bm{x}_{t-1},\bm{y}_{t-1})\|^2 \\
    & = \ell^2 \mathbb{E}\| \bm{x}_{t}  -  \bm{x}_{t-1} \|^2 + \ell^2 \mathbb{E}\| \bm{y}_{t}  -  \bm{y}_{t-1} \|^2\\
    & \leq \ell^2 \eta_x^2 (G^2+\sigma^2) + \ell^2 \eta_y^2 \mathbb{E} \left\|\nabla_y f(\bm{x}_{t-\frac{1}{2}},\bm{y}_{t-\frac{1}{2}})\right\|^2\\
    & \leq \ell^2 \eta_x^2 (G^2+\sigma^2) + \ell^2 \eta_y^2 \mathbb{E} \left[ \Phi(\bm{x}_{t-\frac{1}{2}}) - f(\bm{x}_{t-\frac{1}{2}},\bm{y}_{t-\frac{1}{2}}) \right].
\end{align*}
So we have:
\begin{align*} 
    &\frac{1}{T+1}\sum_{t=0}^{T} \mathbb{E} \|\nabla \Phi_{1/2\ell}(\bm{x}_{t-\frac{1}{2}})\|^2 \leq    
      \frac{\mathbb{E}[\Phi_{1/2\ell}(\bm{x}_{-\frac{1}{2}})-\Phi_{1/2\ell}(\bm{x}_{T+\frac{1}{2}})] }{T+1} \\
      &\qquad+16 \ell \frac{1}{T+1}\sum_{t=0}^{T} \mathbb{E}[\Phi({x}_{t-\frac{1}{2}}) - f(\bm{x}_{t-\frac{1}{2}},\bm{y}_{t-\frac{1}{2}})] + 24\eta_x\ell (G^2+\sigma^2)\\
    & \qquad + 16\ell^2 \eta_y^2 \frac{1}{T+1}\sum_{t=0}^{T} \mathbb{E} \left[ \Phi(\bm{x}_{t-\frac{1}{2}}) - f(\bm{x}_{t-\frac{1}{2}},\bm{y}_{t-\frac{1}{2}}) \right] +16\ell^2 \eta_x^2 (G^2+\sigma^2)   \\
    &\quad \leq    
      \frac{\mathbb{E}[\Phi_{1/2\ell}(\bm{x}_{-\frac{1}{2}})-\Phi_{1/2\ell}(\bm{x}_{T+\frac{1}{2}})] }{T+1} +16(\ell+\ell^2 \eta_y^2) \frac{1}{T+1}\sum_{t=0}^{T} \mathbb{E}[\Phi({x}_{t-\frac{1}{2}}) - f(\bm{x}_{t-\frac{1}{2}},\bm{y}_{t-\frac{1}{2}})]   \\
      & \qquad +16\ell^2 \eta_x^2 (G^2+\sigma^2)+ 24\eta_x \ell (G^2+\sigma^2)\\ 
\end{align*}
Plugging in Lemma~\ref{lemma: SEG avg gap} yields:
\begin{align*} 
    \frac{1}{T+1}\sum_{t=0}^{T} \mathbb{E} \|\nabla \Phi_{1/2\ell}(\bm{x}_{t-\frac{1}{2}})\|^2 
    &\leq   \frac{\mathbb{E}[\Phi_{1/2\ell}(\bm{x}_{-\frac{1}{2}})-\Phi_{1/2\ell}(\bm{x}_{T+\frac{1}{2}})] }{T+1} \\
    &\quad+16(\ell+\ell^2 \eta_y^2)\frac{1}{B} \left(2\eta_x B^2 G^2+ \frac{D^2}{2\eta_y} + \frac{B \eta_x^2 G^2}{2}\right)\\
      & \quad +16\ell^2 \eta_x^2 (G^2+\sigma^2)+ 24\eta_x \ell (G^2+\sigma^2). 
\end{align*}

Choosing $B = O(\frac{D}{\sqrt{\eta_x \eta_y G \sqrt{G^2+\sigma^2}}})$, $\eta_x = O(\min\{\frac{\epsilon^2}{\ell (G^2+\sigma^2)},\frac{\epsilon^4}{D^2\ell^3  G\sqrt{G^2+\sigma^2} },\frac{\epsilon^6}{D^2\ell^3\sigma^2 G\sqrt{G^2+\sigma^2} }\})$, $\eta_y = O(\min\{\frac{1}{\ell},\frac{\epsilon^2}{\ell\sigma^2}\})$,   guarantees that $\frac{1}{T+1} \sum_{t=0}^{T} \mathbb{E} \|\nabla \Phi_{1/2\ell}(\bm{x}_{t})\|^2 \leq \epsilon^2$ holds with the gradient complexity is bounded by:
 \begin{align*}
     O\left(\frac{D^2\ell^3  G\sqrt{G^2+\sigma^2} \hat{\Delta}_{\Phi}}{\epsilon^6} \max\left\{ 1, \frac{  \sigma^2   }{\epsilon^2} \right\}\right).
 \end{align*} 
 which completes the proof.
 \end{proof}
 \end{theorem}

\subsection{Tightness Analysis} \label{app:NCC_Tightness}


In this section, we provide our tightness analysis showing our obtained upper bound is tight given our choice of learning rates. In subsection~\ref{app:NCC_Tightness_GDA}, we introduce our hard example, and show the lower bound on convergence of this example, and then in subsection~\ref{app:NCC_Tightness_EG/OGDA}, we extend the tightness result to EG/OGDA using the same hard example.

\subsubsection{GDA}\label{app:NCC_Tightness_GDA}


\begin{proof}[Proof of Theorem~\ref{thm:NCC_Tightness_GDA}]
Let $L\ge0$ be some constants to be chosen later. Inspired by~\cite{drori2020complexity}, we consider the following function $f: \mathbb{R}\times [-D,D] \to \mathbb{R}$:
\begin{align*}
    f(x,y) = h(x)y
\end{align*}
where 
\begin{align*}
    h(x) = \begin{cases} \frac{L}{2}x^2 & |x| \leq 1 \\ 
    L - \frac{L}{2} (|x| - 2)^2 & 1 \leq |x| \leq  2 \\
    L & |x|\geq 2.
    \end{cases}
\end{align*}
It is easy to verify that $f$ is nonconvex, $2LD$ smooth, and $LD$-Lipschitz. We choose $L=\frac{1}{D}\min\{\ell/2,G\}$ to guarantee that $f$ is $\ell$ smooth and $G$-Lipschitz with respect to $x$. The primal function is $\Phi(x) =D h(x)$ attained when $y=D$. After standard calculations, we know that when $|x|\le 1$, the Moreau envelope $\Phi_{1/2\ell}(x)$ satisfies
\begin{align*}
    \Phi_{1/2\ell}(x) = \frac{LD\ell}{LD+2\ell}x^2,\quad |x|\le 1.
\end{align*} 
By definition, we also know $\Phi_{1/2\ell}(x)\ge 0$ for any $x\in\mathbb{R}$.

We first claim that if we choose $|x_0|\le 1$, $y_0 \ge 0$, we have for any $t\ge 0$, $|x_t| \leq 1$ and $y_t \geq 0$.
We verify this claim by induction. First note that when $t = 0$, the claim holds for sure. Let us assume it holds for $t = k$. Then for $t = k+1$,
\begin{align*}
    x_{k+1} = x_k - \eta_x L x_k y_k = (1-\eta_x L y_k)x_k.
\end{align*}
Since $0\le y_k\le D$, we have $0\le 1-\eta_x L y_k\le 1$. Therefore $|x_{k+1}| \leq 1$. For $y_{k+1}$, we have
\begin{align*}
    y_{k+1} = \mathcal{P}_{[-D,D]} (y_k + \eta_y h(x_k)).
\end{align*}
Since $h(x_k) \geq 0$, we know that $y_{k+1}\geq 0$, which verifies the claim.

We can also bound
\begin{align*}
    |x_T| = \left|\prod_{t=0}^{T-1} (1-\eta_x L y_t )x_0 \right| \ge (1-\eta_x LD)^T |x_0|.
\end{align*}
Since $\nabla \Phi_{1/2\ell}(x)=\frac{2LD\ell}{LD+2\ell}x$, choosing $x_0=\frac{LD+2\ell}{LD\ell}\epsilon$, we have $  \epsilon \geq |\nabla \Phi_{1/2\ell}(x_{T})|\ge 2\epsilon(1-\eta_x LD)^{T}$. Also noting $\hat{\Delta}_\Phi = \frac{LD+2\ell}{LD\ell}\epsilon^2$, we have
\begin{align*}
    T = \Omega\left(\frac{1}{\eta_x LD}\right)&=\Omega\left(\frac{\hat{\Delta}_\Phi }{\eta_x LD\epsilon^2}\cdot\frac{LD\ell}{LD+2\ell}\right)\\
    &=\Omega\left(\frac{\ell^3 G^2 D^2\hat{\Delta}_\Phi }{\epsilon^6}\right).
\end{align*}\vspace{-0.25cm}
\end{proof}

\subsubsection{EG/OGDA}\label{app:NCC_Tightness_EG/OGDA}


\begin{proof}[Proof of Theorem~\ref{thm:NCC_Tightness_EG/OGDA} for OGDA]
We use the same hard example $f(x,y)=h(x)y$ as in proof of Theorem~\ref{thm:NCC_Tightness_GDA}. Similarly, we first claim that if we choose $0\le x_0\le  1 $ and $y_0 = D$,  the following statements hold for any $t\ge 0$:
\begin{align*}
    (a) \ 0\le x_t  \le 1, \text{and} \  x_{t}\geq  x_{t-1}/\sqrt{2}, (b) \ y_t = D,
\end{align*} 
where we define $x_{-1}=x_0$ and $y_{-1}=y_0$.

Now we prove the above claim by induction. First, when $t = 0$, the claim holds for sure. Then, let us assume it holds for $t \leq k$. Then for $t = k+1$, we have
\begin{align*}
    x_{k+1} &= x_k - 2\eta_x LD x_k y_k + \eta_x LD x_{k-1}y_{k-1}\\ 
    & = (1-2\eta_xL D)x_k + \eta_x L D x_{k-1}.
\end{align*}
Since $0\le x_k, x_{k-1} \le 1$ and $0\le \eta_x L D\le 0.1$, we have 
\begin{align*}
    (1-2\eta_x L D)x_k\le x_{k+1} \le (1-\eta_x L D)x_k + \eta_x L D x_{k-1},
\end{align*}
which implies
$0\le  x_k/\sqrt{2}\le 0.8x_k\le x_{k+1} \le 1$. For $y_{k+1}$, we know
\begin{align*}
    y_{k+1} = \mathcal{P}_{[-D,D]} (y_k + 2\eta_y h(x_k)-\eta_y h(x_{k-1})).
\end{align*}
Since $h(x) =\frac{L}{2}x^2$ when $|x| \leq 1$, and $x_k \geq \frac{1}{\sqrt{2}}x_{k-1}$, we know that $2\eta_y h(x_k)-\eta_y h(x_{k-1})\geq 0$ so $y_{k+1} = 1$. Till now, we have proved the claim.

Then, we are going to bound the magnitude of $x_{T}$. According to the updating rule we have:
\begin{align*}
    x_{t+1} &= x_t - 2\eta_x LD x_t  + \eta_x LD x_{t-1}. 
\end{align*}
Solving the above recursion we get the solution for $x_t$ as follows:
\begin{align*}
    x_t &= \left(\frac{1}{2} + \frac{1}{2\sqrt{\Delta}}\right) \left(\frac{1-2\eta_x LD + \sqrt{\Delta}}{2}\right)^{t} x_0\\
    & \quad + \left(\frac{1}{2} - \frac{1}{2\sqrt{\Delta}}\right) \left(\frac{1-2\eta_x LD - \sqrt{\Delta}}{2}  \right)^t x_0,
\end{align*}
where $ \Delta = (1-2\eta_x L D)^2 + 4\eta_xLD$.

Let $a_1 = \left(\frac{1}{2} + \frac{1}{2\sqrt{\Delta}}\right)$, $a_2 = \left(\frac{1}{2} - \frac{1}{2\sqrt{\Delta}}\right)$, and $\lambda_1 = \left(\frac{1-2\eta_x LD + \sqrt{\Delta}}{2}\right)$, $\lambda_2 = \left(\frac{1-2\eta_x LD - \sqrt{\Delta}}{2}  \right)$. We observe the following facts:
\begin{align*}
    & a_1 \geq \frac{1}{2}  ,  a_2 \leq \eta_x^2 L^2D^2, \\
    &1-\eta_x LD \leq \lambda_1 \leq 1, -\eta_x LD \leq \lambda_2 \leq 0.
\end{align*}
Now, we can bound the magnitude of $x_T$
\begin{align*}
   |x_{T }|  =  \left| a_1\lambda_1^T  +  a_2 \lambda_2^T \right|x_0 &\geq \left||a_1\lambda_1^T| - |a_2 \lambda_2^T|\right|x_0 \\
   &\geq 
   \left(\frac{1}{2}(1-2\eta_x LD)^T - (\eta_xLD)^{T+2}\right) x_0.
\end{align*}
Since $\nabla \Phi_{1/2\ell}(x)=\frac{2LD\ell}{LD+2\ell}x$, by choosing $x_0=\frac{LD+2\ell}{LD\ell}\cdot 4\epsilon$, we have
\begin{align*}
   \epsilon &\geq |\nabla \Phi_{1/2\ell}(x_{T})|\ge 8\epsilon\left(\frac{1}{2}(1-2\eta_x LD)^T - \frac{1}{4}\right),
\end{align*}
which yields $(1-2\eta_x LD)^T\le 3/4$. The rest of proof is similar to that of Theorem~\ref{thm:NCC_Tightness_GDA}.
\end{proof}

\begin{proof}[Proof of Theorem~\ref{thm:NCC_Tightness_EG/OGDA} for EG]
We use the same hard example $f(x,y)=h(x)y$ as in proof of Theorem~\ref{thm:NCC_Tightness_GDA}. Similarly to our previous proofs for GDA and OGDA, we first claim that if we choose $0\le x_0 \le 1$ and $y_0 = D$,  the following statements hold for any $t\ge 0$:
\begin{align*}
    (a) \ 0\le x_t\le 1;  (b) \ y_t = D, y_{t+1/2}=D.
\end{align*} 
We prove this claim by induction. First, when $t = 0$, the claim holds for sure. Then, let us assume it holds for $t \leq k$. Then for $t = k+1$, we have
\begin{align*}
    x_{k+1} &= x_k-\eta_x L y_{k+1/2} x_{k+1/2}\\
    &=x_k-\eta_x L y_{k+1/2} \left(1-\eta_x L y_{k}\right)x_k\\
    &=(1-\eta_x LD +\eta_x^2 L^2 D^2)x_k.
\end{align*}
Note that since $0\le \eta_x LD\le 1/2$, we know
\begin{align*}
    0\le 1-\eta_x LD +\eta_x^2 L^2 D^2\le 1,
\end{align*}
which implies $0\le x_{k+1}\le 1$. Regarding $y$, note that 
\begin{align*}
    y_{k+1} &= \mathcal{P}_{[-D,D]} (y_{k} +  \eta_y h(x_{k+1/2}))),\\
    y_{k+3/2} &= \mathcal{P}_{[-D,D]} (y_{k+1} +  \eta_y h(x_{k+1}))).
\end{align*}
As $h(x_{k+1/2}), h(x_{k+1})\ge 0$ and $y_k=D$, we have $y_{k+1}=y_{k+3/2}=D$. Till now, we have verified the claim. 

Note that
\begin{align*}
    x_{k+1}=&(1-\eta_x LD +\eta_x^2 L^2 D^2)x_k\ge (1-\eta_x LD )x_k.
\end{align*}
Hence we can unroll the recursion and lower bound the magnitude of $\nabla \Phi_{1/2\ell}(x_T)$, which is similar to the proof of Theorem~\ref{thm:NCC_Tightness_GDA}.
\end{proof}

\section{Proof of Stepsize-Independent Lower Bound Results in Nonconvex-Strongly-Concave Setting}
\label{app:lw_ncsc}

In this section, we prove general lower bounds on the convergence rate of GDA/EG/OGDA for the NC-SC setting. In subsection~\ref{app:ncsc_lower}, proof of theorem~\ref{thm:lower_bound_gda} is established giving the lower bound for GDA in NC-SC, and in subsection~\ref{app:ncsc_lower_og}, the proof of Theorem~\ref{thm:lower_bound_eg} is established, proving the lower bound of EG/OGDA for NC-SC problems.

\subsection{Lower Bound for GDA}
\label{app:ncsc_lower}
\begin{theorem}[Theorem~\ref{thm:lower_bound_gda} restated]
For GDA algorithm, given $\eta_y = \Theta(1/\ell)$, for any $\eta_x $, there exists a $\ell$-smooth function that is nonconvex in $x$ and $\mu$-strongly-concave in y, such that for $\|\Phi(x_T)\|\leq \epsilon$, we must have:
\begin{align*}
    T = \Omega \left(\frac{\kappa \ell \Delta_\phi}{\epsilon^2}\right)
\end{align*}
\end{theorem}
\begin{proof}
Combining Proposition~\ref{prop:NCSC_GDA_1} and~\ref{prop:NCSC_GDA_2} will conclude the proof. Proposition~\ref{prop:NCSC_GDA_2} shows that when $\eta_x \in (\frac{1}{\kappa \ell} . \infty)$, GDA diverges, and Proposition~\ref{prop:NCSC_GDA_1} shows the lower bound on the convergence rate when $\eta_x \in (0 , \frac{1}{\kappa \ell}]$.
\end{proof}

\begin{proposition}\label{prop:NCSC_GDA_1}
For GDA algorithm, given $\eta_y = \Theta(1/\ell)$, for any $\eta_x \in (0,\frac{1}{\kappa\ell}]$, there exists a $\ell$-smooth function that is nonconvex in $x$ and $\mu$-strongly-concave in y, such that for $\|\Phi(x_T)\|\leq \epsilon$, we must have:
\begin{align*}
    T = \Omega \left(\frac{\kappa \ell \Delta_\phi}{\epsilon^2}\right)
\end{align*}

\begin{proof}
Recall that we consider the following quadratic NC-SC function $f:\mathbb{R}\times\mathbb{R}\to \mathbb{R}$
\begin{align*}
	f(x,y):=-\tfrac{1}{2}\ell x^2+bxy-\tfrac{1}{2}\mu y^2.
	\end{align*}
Recall that $f$ is nonconvex in $x$ (it is actually concave in $x$) and $\mu$ strongly concave in $y$. Assume $\kappa:=\ell/\mu\ge 4$ and choose 	$b = \sqrt{\mu (\ell+\mu_x)}
$ for some $0<\mu_x\le \ell/2$ to be chosen later.
Then we know $b\le \ell/2$, and it is easy to verify $f$ is $\ell$ smooth. Note that the primal function
\begin{align*}
	\Phi(x)=\max_y f(x,y)=\tfrac{1}{2}\mu_x x^2
\end{align*}
is actually strongly convex. This also justifies the symbol for $\mu_x$. We use GDA to find the solution for $\min_x\max_y f(x,y)$. Actually, for this problem, the optimal solution is achieved at the origin. The stepsizes ratio is chosen as
$	r = \frac{\eta_y}{\eta_x}$ and $ \eta_y = \frac{1}{\ell}$ for some numerical constants $c$. Then the GDA update rule can be written as
\begin{align}
\begin{pmatrix}
x_{k+1}\\ y_{k+1} 
\end{pmatrix}= (\mathbf{I}+\eta_x \mathbf{M})\cdot \begin{pmatrix}
x_{k}\\ y_{k} 
\end{pmatrix},
\label{eq:gda_update}
\end{align}
where 
\begin{align}
\mathbf{M}:=\begin{pmatrix}
\ell & -b\\ r b & -\mu r 
\end{pmatrix}. \label{eq:matrix}
\end{align}
Note that \eqref{eq:gda_update} is a linear time-invariant system. We need to analyze its eigenvalues. Let $\lambda_1$ and $\lambda_2$ be the two eigenvalues of $\mathbf{M}$, we have
\begin{align*}
	\lambda_{1,2} = -\frac{1}{2}\left(\mu r- \ell\right) \pm \frac{1}{2}\sqrt{\left(\mu r-\ell\right)^2-4r\mu\mu_x}.
\end{align*}
Note that if we choose $\mu_x< \ell/8$, plugging into $r=c\kappa$, we can bound
\begin{align*}
0\ge	\lambda_1 &= -\frac{(2\kappa-1)\ell}{4}\left(1-\sqrt{1-\frac{4c\kappa\mu_x}{(\mu r-\ell)^2}}\right)\\
&\ge-\frac{2\mu r \mu_x}{\mu r -\ell} \ge -4\mu_x.
\end{align*}

Let $s_1$ be the corresponding eigenvalue of $\mathbf{I}+\eta_x \mathbf{M}$, for small enough $c_1\le 1$, it satisfies
\begin{align*}
	 0\le 1-\frac{\mu_x}{r\ell} = 1-\frac{1}{r\kappa_x}\le s_1=1+\eta_x \lambda_1\le 1.
\end{align*}
We adversarially choose the initial point $(x_0,y_0)$ such that it is parallel to the eigenvector of $\mathbf{I}+\eta_x \mathbf{M}$ corresponding to $s_1$. We can always choose $x_0\ge 0$ for simplicity. Then we have
\begin{align*}
 \begin{pmatrix}
    x_{k+1}\\ y_{k+1} 
\end{pmatrix}&= (\mathbf{I}+\eta_x \mathbf{M})^T \begin{pmatrix}
x_{0}\\ y_{0} 
\end{pmatrix} = s_1^T \begin{pmatrix}
x_{0}\\ y_{0} 
\end{pmatrix},
\end{align*}
so we can compute the magnitude of $x_T$ as	$x_T = s_1^T x_0$. Choose $\mu_x = \frac{\kappa \ell}{2T}$, and thus we have:
\begin{align*}
    \|\nabla \Phi(x_T)\| = \|\mu_x x_0 \| = \mu_x \left (1-\frac{1}{r\kappa_x}\right)^T|x_0| \geq\mu_x \left (1-\frac{1}{\kappa\kappa_x}\right)^T|x_0| \geq \mu_x  \exp\left(\frac{2T}{\kappa\kappa_x}\right) |x_0| \geq \frac{1}{2} \mu_x |x_0|
\end{align*}
where we use the inequality that $1-\frac{z}{2} \geq \exp (z\ln \frac{1}{2})$ and $\exp(z\ln\frac{1}{2}) \geq \frac{1}{2}$ for $z \in [0,1]$. Recall that we choose $x_0 =   \sqrt{\frac{2\Delta_\Phi}{\mu_x}}$, we have:
\begin{align*}
    \|\nabla \Phi(x_T)\| \geq \frac{1}{2}\sqrt{{2\mu_x\Delta}} = \Omega\left(\sqrt{{\frac{\kappa \ell\Delta}{T}}}\right),
\end{align*}
which means to guarantee that $ \|\nabla \Phi(x_T)\| \leq \epsilon$, we must have $T \geq \Omega \left(\frac{\kappa \ell \Delta_\Phi}{\epsilon^2}\right)$.

\end{proof}
\end{proposition}

\begin{proposition}\label{prop:NCSC_GDA_2}
For GDA algorithm, given $\eta_y = \Theta(1/\ell)$, for any $\eta_x \in (\frac{1}{\kappa\ell},\infty)$, there exists a $\ell$-smooth function that is nonconvex in $x$ and $\mu$-strongly-concave in y, such that:
\begin{align*}
    \|\nabla \Phi(x_T)\| \geq c
\end{align*}
where $c$ is some constant that does not vanish as $T$ increases.
\begin{proof}
Recall the transition matrix in (\ref{eq:matrix}). We notice that
\begin{align*}
    {\textsf{trace}}(\mathbf{M}) = \lambda_1 + \lambda_2 = L - \mu r.
\end{align*}
Since $r\leq \kappa$, then $\lambda_1 + \lambda_2 \geq 0$, which means that $\max\{Re[\lambda_1],Re[\lambda_2]\} \geq 0$, so:
\begin{align*}
    \|(\mathbf{I}+\eta_x \mathbf{M})^T\| \geq \max\{|1+\eta_x \lambda_1|,|1+\eta_x \lambda_2|\}^T \geq \alpha^T
\end{align*}
where $\alpha$ is some constant larger than $1$. If we choose the initialization to be $[x_0,0]$, the gradient $\|\nabla \Phi(x_T)\| =  {\mu_x} \|(\mathbf{I}+\eta_x \mathbf{M})^T\|x_0$ diverges.
\end{proof}
\end{proposition}

\subsection{Lower bound for EG/OGDA}
\label{app:ncsc_lower_og}
\begin{theorem}[Theorem~\ref{thm:lower_bound_eg}  restated] 
\label{thm:aa1}
For deterministic EG/OGDA algorithm, given $\eta_y = \Theta(1/\ell)$, for any $\eta_x $, there exists a $\ell$-smooth function that is nonconvex in $x$ and $\mu$-strongly-concave in y, such that for $\|\Phi(x_T)\|\leq \epsilon$, we must have:
\begin{align*}
    T = \Omega \left(\frac{\kappa \ell \Delta_\phi}{\epsilon^2}\right)
\end{align*}
\begin{proof}[Proof of Theorem~\ref{thm:aa1} for EG]
 We consider the same quadratic hard example $f$ and notation used in the proof of Theorem~\ref{thm:lower_bound_gda}.
For simplicity, denote $\vw=(x,y)$. Then the updating rule for EG can be written as:

\begin{align*}
	\vw_{k+1/2}=&(\mathbf{I}+\eta_x \mathbf{M}) \vw_k,\\
	\vw_{k+1}=& \vw_k+\eta_x \mathbf{M} \vw_{k+1/2}\\
	=& (\mathbf{I}+\eta_x \mathbf{M}+\eta_x^2\mathbf{M}^2) \vw_k.
\end{align*}

Therefore, similar to GDA, EG is also a linear time-invariant system with the difference that the transition matrix now becomes as $\mathbf{M}' = (\mathbf{I}+\eta_x \mathbf{M}+\eta_x^2\mathbf{M}^2)$.

The rest of the analysis is the same as that of GDA in Proposition~\ref{prop:NCSC_GDA_1}. Then, we are going to show that when $\eta_x \in (\frac{1}{c_x\kappa\ell},+\infty)$ for some $c_x$, the EG method diverges. Consider  \begin{align*}
	f(x,y):=-\tfrac{1}{2}\ell x^2+bxy-\tfrac{1}{2}\mu y^2.
	\end{align*}
Then according to Proposition~\ref{prop:NCSC_GDA_1}, we have:
\begin{equation}
\begin{split}
    \textsf{trace}(\mathbf{M}') &=  \textsf{trace}( \mathbf{I}+\eta_x \mathbf{M}+\eta_x^2\mathbf{M}^2) \\
    &=1  +\eta_x (\ell - \mu r) +  \eta_x^2 (\ell^2 + \mu^2 r^2 - 2 r b^2) \\
    &= 1 +   \eta_x (\ell - \mu r) +  \eta_x^2 \left(( \ell - \mu r)^2- 2 r \mu \mu_x \right ) 
\end{split}
\end{equation}
Now note that since $r \le \kappa$, to show $\textsf{trace}(\mathbf{M}') \ge 1$, it is enough to have $\mu_x \le \frac{(\ell - \mu r)^2}{2 r \mu}$. However, by choosing $\mu_x = \Theta( \epsilon^2)$, and by choosing the small enough $\epsilon$, we can satisfy the condition that $\mu_x \le \frac{(\ell - \mu r)^2}{2 r \mu}$, thus we can conclude that under this situation $\textsf{trace}(\mathbf{M}') \ge 1$, which means that same step as the Proposition~\ref{prop:NCSC_GDA_2} can be taken to prove the divergence of $\| \nabla \Phi(x_T) \|^2$.
 
\end{proof}
\begin{proof}[Proof of Theorem~\ref{thm:aa1} for OGDA] Assuming the same setup as the proof of EG, the update rule can be written as follows: 
The dynamics of OGDA is
\begin{align*}
	\vw_{k+1} = \vw_k +2\eta_x\mathbf{M} \vw_k -\eta_x \mathbf{M} \vw_{k-1}.
\end{align*}
If we initialize $\vw_0$ parallel to the eigenvector of $\mathbf{M}$ corresponding to $\lambda_1$ and let $\vw_1=\vw_0$, we know every $\vw_k$ is parallel to it, i.e., $\vw_k = z_k \vw_0$ for some scalar $z_k$ which satisfies
\begin{align*}
		z_{k+1} = z_k +2\eta_x \lambda_1 z_k -\eta_x \lambda_1 z_{k-1}.
\end{align*}
The general solution of the above recurrence relation is
\begin{align*}
	z_k = a \alpha^k + b\beta^k
\end{align*}
for some constant $a,b$ and
\begin{align*}
	\alpha =& \frac{1}{2}\left(1+2\eta_x\lambda_1+\sqrt{1+4\eta_x^2\lambda_1^2}\right),\\
	\beta =& \frac{1}{2}\left(1+2\eta_x\lambda_1-\sqrt{1+4\eta_x^2\lambda_1^2}\right).
\end{align*}
We have
\begin{align*}
	1+\eta_x\lambda_1 \le \alpha\le 1, \quad \eta_x\lambda_1\le  \beta\le 0.
\end{align*}
Using the initial condition $z_{-1}=z_0=1$, we can get the constants
\begin{align*}
	a=&\frac{\alpha(1-\beta)}{\alpha-\beta}=\frac{1}{2}+\frac{1}{2\sqrt{1+4\eta_x^2\lambda_1^2}}\ge 1/2,\\
	b=&-\frac{\beta(1-\alpha)}{\alpha-\beta}=\frac{\sqrt{1+4\eta_x^2\lambda_1^2}-1}{2\sqrt{1+4\eta_x^2\lambda_1^2}}\le \eta_x^2\lambda_1^2.
\end{align*}
We can bound
\begin{align*}
	|z_T|&\ge \frac{1}{2}\left(1+\eta_x\lambda_1 \right)^T - |\eta_x\lambda_1|^{k+2}\\
	&\ge \frac{1}{2}\left(1-\frac{4c_1\mu_x}{\kappa}\right)^{T}-\frac{1}{4},
\end{align*}
where we use the fact $|\eta_x\lambda_1|\le 1/2$.
Similar to the analysis for GDA, choosing $\mu_x=50\epsilon^2/\Delta_{\Phi}$, we have
\begin{align*}
	\abs{\nabla \Phi(\Bar{x})}=\mu_x \Bar{x}\ge& \mu_x x_T\ge \mu_x x_0\left[\frac{1}{2}\left(1-\frac{4c_1\mu_x}{\kappa}\right)^{T}-\frac{1}{4}\right]\\
	=&10\epsilon \left[\frac{1}{2}\left(1-\frac{4c_1\mu_x}{\kappa}\right)^{T}-\frac{1}{4}\right].
\end{align*}
Therefore, if $\abs{\nabla \Phi(\Bar{x})}\le\epsilon$, we must have
\begin{align*}
	T= \Omega\left(\frac{\kappa}{\mu_x}\right)=\Omega\left(\frac{\kappa\Delta_\Phi}{\epsilon^2}\right).
\end{align*}

Now, we will show that Proposition~\ref{prop:NCSC_GDA_2} also holds for OGDA. Consider the following $4 \times 4$ matrix $\mathbf{M}'$: 

\begin{equation}
\mathbf{M}' = \left[
    \begin{array}{c;{2pt/2pt}c}
        (\mathbf{I} + 2 \eta_x \mathbf{M})^2 & -\eta_x (\mathbf{I} + 2 \eta_x \mathbf{M}) \mathbf{M} \\ \hdashline[2pt/2pt]
        \mathbf{I} + 2 \eta_x \mathbf{M}  & - \eta_x \mathbf{M}  
    \end{array}
\right]
\end{equation}
It can be easily shown that, the OGDA dynamic can be written as follows: 

\begin{equation}
   \left[ \begin{array}{c}
        \vw_{k+1} \\
        \vw_k
    \end{array} \right] 
 = \left[
    \begin{array}{c;{2pt/2pt}c}
        (\mathbf{I} + 2 \eta_x \mathbf{M})^2 & -\eta_x (\mathbf{I} + 2 \eta_x \mathbf{M}) \mathbf{M} \\ \hdashline[2pt/2pt]
        \mathbf{I} + 2 \eta_x \mathbf{M}  & - \eta_x \mathbf{M}  
    \end{array}
\right]   \left[  \begin{array}{c}
        \vw_{k-1} \\ 
        \vw_{k-2}
    \end{array} \right]
\end{equation}

Now similar to proof of Proposition~\ref{prop:NCSC_GDA_2} for GDA, it suffices to show that the $\textsf{trace}(\mathbf{M}') \ge 1$ given the conditions on the learning rate. To this end, note that we can write: 
\begin{equation}
\begin{split}
    \textsf{trace}(\mathbf{M}') &= \textsf{trace}(-\eta_x \mathbf{M} ) + \textsf{trace}( \mathbf{I} + 4 \eta_x \mathbf{M} + 4 \eta_x^2 \mathbf{M}^2) \\
    &=1  - \eta_x (\ell - \mu r) + 4 \eta_x (\ell - \mu r) + 4 \eta_x^2 (\ell^2 + \mu^2 r^2 - 2 r b^2) \\
    &= 1 +   3 \eta_x (\ell - \mu r) + 4 \eta_x^2 \left(( \ell - \mu r)^2- 2 r \mu \mu_x \right ) 
\end{split}
\end{equation}
Now note that since $r \le \kappa$, to show $\textsf{trace}(\mathbf{M}') \ge 1$, it is enough to have $\mu_x \le \frac{(\ell - \mu r)^2}{2 r \mu}$. However, note that we let $\mu_x = \frac{50 \epsilon^2}{\Delta_{\Phi}}$, thus by choosing the small enough $\epsilon$, we can satisfy the condition that $\mu_x \le \frac{(\ell - \mu r)^2}{2 r \mu}$, thus we can conclude that $\textsf{trace}(\mathbf{M}') \ge 1$ holds. Consequently, similar argument  as the Proposition~\ref{prop:NCSC_GDA_2} can be made to prove the divergence of $\| \nabla \Phi(x_T) \|^2$.
\end{proof}
\end{theorem}

\clearpage

\section{Extension to Generalized OGDA}\label{app:gen_ogda}
In this section, we analyze the convergence of generalized OGDA (Algorithm~\ref{alg:gsogda}) where we utilize different learning rates  for descent/ascent gradients and correction terms.  Specifically, we propose  to use different learning rates for $\nabla_x f(\vx_t , \vy_t)$, and $\nabla_x f(\vx_t , \vy_t) - \nabla_x f(\vx_{t-1} , \vy_{t-1})$ terms, and also $\nabla_y f(\vx_t , \vy_t)$, and $\nabla_y f(\vx_t, \vy_t) - \nabla_y f(\vx_{t-1} , \vy_{t-1})$, in order to make the algorithm more stable. We believe this algorithm is more convenient in practice due to the more flexibility it provides in deciding the  learning rates. We demonstrated this stabilizing effect of generalized OGDA in our empirical results in Section~\ref{sec:exp}. Also, note that if we let $\eta_{x,1} = \eta_{x,2}$, and $\eta_{y,1}=\eta_{y,2}$ in Algorithm~\ref{alg:gsogda}, it reduces to stochastic OGDA. Theorem~\ref{thm:gogda} establishes the convergence rate of generalized OGDA in NC-SC. However, it still remains open to analyze this algorithm in C-C/SC-SC and NC-C settings. 

We remark that the analysis of generalized OGDA was only known for the restricted  bilinear functions, which is established in~\cite{mokhtari2020unified}, and convergence analysis beyond these simple functions previously was unknown that we provide here.

\begin{algorithm}[H]
\caption{Generalized Stochastic OGDA}
\label{alg:gsogda}
\begin{algorithmic}
\STATE\textbf{Input:} $(\vx_0,\vy_0)$, stepsizes $(\eta_{x,1}, \eta_{x,2} , \eta_{y,1} , \eta_{y,2})$
\FOR{$t=1,2,\dots,T$}
\STATE $\vx_t \gets \vx_{t-1} - \eta_{x,1}  \vg_{x,t-1} -  \eta_{x,2} (\vg_{x,t-1} -  \vg_{x,t-2})$ 
\STATE $\vy_t \gets \vy_{t-1} + \eta_{y,1}  \vg_{y,t-1} +  \eta_{y,2} (\vg_{y,t-1} -  \vg_{y,t-2})$ 
\ENDFOR
\STATE Randomly choose $\Bar{\vx}$ from $\vx_1,\dots,\vx_T$ 
\STATE \textbf{Output:}$\Bar{\vx}$
\end{algorithmic}\vspace{-0.1cm}
\end{algorithm}

\begin{theorem}
\label{thm:gogda}
Let $\eta_{x,1} = \frac{1}{50 \kappa^2 \ell}$, $\eta_{y,2} = \frac{1}{6 \ell}$. Also, let $\alpha = \frac{\eta_{x,2}}{\eta_{x,1}}$, and $\beta = \frac{\eta_{y,1}}{\eta_{y,2}}$. Then assuming $\beta \le 1$, and $\alpha \le  2 \kappa^2 \sqrt{\beta}$,  under Assumptions~\ref{asm:1}, and~\ref{asm:2} for Algorithm~\ref{alg:gsogda} we have: 
\begin{equation}
\begin{split}
\E [\| \nabla \Phi (\bar{\vx}) \|^2] &\le O\Big{(}\frac{\kappa^2 \ell \Delta}{T} + \frac{(\kappa + \alpha^2) \ell^2 D}{\beta T} + \frac{\kappa \sigma^2}{M_y} 
 + \frac{(1 +\alpha^2) \sigma^2}{M_x}\Big{)},
\end{split}
\end{equation}
where  $D = \max (\|\vy_1 - \vy^*_1\|^2 , \|\vy_1 - \vy_0 \|^2 , \|\vx_1 - \vx_0 \|^2)$, and $\Delta = \phi(\vx_1) - \min_{\vx} \Phi(\vx)$. 
\end{theorem}
A few observations about the obtained rate are in place. 
\begin{corollary}
Let $\sigma = 0$, and pick an $\alpha \le \sqrt{\kappa} $. Then deterministic generalized OGDA converges to $\epsilon$-stationary point of $\Phi(\vx)$ with gradient complexity of $O(\frac{\kappa^2}{\epsilon^2})$. 
\end{corollary}

\begin{corollary}
For any $\alpha = O(\sqrt{k})$, and any $\mu \le \beta \le 1$, if we choose  $M_x = O( \kappa \frac{\sigma^2}{\epsilon^2})$, and $M_y = O(\frac{\kappa}{\epsilon^2})$, then stochastic generalized OGDA converges to $\epsilon$-stationary point of $\Phi(\vx)$ with gradient complexity of $O(\frac{\kappa^3}{\epsilon^4})$.
\end{corollary}

\begin{remark}
Theorem~\ref{thm:gogda} establishes the convergence rate under broad range of primal learning rates ratio ($0 \le \alpha \le O(\kappa^2))$, and it shows that as long as $\alpha \le \sqrt{\kappa}$, we can achieve the same convergence rate as OGDA if we assume $ \mu \le \beta \le 1$.  
\end{remark}

\subsection{Nonconvex-strongly-concave setting}

We follow exact same steps as Lemma~\ref{lemma:3}, to derive the following lemmas.
\begin{lemma}
\label{lemma:30}
Let $\Phi (\vx) = \max_{\vy} f(\vx,\vy)$, and $\vy^*(\vx)= \arg \max_{\vy} f(\vx,\vy)$. Also, let $\vg_{i}= \vg_{x,i} + \alpha  (\vg_{x,i} - \vg_{x,i-1}) $, where $\alpha = \frac{\eta_{x,2}}{\eta_{x,1}}$. Therefore, we have $\vx_i = \vx_{i-1} - \eta_{x,1} \vg_i$. Then for Algorithm~\ref{alg:gsogda}, we have: 
\begin{equation}
\label{eqn:glm1}
\begin{split}
&\E[\Phi(\vx_t)] \le \E[\Phi(\vx_{t-1})] - \frac{\eta_{x,1}}{2} \E[\| \nabla \Phi(\vx_{t-1})\|^2] - \frac{\eta_{x,1}}{2}(1- 2 \kappa \ell \eta_{x,1}) \E[\|\vg_{t-1}\|^2] \\
&\qquad+ \frac{3}{2} \eta_{x,1}^3 \alpha^2 \ell^2 \E[\|\vg_{t-2}\|^2] + \frac{3}{2}\eta_{x,1} \ell^2  \E[\|\vy^*_{t-1} - \vy_{t-1} \|^2] + \frac{3}{2}\eta_{x,1} \alpha^2 \ell^2 \E[\|\vy_{t-1} - \vy_{t-2}\|^2] \\
&\qquad+  3 ((1+ \alpha)^2+1) \eta_{x,1} \frac{\sigma^2}{M_x}
\end{split}
\end{equation}

\end{lemma}

\begin{proof}[Proof of Lemma~\ref{lemma:30}] Proof is pretty much similar to proof of Lemma~\ref{lemma:3}, and we only include this proof for sake of completeness.
First, let $\vdelta^x_i = \vg_{x,i} - \nabla_x f(\vx_i,\vy_i)$. By definition of $\vg_{x,i}$, we have $\E[\vdelta^x_i] =0$, for all $i \in [T]$.

Using the fact that $ \Phi(\vx)$ is $ 2 \kappa \ell$ smooth, we have: 
\begin{equation}
\label{eqn:gogda1}
\begin{split}
\Phi(\vx_t) &\le \Phi(\vx_{t-1})  + \ip{\nabla \Phi(\vx_{t-1}) , \vx_t - \vx_{t-1} } + \kappa \ell \|\vx_t - \vx_{t-1}\|^2 \\
&= \Phi(\vx_{t-1})  -\eta_{x,1} \ip{\nabla \Phi(\vx_{t-1}) , \vg_{t-1} }  + \kappa \ell \eta_{x,1}^2 \| \vg_{t-1}\|^2 \\
&= \Phi(\vx_{t-1}) - \frac{\eta_{x,1}}{2} \| \nabla \Phi(\vx_{t-1})\|^2 - \frac{\eta_{x,1}}{2} \|\vg_{t-1}\|^2 + \frac{\eta_{x,1}}{2} \| \nabla \Phi (\vx_{t-1}) - \vg_{t-1} \|^2 \\
&\quad+  \kappa \ell \eta_{x,1}^2 \|\vg_{t-1}\|^2 \\
&= \Phi(\vx_{t-1}) - \frac{\eta_{x,1}}{2} \| \nabla \Phi(\vx_{t-1})\|^2 - \frac{\eta_{x,1}}{2}(1- 2 \kappa \ell \eta_{x,1}) \|\vg_{t-1}\|^2 \\
&\quad+ \frac{\eta_{x,1}}{2} \| \nabla \Phi (\vx_{t-1}) - \vg_{t-1} \|^2 \\
\end{split}
\end{equation}
Now using $\ell$-smoothness of $f$, and $\kappa$-Lipschitzness of $\vy^*(\vx)$ (Lemma~\ref{app:lemma:smooth}) we have: 
\begin{equation}
\label{eqn:gogda2}
\begin{split}
 &\| \nabla \Phi (\vx_{t-1}) - \vg_{t-1} \|^2  = \| \nabla \Phi(\vx_{t-1}) - \nabla_x f(\vx_{t-1},\vy_{t-1}) \\
 &\qquad- \alpha \left( \nabla_x f(\vx_{t-1},\vy_{t-1}) -  \nabla_x f(\vx_{t-2},\vy_{t-2})\right) - ((\alpha + 1) \vdelta^x_{t-1} - \vdelta^x_{t-2}) \|^2 \\
 &\quad\le 3  \| \nabla \Phi(\vx_{t-1}) - \nabla_x f(\vx_{t-1},\vy_{t-1})\|^2 + 3 \alpha^2 \| \nabla_x f(\vx_{t-1},\vy_{t-1}) -  \nabla_x f(\vx_{t-2},\vy_{t-2})\|^2 \\
 &\qquad+ 3 \| (\alpha +1 ) \vdelta^x_{t-1} - \vdelta^x_{t-2} \|^2 \\
 &\quad\le 3 \ell^2 \| \vy^*(\vx_{t-1}) - \vy_{t-1} \|^2 +  3 \alpha^2 \ell^2 \|\vx_{t-1} - \vx_{t-2}\|^2 + 3 \alpha^2 \ell^2 \|\vy_{t-1} - \vy_{t-2}\|^2 \\
 &\qquad+ 6 (\alpha +1)^2 \|\vdelta^x_{t-1}\|^2 + 6 \|\vdelta^x_{t-2}\|^2
\end{split}
\end{equation}
where in the first and second inequalities we used Young's inequality.

By combining Equations~\ref{eqn:gogda1} and~\ref{eqn:gogda2} we have: 
\begin{equation}
\label{eqn:gogda3}
\begin{split}
\Phi(\vx_t) &\le \Phi(\vx_{t-1}) - \frac{\eta_{x,1}}{2} \| \nabla \Phi(\vx_{t-1})\|^2 - \frac{\eta_{x,1}}{2}(1- 2 \kappa \ell \eta_{x,1}) \|\vg_{t-1}\|^2 \\
&\quad+ \frac{3}{2}\eta_{x,1} \ell^2  \|\vy^*_{t-1} - \vy_{t-1} \|^2 + \frac{3}{2} \eta_{x,1} \alpha^2 \ell^2 \|\vx_{t-1} - \vx_{t-2}\|^2 + \frac{3}{2}\eta_{x,1} \alpha^2 \ell^2 \|\vy_{t-1} - \vy_{t-2}\|^2 \\ 
&\quad+ 3 \eta_{x,1}(\alpha +1 )^2 \|\vdelta_{t-1}^x\|^2 + 3 \eta_{x,1} \|\vdelta^x_{t-2}\|^2\\
&\le \Phi(\vx_{t-1}) - \frac{\eta_{x,1}}{2} \| \nabla \Phi(\vx_{t-1})\|^2 - \frac{\eta_{x,1}}{2}(1- 2 \kappa \ell \eta_{x,1}) \|\vg_{t-1}\|^2 + \frac{3}{2} \eta_{x,1}^3 \alpha^2 \ell^2 \|\vg_{t-2}\|^2 \\
&\quad+ \frac{3}{2}\eta_{x,1} \ell^2  \|\vy^*_{t-1} - \vy_{t-1} \|^2 + \frac{3}{2}\eta_{x,1} \ell^2 \alpha^2 \|\vy_{t-1} - \vy_{t-2}\|^2 + 3  \eta_{x,1} (\alpha + 1 )^2 \|\vdelta_{t-1}^x\|^2 \\
&\quad+ 3 \eta_{x,1} \|\vdelta^x_{t-2}\|^2
\end{split}
\end{equation}

We proceed by taking expectation on both side of Equation~\ref{eqn:gogda3}, to get: 
\begin{equation}
\label{eqn:gogda4}
\begin{split}
\E[\Phi(\vx_t)] &\le \E[\Phi(\vx_{t-1})] - \frac{\eta_{x,1}}{2} \E[\| \nabla \Phi(\vx_{t-1})\|^2] - \frac{\eta_{x,1}}{2}(1- 2 \kappa \ell \eta_{x,1}) \E[\|\vg_{t-1}\|^2] \\
&\quad + \frac{3}{2} \eta_{x,1}^3 \alpha^2 \ell^2 \E[\|\vg_{t-2}\|^2] + \frac{3}{2}\eta_{x,1} \ell^2  \E[\|\vy^*_{t-1} - \vy_{t-1} \|^2] \\
&\quad+ \frac{3}{2}\eta_{x,1} \alpha^2 \ell^2 \E[\|\vy_{t-1} - \vy_{t-2}\|^2] +  3 ((1+ \alpha)^2+1) \eta_{x,1} \frac{\sigma^2}{M_x}
\end{split}
\end{equation}
where we used the fact that $\E[\vdelta^x_i] \le \frac{\sigma^2}{M_x}$ for all $ i \in [T]$. 

\end{proof}

\begin{lemma}
\label{lemma:40}
Let $\eta_{y,2} = \frac{1}{6 \ell}$, then the following inequality holds true for generalized OGDA iterates:
\begin{equation}
\label{eqn:glm2}
\begin{split}
\sum_{i=1}^{t+1} \E[\| \vy_i - \vy^*_i\|^2 ] &\le \frac{9}{7} \E [\|\vy_1 - \vy^*_1\|^2] + \frac{36}{7} \sum_{i=2}^{t+1} \E[ \|\vz_{i} - \vy^*_i \|^2] + \frac{18}{7} \eta_{x,1}^2 \kappa^2 \sum_{i=1}^t \E[\| \vg_i\|^2] \\
&\quad+ \frac{2T \sigma^2}{ 7 \ell^2 M_y}
\end{split}     
\end{equation}

\end{lemma}

\begin{proof}[Proof of Lemma~\ref{lemma:40}]
Using Young's inequality, and $\kappa$-Lipschitzness of $\vy^*(\vx)$  we have:
\begin{equation}
\label{eqn:gogda5}
\begin{split}
\| \vy_{t+1} - \vy^*_{t+1} \|^2 &\le 2 \| \vy_{t+1} - \vy^*_t \|^2 + 2 \| \vy^*_{t+1} - \vy^*_t \|^2 \\
&\le  2 \| \vy_{t+1} - \vy^*_t \|^2 + 2 \kappa^2 \| \vx_{t+1} - \vx_t \|^2 
\end{split}
\end{equation}
Similar to Lemma~\ref{lemma:4}, we try to find an upper bound for $\|\vy_{t+1} - \vy^*_t \|^2$. Let $\vz_{t+1} = \vy_t + \eta_{y,1} \vg_{y,t} - \eta_{y,2} \vg_{y,t-1}$, and $\vdelta^y_i = \vg_{y,i} - \nabla_y f(\vx_i,\vy_i) $. Then we have:
\begin{equation}
\label{eqn:gogda6}
\begin{split}
\|\vy_{t+1} - \vy^*_t \|^2 &= \| \vz_{t+1} - \vy^*_t + \eta_{y,2} \vg_{y,t} \|^2 \\
&\le 2 \| \vz_{t+1} - \vy^*_t \|^2  + 2 \eta_{y,2}^2 \| \vg_{y,t}\|^2 \\
&\le 2 \| \vz_{t+1} - \vy^*_t \|^2 + 4 \eta_{y,2}^2 \| \nabla_y f(\vx_t , \vy_t) \|^2+ 4\eta_{y,2}^2 \| \vdelta_t^y\|^2 \\
&\le 2 \| \vz_{t+1} - \vy^*_t \|^2 + 4 \eta_{y,2}^2 \ell^2 \| \vy_t - \vy^*_t \|^2+ 4\eta_{y,2}^2 \| \vdelta_t^y\|^2
\end{split}
\end{equation}
The rest of the proof is exactly same as proof of Lemma~\ref{lemma:4}.
\end{proof}

Similar to Lemma~\ref{lemma:2}, we have:
\begin{lemma}
\label{lemma:20}
Let $\vz_{t+1} = \vy_t + \eta_{y,1} \vg_{y,t} - \eta_{y,2} \vg_{y,t-1}$, $\vr_t = \|\vz_{t+1} - \vy^*_t \|^2 + \frac{\beta}{4} \|\vy_t - \vy_{t-1} \|^2 $ and $\eta_{y,2} = \frac{1}{6 \ell}$. Also let $\frac{\eta_{y,1}}{\eta_{y,2}} = \beta$, and assume $\beta \le 1$. Then OGDA iterates satisfy the following inequalities:
\begin{equation}
\label{eqn:lm3.10}
\begin{split}
 \E[\vr_t]   &\le  (1 - \frac{\beta}{12 \kappa})  \E[\vr_{t-1}]  + 12 \eta_{x,1}^2 \kappa^3 \E[\|\vg_{t-1}\|^2]   + \frac{\beta \eta_{x,1}^2}{18} \E[\|\vg_{t-2}\|^2]  + \frac{\beta \sigma^2}{3 \ell^2 M_y}
\end{split}
\end{equation}

\begin{equation}
\label{eqn:lm3.20}
\begin{split}
\sum_{i=1}^{t} \E[\vr_i] \le \frac{12 \kappa}{\beta} \E[\vr_1] + \frac{2}{3}  \kappa \E[\|\vx_1 - \vx_0 \|^2] +  145 \frac{\eta_{x,1}^2 \kappa^4}{\beta}  \sum_{i=1}^{t-1} \E[\|\vg_{i} \|^2]    + \frac{4 \kappa \sigma^2 (t-1)}{ \ell^2 M_y}
\end{split}
\end{equation}
\end{lemma}

\begin{proof}[Proof of Lemma~\ref{lemma:20}]

Let $\vdelta^y_i = \vg_{y,i} - \nabla_y f(\vx_i , \vy_i)$, and note that we have $\vz_{t+1} - \vz_{t} = \eta_{y,1} \vg_{y,t}$. We have:
\begin{equation}
\label{eqn:gogda11}
\begin{split}
    \| \vz_{t+1} - \vy^*_t \|^2  &= \| \vz_{t} - \vy^*_t + \eta_{y,1} \vg_{y,t} \|^2   \\
    &= \| \vz_{t} - \vy^*_t \|^2 + 2 \eta_{y,1} \ip{ \vg_{y,t} , \vz_{t} - \vy^*_t } + \eta_{y,1}^2 \| \vg_{y,t}\|^2 \\
    &= \| \vz_t - \vy^*_t \|^2 - 2 \eta_{y,1} \eta_{y,2} \ip{ \vg_{y,t}, \vg_{y,t-1} } + 2 \eta_{y,1} \ip{ \vg_{y,t} , \vy_t - \vy^*_t } +  \eta_{y,1}^2 \| \vg_{y,t} \|^2 \\
    &= \| \vz_t - \vy^*_t \|^2  + \eta_{y,1} \eta_{y,2} \| \vg_{y,t} - \vg_{y,t-1} \|^2 + 2 \eta_{y,1} \ip{ \vg_{y,t} , \vy_t - \vy^*_t } \\
    &\quad- \eta_{y,1} \eta_{y,2} \| \vg_{y,t-1}\|^2 - \eta_{y,1}(  \eta_{y,2}- \eta_{y,1} ) \| \vg_{y,t}\|^2 \\
    &\le \| \vz_t - \vy^*_t \|^2 + 3 \eta_{y,1} \eta_{y,2} \|\nabla_y f(\vx_t,\vy_t) - \nabla_y f(\vx_{t-1},\vy_{t-1}) \|^2 \\
    &\quad+ 2 \eta_{y,1} \ip{\nabla_y f(\vx_t , \vy_t) , \vy_t - \vy^*_t} - \eta_{y,1} \eta_{y,2} \| \vg_{y,t-1} \|^2 - \eta_{y,1}(  \eta_{y,2}- \eta_{y,1} ) \| \vg_{y,t}\|^2  \\
    &\quad+ 3 \eta_{y,1} \eta_{y,2} \| \vdelta^y_t\|^2 + 3 \eta_{y,1} \eta_{y,2} \| \vdelta^y_{t-1}\|^2  +  2 \eta_{y,1} \ip{\vdelta^y_t, \vy_t - \vy^*_t}
    \\
    &\le  \| \vz_t - \vy^*_t \|^2 + 3 \eta_{y,1} \eta_{y,2} \ell^2 \|\vx_t - \vx_{t-1}\|^2 + 3 \eta_{y,1} \eta_{y,2} \ell^2 \|\vy_t - \vy_{t-1}\|^2\\
    &\quad- 2 \eta_{y,1} \mu  \|\vy_t - \vy^*_t\|^2 - \eta_{y,1} \eta_{y,2} \|\vg_{y,t-1} \|^2 - \eta_{y,1}(  \eta_{y,2}- \eta_{y,1} ) \| \vg_{y,t}\|^2  \\
    &\quad+3 \eta_{y,1} \eta_{y,2} \| \vdelta^y_t\|^2  + 3 \eta_{y,1} \eta_{y,2} \| \vdelta^y_{t-1}\|^2  +  2 \eta_{y,1} \ip{\vdelta^y_t, \vy_t - \vy^*_t}
\end{split}
\end{equation}
where the last inequality follows from smoothness of $f$, and strong concavity of $f(\vx_t,.)$. Now note that using Young's inequality we can write: 
\begin{equation}
\label{eqn:gogda12}
    \|\vy_t - \vy^*_t\|^2 \ge \frac{1}{2} \| \vz_t - \vy^*_t\|^2 - \eta_{y,2}^2 \| \vg_{y,t-1} \|^2 
\end{equation}
Now plugging Equation~\ref{eqn:gogda12} back to Equation~\ref{eqn:gogda11}, and letting $\eta_{y,1} = \beta \eta_{y,2}$, we have: 
\begin{equation}
\label{eqn:gogda13}
\begin{split}
     \| \vz_{t+1} - \vy^*_t \|^2   &\le  (1 - \beta \eta_{y,2} \mu)\| \vz_t - \vy^*_t \|^2 + 3  \beta  \eta_{y,2}^2 \ell^2 \|\vx_t - \vx_{t-1}\|^2 + 3 \beta \eta_{y,2}^2 \ell^2 \|\vy_t - \vy_{t-1}\|^2 \\
     &\quad- \beta \eta_{y,2}^2 (1 - 2 \eta_{y,2} \mu)  \|\vg_{y,t-1} \|^2 - \beta \eta_{y,2}^2(  1 - \beta ) \| \vg_{y,t}\|^2  \\
     &\quad + 3 \beta \eta_{y,2}^2 \| \vdelta^y_t\|^2 + 3 \beta \eta_{y,2}^2 \| \vdelta^y_{t-1}\|^2  +  2 \beta \eta_{y,2} \ip{\vdelta^y_t, \vy_t - \vy^*_t}
\end{split}
\end{equation}
We can also write: 
\begin{equation}
\label{eqn:gogda14}
\begin{split}
    \|\vy_{t} - \vy_{t-1} \|^2 &=  \| \eta_{y,1}\vg_{y,t-1} + \eta_{y,2} ( \vg_{y,t-1} - \vg_{y,t-2}) \|^2 \\
    &\le 2 \eta_{y,1}^2 \|\vg_{y,t-1}\|^2 + 2 \eta_{y,2}^2 \|\vg_{y,t-1} - \vg_{y,t-2}\|^2 \\
    &\le 2 \eta_{y,1}^2 \|\vg_{y,t-1}\|^2 + 6 \eta_{y,2}^2 \|\nabla_y f(\vx_{t-1},\vy_{t-1}) - \nabla_y f(\vx_{t-2},\vy_{t-2})\|^2 \\
    & \quad + 6 \eta_{y,2}^2 \| \vdelta^y_{t-1}\|^2 +  6 \eta_{y,2}^2 \| \vdelta^y_{t-2}\|^2 
    \\
    &\le 2 \eta_{y,1}^2 \|\vg_{y,t-1}\|^2 + 6  \eta_{y,2}^2 \ell^2 \|\vx_{t-1}-\vx_{t-2}\|^2 + 6  \eta_{y,2}^2 \ell^2 \|\vy_{t-1}-\vy_{t-2}\|^2 \\
    & \quad +  6 \eta_{y,2}^2 \| \vdelta^y_{t-1}\|^2 +  6 \eta_{y,2}^2 \| \vdelta^y_{t-2}\|^2\\
    &= 2 \beta^2 \eta_{y,2}^2 \|\vg_{y,t-1}\|^2 + 6  \eta_{y,2}^2 \ell^2 \|\vx_{t-1}-\vx_{t-2}\|^2 + 6  \eta_{y,2}^2 \ell^2 \|\vy_{t-1}-\vy_{t-2}\|^2 \\
    & \quad +  6 \eta_{y,2}^2 \| \vdelta^y_{t-1}\|^2 +  6 \eta_{y,2}^2 \| \vdelta^y_{t-2}\|^2\\
\end{split}
\end{equation}

Now adding $ 9 \beta \eta_{y,2}^2 \ell^2 \|\vy_t -\vy_{t-1}\|^2$ to both side of Equation~\ref{eqn:gogda13}, and using Equation~\ref{eqn:gogda14} we have:
\begin{equation}   
\label{eqn:gogda15}
\begin{split}
 &\| \vz_{t+1} - \vy^*_t \|^2 +  9\beta \eta_y^2 \ell^2 \|\vy_t -\vy_{t-1}\|^2   \le  (1 - \beta \eta_{y,2} \mu)\| \vz_t - \vy^*_t \|^2 + 3 \beta \eta_{y,1}^2 \ell^2 \|\vx_t - \vx_{t-1}\|^2  \\
 &\qquad- \beta \eta_{y,2}^2 (1 - 2 \eta_{y,2} \mu -24 \beta^2 \eta_{y,2}^2 \ell^2)  \|\vg_{y,t-1} \|^2 - \beta \eta_{y,2}^2 ( 1- \beta) \|\vg_{y,t}\|^2\\
 &\qquad+ 72 \beta \eta_{y,2}^4 \ell^4 \|\vx_{t-1}-\vx_{t-2}\|^2   + 72 \beta \eta_{y,2}^4 \ell^4 \|\vy_{t-1}-\vy_{t-2}\|^2 \\
 &\qquad+ 3 \beta \eta_{y,2}^2 (1 + 24 \eta_{y,2}^2 \ell^2) \| \vdelta^y_t\|^2 + 3 \beta \eta_{y,2}^2  (1 + 24 \eta_{y,2}^2 \ell^2) \| \vdelta^y_{t-1}\|^2 \\
&\qquad +  2 \beta \eta_{y,2} \ip{\vdelta^y_t, \vy_t - \vy^*_t}
\end{split}
\end{equation}

Now plugging $\eta_{y,2} = \frac{1}{6 \ell}$ into Equation~\ref{eqn:gogda15}, and assuming $\beta \le 1$ we have: 
\begin{equation}
\label{eqn:gogda16}
\begin{split}
 \| \vz_{t+1} - \vy^*_t \|^2 +   \frac{\beta}{4} \|\vy_t -\vy_{t-1}\|^2   &\le  (1 - \frac{\beta}{6 \kappa}) \left (\| \vz_t - \vy^*_t \|^2 \right)+  \frac{\beta}{18} \|\vy_{t-1} -\vy_{t-2}\|^2  \\
 & \quad + \frac{\beta}{12} \|\vx_t - \vx_{t-1}\|^2   + \frac{\beta}{18} \|\vx_{t-1}-\vx_{t-2}\|^2   \\
 &\quad + \frac{\beta}{6 \ell^2} \| \vdelta^y_t\|^2 +  \frac{\beta}{6 \ell^2} \| \vdelta^y_{t-1}\|^2  +  \frac{2 \beta}{6 \ell} \ip{\vdelta^y_t, \vy_t - \vy^*_t}
\end{split}
\end{equation}
Taking expectation from both side of Equation~\ref{eqn:gogda16}, we have:
\begin{equation}
\label{eqn:gogda17}
\begin{split}
 \E \left[\| \vz_{t+1} - \vy^*_t \|^2 +   \frac{\beta}{4} \|\vy_t -\vy_{t-1}\|^2 \right]   &\le  (1 - \frac{\beta}{6 \kappa}) \E\left [\| \vz_t - \vy^*_t \|^2\right ] +  \frac{\beta}{18} \E[\|\vy_{t-1} -\vy_{t-2}\|^2]   \\
 &\quad+ \frac{\beta}{12} \E[\|\vx_t - \vx_{t-1}\|^2]   + \frac{\beta}{18} \E[\|\vx_{t-1}-\vx_{t-2}\|^2]   \\
 &\quad+ \frac{ \beta \sigma^2}{3 \ell^2 M_y}
\end{split}
\end{equation}
Also using Young's inequality we have: 
\begin{equation}
\label{eqn:gogda18}
\| \vz_t - \vy^*_t \|^2 \le (1+\frac{\beta}{12  \kappa}) \| \vz_{t} - \vy^*_{t-1} \|^2 + (1 + 12 \frac{ \kappa}{\beta}) \kappa^2 \|\vx_t-\vx_{t-1}\|^2  
\end{equation}
where we used the fact that for any $\alpha > 0$, $\|\vx+\vy\|^2 \le (1 + \alpha) \|\vx\|^2 + (1 + \frac{1}{\alpha}) \|\vy\|^2$, and $\kappa$-lipschitzness of $\vy^*(\vx)$. Plugging Equation~\ref{eqn:gogda18} back to Equation~\ref{eqn:gogda17}, we have:
\begin{equation}
\label{eqn:gogda19}
\begin{split}
 \E \left[\| \vz_{t+1} - \vy^*_t \|^2 +   \frac{\beta}{4} \|\vy_t -\vy_{t-1}\|^2 \right]   &\le  (1 - \frac{\beta}{12 \kappa}) \E\left [\| \vz_t - \vy^*_{t-1} \|^2 +  \frac{\beta}{4} \E[\|\vy_{t-1} -\vy_{t-2}\|^2] \right ]   \\
 & \quad+ 12 \kappa^3 \E[\|\vx_t - \vx_{t-1}\|^2]   + \frac{\beta}{18} \E[\|\vx_{t-1}-\vx_{t-2}\|^2]   \\
 &\quad+ \frac{\beta \sigma^2}{3 \ell^2 M_y}
\end{split}
\end{equation}
Therefore, if we let $\vr_t = \|\vz_{t+1} - \vy^*_t \|^2 +   \frac{\beta}{4} \|\vy_t -\vy_{t-1}\|^2 $, then we have: 
\begin{equation}
\label{eqn:gogda20}
\begin{split}
 \E[\vr_t]   &\le  (1 - \frac{\beta}{12 \kappa})  \E[\vr_{t-1}]  + 12 \eta_{x,1}^2 \kappa^3 \E[\|\vg_{t-1}\|^2]   + \frac{\beta \eta_{x,1}^2}{18} \E[\|\vg_{t-2}\|^2]  + \frac{\beta \sigma^2}{3 \ell^2 M_y}
\end{split}
\end{equation}

We can derive the following equation, by applying Lemma~\ref{lemma:a2}. 
\begin{equation}
\label{eqn:gogda21}
\begin{split}
\sum_{i=1}^{t} \E[\vr_i] &\le \frac{12 \kappa}{\beta} \E[\vr_1] + 144 \frac{\eta_{x,1}^2 \kappa^4}{\beta}  \sum_{i=1}^{t-1} \E[\|\vg_{i} \|^2] + \frac{2}{3} \eta_{x,1}^2 \kappa  \sum_{i=1}^{t-2} \E[\|\vg_{i} \|^2] + \frac{2}{3} \kappa \E[\|\vx_1 - \vx_0] \|^2 \\
&\quad+ \frac{4 \kappa \sigma^2 (t-1)}{ \ell^2 M_y}
\end{split}
\end{equation}
Or equivalently we have: 
\begin{equation}
\label{eqn:gogda22}
\begin{split}
\sum_{i=1}^{t} \E[\vr_i] \le \frac{12 \kappa}{\beta} \E[\vr_1] + \frac{2}{3}  \kappa \E[\|\vx_1 - \vx_0 \|^2] +  145 \frac{\eta_{x,1}^2 \kappa^4}{\beta}  \sum_{i=1}^{t-1} \E[\|\vg_{i} \|^2]    + \frac{4 \kappa \sigma^2 (t-1)}{ \ell^2 M_y}
\end{split}
\end{equation}

\end{proof}

\begin{proof}[\textbf{Proof of Theorem}~\ref{thm:gogda}]

We begin by taking summation of Equation~\ref{eqn:glm1} (Lemma~\ref{lemma:30}) from $t=2$ to $t = T$ which yields: 
\begin{equation}
\begin{split}
    \frac{\eta_{x,1}}{2} \sum_{i=1}^{T-1} \E[\| \nabla \Phi (\vx_i) \|^2] &\le \Phi(\vx_1) -\E[ \Phi(\vx_T)] + \frac{3}{2} \eta_{x,1} \alpha^2 \ell^2 \|\vx_1 - \vx_0 \|^2 \\
    &\quad- \frac{\eta_{x,1}}{2} ( 1 - 2 \kappa \ell \eta_{x,1} ) \sum_{i=1}^{T-1} \E[\|\vg_i\|^2] + \frac{3}{2}\eta_{x,1}^3  \alpha^2 \ell^2 \sum_{i=1}^{T-2} \E[\|\vg_i \|^2] \\
    &\quad+ \frac{3}{2} \eta_{x,1} \ell^2 \sum_{i=1}^{T-1} \| \vy_i - \vy^*_i\|^2  + 
     \frac{3}{2} \eta_{x,1} \alpha^2 \ell^2 \sum_{i=1}^{T-1} \E[\|\vy_i - \vy_{i-1}\|^2] \\
     &+3 ((1+\alpha)^2 +1 ) \eta_{x,1} \frac{(T-1)\sigma^2}{M_x}
\end{split}
\end{equation}
Now note that if $\eta_x \le \frac{1}{2 \kappa \ell}$ then we can drop $\|\vg_{T-1}\|^2$ term in above equation. By considering this, and multiplying both  sides by $\frac{2}{\eta_{x,1}}$ we get (also let $\Delta = \Phi(\vx_1) - \min_{\vx} \Phi(\vx) $) :
\begin{equation}
\begin{split}
     \sum_{i=1}^{T-1} \E[\| \nabla \Phi (\vx_i) \|^2] &\le \frac{2 \Delta}{\eta_{x,1}} + 3  \alpha^2 \ell^2 \| \vx_1 - \vx_0\|^2  \\
     &\quad- ( 1 - 2 \kappa \ell \eta_{x,1} - 3 \eta_{x,1}^2 \alpha^2 \ell^2) \sum_{i=1}^{T-2} \E[\|\vg_i\|^2]  \\
    &\quad+  3 \ell^2 \sum_{i=1}^{T-1} \E[\|\vy^*_i - \vy_i\|^2]  + 
     3  \alpha^2 \ell^2 \sum_{i=1}^{T-1} \E[\|\vy_i - \vy_{i-1}\|^2] \\
     &+ 6((1+\alpha)^2 +1 ) \frac{(T-1)\sigma^2}{M_x} \\
\end{split}
\end{equation}
We can replace $\sum_{i=1}^{T-1} \|\vy_i^* - \vy_i\|^2$ with its upper bound obtained  in Lemma~\ref{lemma:40} to get: 
\begin{equation}
\begin{split}
     \sum_{i=1}^{T-1} \| \nabla \Phi (\vx_i) \|^2 &\le \frac{2 \Delta}{\eta_{x,1}} +  3 \alpha^2 \ell^2 \| \vx_1 - \vx_0 \|^2  + \frac{27}{7}   \ell^2 \|\vy_1 - \vy^*_1\|^2 \\
     &\quad- ( 1 - 2 \kappa \ell \eta_{x,1} - 3 \eta_{x,1}^2 \alpha^2 \ell^2 - \frac{54}{7}\eta_{x,1}^2 \kappa^2 \ell^2) \sum_{i=1}^{T-2} \E[\|\vg_i\|^2]  \\
    &\quad+  \frac{108}{7} \ell^2 \sum_{i=2}^{T-1} \E[ \| \vz_i - \vy^*_{i-1} \|^2] + 
     3  \ell^2 \sum_{i=1}^{T-1} \E[\|\vy_i - \vy_{i-1}\|^2]\\
     &+  6((1+\alpha)^2 +1 ) \frac{(T-1)\sigma^2}{M_x} + \frac{6}{7}\frac{(T-2)\sigma^2}{M_y} \\
\end{split}
\end{equation}
Now note that $\frac{108}{7}  \E[ \| \vz_{i+1} - \vy^*_{i} \|^2] + 
     3 \beta  \sum_{i=2}^{T-1} \E[\|\vy_i - \vy_{i-1}\|^2] \le 15.5 \E[\vr_i] $. Therefore we have: 
\begin{equation}
\begin{split}
     \sum_{i=1}^{T-1} \| \nabla \Phi (\vx_i) \|^2 &\le \frac{2 \Delta}{\eta_{x,1}} +  3 \alpha^2 \ell^2 \| \vx_1 - \vx_0 \|^2  + \frac{27}{7}   \ell^2 \|\vy_1 - \vy^*_1\|^2 \\
     &\quad- ( 1 - 2 \kappa \ell \eta_{x,1} - 3 \eta_{x,1}^2 \alpha^2 \ell^2 - \frac{54}{7}\eta_{x,1}^2 \kappa^2 \ell^2) \sum_{i=1}^{T-2} \E[\|\vg_i\|^2]  \\
    &\quad+ 15.5 \ell^2 \sum_{i=1}^{T-1} \E[\vr_i] +  6((1+\alpha)^2 +1 ) \frac{(T-1)\sigma^2}{M_x} + \frac{6}{7}\frac{(T-2)\sigma^2}{M_y} \\
\end{split}
\end{equation}
Furthermore, using Lemma~\ref{lemma:20}, we can find an upper bound on $\sum_{i=1}^{T-1} \E[\vr_i] $, and replacing it in above equation yields:
\begin{equation}
\begin{split}
     &\sum_{i=1}^{T-1} \| \nabla \Phi (\vx_i) \|^2 \le \frac{2 \Delta}{\eta_{x,1}} + 186 \frac{\kappa \ell^2}{\beta} \E[\vr_1] + 11 \kappa \ell^2 \|\vx_1 - \vx_0\|^2+  3 \alpha^2 \ell^2 \| \vx_1 - \vx_0 \|^2  \\
     &\qquad+ \frac{27}{7}   \ell^2 \|\vy_1 - \vy^*_1\|^2 - ( 1 - 2 \kappa \ell \eta_{x,1} - 3 \eta_{x,1}^2 \alpha^2  \ell^2 - \frac{54}{7}\eta_{x,1}^2 \kappa^2 \ell^2 - 2248 \eta_{x,1}^2 \frac{\kappa^4 \ell^2}{\beta}) \sum_{i=1}^{T-2} \E[\|\vg_i\|^2]  \\
    &\qquad+ \frac{62 \kappa \sigma^2 (T-2)}{M_y} +  6((1+\alpha)^2 +1 ) \frac{(T-1)\sigma^2}{M_x} + \frac{6}{7}\frac{(T-2)\sigma^2}{M_y} \\
\end{split}
\end{equation}
By letting $\eta_{x,1} = \frac{\sqrt{\beta}}{50 \kappa^2 \ell}$, and $\eta_{x,2} \le \frac{1}{25 \ell}$, it holds that $-(  1 - 2 \kappa \ell \eta_{x,1} - 3 \eta_{x,1}^2 \alpha^2  \ell^2 - \frac{54}{7}\eta_{x,1}^2 \kappa^2 \ell^2 - 2248 \eta_{x,1}^2 \frac{\kappa^4 \ell^2}{\beta}) \sum_{i=1}^{T-2} \E[\|\vg_i\|^2] \le 0$. Therefore, with the choice of  letting rate $\eta_{x,1} = \frac{\sqrt{\beta}}{50 \kappa^2 \ell}$ and simplifying the terms, we have: 
\begin{equation}
\label{eqn:gogdaf}
\begin{split}
     \frac{1}{T-1} \sum_{i=1}^{T-1} \E[\| \nabla \Phi (\vx_i) \|^2] &\le 100 \frac{\kappa^2 \ell \Delta}{\sqrt{\beta}(T-1)}  +  186 \frac{\kappa \ell^2}{\beta(T-1)} \|\vy_1 - \vy^*_1+ \eta_{y,1} \vg_{y,1} -  \eta_{y,2}\vg_{y,0} \|^2 \\
     &\quad+ 47 \beta \frac{\kappa \ell^2}{T-1} \|\vy_1 - \vy_0 \|^2 +  \frac{(11\kappa + 3 \alpha^2) \ell^2}{T-1} \| \vx_1 - \vx_0 \|^2 \\
     &\quad + \frac{27}{7} \frac{\ell^2}{T-1} \| \vy_1 - \vy^*_1 \|^2 + \frac{63 \kappa \sigma^2 }{M_y} +  6((1+\alpha)^2 +1 ) \frac{\sigma^2}{M_x}  
\end{split}
\end{equation}
Using Young's inequality, and $\ell$-smoothness of $f$, we have: 
\begin{equation}
\begin{split}
\|\vy_1 - \vy^*_1+ \eta_{y,1} \vg_{y,1} -  \eta_{y,2}\vg_{y,0} \|^2 &\le 2 \|\vy_1 - \vy^*_1\|^2 +  2 \| \eta_{y,2} (\vg_{y,1}-\vg_{y,0}) + \eta_{y,2} (\beta -1) \vg_{y,1}\|^2 \\
& \le 2 \|\vy_1 - \vy^*_1\|^2  + \frac{1}{9} \|x_1 - x_0\|^2 + \frac{1}{9} \|y_1 - y_0\|^2 + \frac{1 - \beta}{9} \|\vy_1 - \vy^*_1\|^2
\end{split}
\end{equation}
Plugging this into Equation~\ref{eqn:gogdaf}, we have: 
\begin{equation}
\label{eqn:gogdaff}
\begin{split}
     \frac{1}{T-1} \sum_{i=1}^{T-1} \E[\| \nabla \Phi (\vx_i) \|^2] &\le 100 \frac{\kappa^2 \ell \Delta}{T-1}  +  376 \frac{\kappa \ell^2}{\beta(T-1)} \|\vy_1 - \vy^*_1 \|^2 \\
     &\quad+ 68 \frac{\kappa \ell^2}{\beta(T-1)} \|\vy_1 - \vy_0 \|^2 +  \frac{(32\kappa + 3 \alpha^2) \ell^2 }{\beta (T-1)} \| \vx_1 - \vx_0 \|^2 \\
     &\quad+ \frac{63 \kappa \sigma^2 }{M_y} +  6((1+\alpha)^2 +1 ) \frac{\sigma^2}{M_x}  
\end{split}
\end{equation}
which completes the proof as stated.
\end{proof} 

\clearpage
\bibliography{main}
\bibliographystyle{abbrvnat}

\end{document}